\documentclass[lettersize,journal]{IEEEtran}
\usepackage{bm}
\usepackage{verbatim}
\usepackage{amsmath,amsfonts}
\usepackage{cases}
\usepackage{booktabs}
\usepackage{graphicx}
\usepackage{multirow}
\usepackage{algorithmic}
\usepackage{subfigure}
\usepackage{algorithm}
\usepackage{array}
\usepackage{textcomp}
\usepackage{stfloats}
\usepackage{url}
\usepackage{cite}
\usepackage{amsthm,amsmath,amssymb}
\usepackage{mathrsfs}
\usepackage{threeparttable}

\usepackage{bbding}

\usepackage{xcolor}

\begin{document}

\title{Joint Multi-view Unsupervised Feature Selection and Graph Learning}
	
	\author{Si-Guo Fang,
		Dong Huang,
		Chang-Dong Wang,
		and~Yong Tang
		\thanks{S.-G. Fang and D. Huang are with the College of Mathematics and Informatics, South China Agricultural University, Guangzhou, China, and also with Key Laboratory of Smart Agricultural Technology in Tropical South China, Ministry of Agriculture and Rural Affairs, China. \protect\\
			E-mail: siguofang@hotmail.com, huangdonghere@gmail.com.}
			\thanks{C.-D. Wang is with the School of Computer Science and Engineering,
			Sun Yat-sen University, Guangzhou, China, and also with Guangdong Key Laboratory of Information Security Technology, Guangzhou, China.
			E-mail: changdongwang@hotmail.com.}
			\thanks{Y. Tang is with the School of Computer Science, South China Normal University, Guangzhou, China.
			E-mail: ytang@m.scnu.edu.cn.}
		\thanks{© 2023 IEEE.  Personal use of this material is permitted.  Permission from IEEE must be obtained for all other uses, in any current or future media, including reprinting/republishing this material for advertising or promotional purposes, creating new collective works, for resale or redistribution to servers or lists, or reuse of any copyrighted component of this work in other works.}
		}
	
	\markboth{}
	{Shell \MakeLowercase{\textit{et al.}}: A Sample Article Using IEEEtran.cls for IEEE Journals}
	
	\maketitle
	
	\begin{abstract}
		
		Despite significant progress, previous multi-view unsupervised feature selection methods mostly suffer from two limitations. First, they generally utilize \emph{either} cluster structure \emph{or} similarity structure to guide the feature selection, which neglect the possibility of a joint formulation with mutual benefits. Second, they often learn the similarity structure by \emph{either} global structure learning \emph{or} local structure learning, which lack the capability of graph learning with both global and local structural awareness. In light of this, this paper presents a \textit{j}oint \textit{m}ulti-\textit{v}iew unsupervised \textit{f}eature selection and \textit{g}raph learning (JMVFG) approach. Particularly, we formulate the multi-view feature selection  with orthogonal decomposition, where each target matrix is decomposed into a view-specific basis matrix and a view-consistent cluster indicator. The cross-space locality preservation is incorporated to bridge the cluster structure learning in the projected space and the similarity learning (i.e., graph learning) in the original space. Further, a unified objective function is presented to enable the simultaneous learning of the cluster structure, the global and local similarity structures, and the multi-view consistency and inconsistency, upon which an alternating optimization algorithm is developed with theoretically proved convergence. Extensive experiments on a variety of real-world  multi-view datasets demonstrate the superiority of our approach for both the multi-view feature selection and graph learning tasks. The code is available at \url{https://github.com/huangdonghere/JMVFG}.
		
	\end{abstract}
	

\section{Introduction}\label{sec:introduction}

\IEEEPARstart{T}{he} rapid development of information technology gives rise to the mass emergence of high-dimensional data, which can be collected from different sources (or views) and have brought significant challenges to the field of computational intelligence and machine learning. In particular, due to the scarcity of true labels in massive data, the unsupervised learning has recently emerged as a promising direction, where the unsupervised feature selection \cite{solori20_air} and  the clustering analysis \cite{luo21_tetci} may be two of the most popular topics. Although extensive studies have been conducted on each of them, yet surprisingly few efforts have been devoted to the simultaneous and unified modeling of these two research topics for multi-view high-dimensional data. In view of this, this paper  focuses on the intersection of multi-view unsupervised feature selection and multi-view clustering (especially via graph learning).

The clustering analysis is a fundamental yet still challenging research topic in computational intelligence \cite{FastMICE,Fang2023,9964104,9141397,9732521,zhong23_npl}. In single-view clustering, the graph-based methods have been a widely-studied category, where the sample-wise relationships are captured by some graph structure, and the final clustering is typically obtained via graph partitioning. Yet a common limitation to most graph-based clustering methods is that they often rely on some predefined affinity graph, which lack the ability to learn the affinity graph adaptively. To deal with this limitation, some graph learning methods have been developed, aiming to learn a better graph via optimization or some heuristics \cite{niesingle,kangsingle,Huang2021}.

Extending from single-view to multiple views, the multi-view graph learning technique has recently shown its promising advantage in adaptively learning a robust and unified graph from multiple graph structures built in multiple views \cite{SwMC,zhan666,Liang2022,tang21_ijcai}. For example, Nie et al. \cite{SwMC} fused multiple affinity graphs into a unified graph, where the graphs from multiple views are adaptively weighted. Zhan et al. \cite{zhan666} explored the correlation of multiple graph structures, and simultaneously learned a unified graph and the corresponding cluster indicator. Liang et al. \cite{Liang2022} modeled both the multi-view consistency and inconsistency into a graph learning framework for robust multi-view clustering. However, these multi-view graph learning works \cite{SwMC,zhan666,Liang2022} typically rely on the multiple single-view graphs built on the original features, which undermine their ability to deal with multi-view high-dimensional data, where redundant and noisy features may exist or even widely exist. As an early attempt, Xu et al. \cite{xu666} proposed a weighted multi-view clustering method with a feature weighting (or feature selection) strategy, where neither the view-wise relationship nor the structural information of multiple views has been considered in the feature weighting process.

More recently, some unsupervised feature selection methods have been developed for multi-view high-dimensional data, where a crucial issue lies in how the multi-view information can be modeled to guide the feature selection process. In terms of this issue, two types of guidance are often leveraged, that is, the cluster structure (via some cluster indicator) and the similarity structure (via graph learning). Specifically, Liu et al. \cite{RMFS} dealt with the multi-view unsupervised feature selection problem by utilizing the cluster indicator (via pseudo labels produced by multi-view $k$-means clustering) as the global guidance of feature selection. Dong et al. \cite{ACSL} guided the multi-view unsupervised feature selection by learning an adaptive similarity graph that is expected to be close to the weighted sum of the multiple single-view graphs, which typically resorts to the consistency of global structures. Zhang et al. \cite{MAMFS} learned a similarity graph from data representations (via projection matrices on multi-view features) for unsupervised feature selection. Wan et al. \cite{ASE-UMFS} exploited the projected features to learn the similarity graph with the sparsity constraint imposed on the projection matrices. The common intuition behind the similarity learning in \cite{MAMFS} and \cite{ASE-UMFS} is that two samples with a smaller distance in the projected space should be assigned a greater similarity, which typically resort to the consistency of local structures.

Despite the considerable progress in these multi-view unsupervised feature selection methods \cite{RMFS,ACSL,MAMFS,ASE-UMFS} , there are still three critical questions that remain to be addressed.

\begin{description}
	\item[Q1:] Regarding the types of guidance for feature selection, they mostly exploit \emph{either} the cluster structure \cite{RMFS} \emph{or} the similarity structure \cite{ACSL,MAMFS,ASE-UMFS}, where a unified formulation of both types of guidance is still absent. This leads to the first question: \textit{how to jointly exploit \textbf{cluster} structure learning and \textbf{similarity} structure learning, and enable them to promote each other mutually?}
	\item[Q2:] Regarding the similarity learning, they generally learn a similarity structure \emph{either} by global structure learning  \cite{ACSL}  \emph{or} by local structure learning \cite{MAMFS,ASE-UMFS}, which overlook the potential benefits of simultaneous global and local learning. This gives rise to the second question: \textit{how to adaptively leverage \textbf{global} and \textbf{local} structures to enhance the multi-view feature selection and graph learning performance?}
	\item[Q3:] Regarding the common and complementary information of multiple views, starting from Q1 and Q2, the third question that arises is \emph{how to simultaneously and sufficiently investigate multi-view \textbf{consistency} and \textbf{inconsistency} throughout the framework?}
\end{description}

In light of the above questions, this paper proposes a \textbf{j}oint \textbf{m}ulti-\textbf{v}iew unsupervised \textbf{f}eature selection and \textbf{g}raph learning (JMVFG) approach (as illustrated in Fig.~\ref{fig_model}). Specifically, the projection matrix (i.e., the feature selection matrix) is constrained via orthogonal decomposition, where the target matrix of each view is decomposed into a view-specific basis matrix and a view-consistent cluster indicator matrix.  The multi-view local structures \emph{in the projected space} is captured by a unified similarity matrix, which is adaptively and mutually collaborated with the global multi-graph fusion \emph{in the original space}, leading to the cross-space locality preservation mechanism. Throughout the multi-view formulation, the consistency and inconsistency of multiple views are sufficiently considered. To enable the joint learning of the cluster structure (via orthogonal decomposition), the local structures in the projected space, the global structure in the original space, and the multi-view consistency and inconsistency, we present a unified objective function, which can be optimized via an alternating minimization algorithm.
In particular, we theoretically prove the convergence of the proposed optimization algorithm, and validate the convergence property with further empirical study.
Experiments are conducted on a variety of multi-view high-dimensional datasets, which have confirmed the superior performance of the proposed approach for both the multi-view unsupervised feature selection and graph learning (for clustering) tasks.

For clarity, the main contributions of this work can be summarized as follows:
\begin{itemize}
  \item We formulate the multi-view unsupervised feature selection and graph learning into a joint learning framework, through which the consistency and inconsistency of multiple views as well as the global and local structures can be simultaneously exploited.

  \item We employ a cross-space locality preservation mechanism to bridge the gap between the cluster structure learning (via orthogonal decomposition strategy) in the projected space and the similarity structure learning (via graph fusion) in the original space.

  \item A joint multi-view unsupervised feature selection and graph learning approach termed JMVFG is proposed, whose convergence is theoretically proved. Extensive experimental results have confirmed its superiority over the state-of-the-art.

\end{itemize}

The rest of this paper is arranged as follows. Section \ref{sec:work} reviews the related works. Section \ref{sec:method} describes the proposed framework. Section \ref{sec:Opt} presents the optimization algorithm and its theoretical analysis. Section \ref{sec:experiments} reports the experimental results. Finally, we conclude this paper in Section \ref{sec:conclusion}.

\begin{figure*}[!th]\vskip -0.2 in
	\centering
	\includegraphics[width=6.6 in]{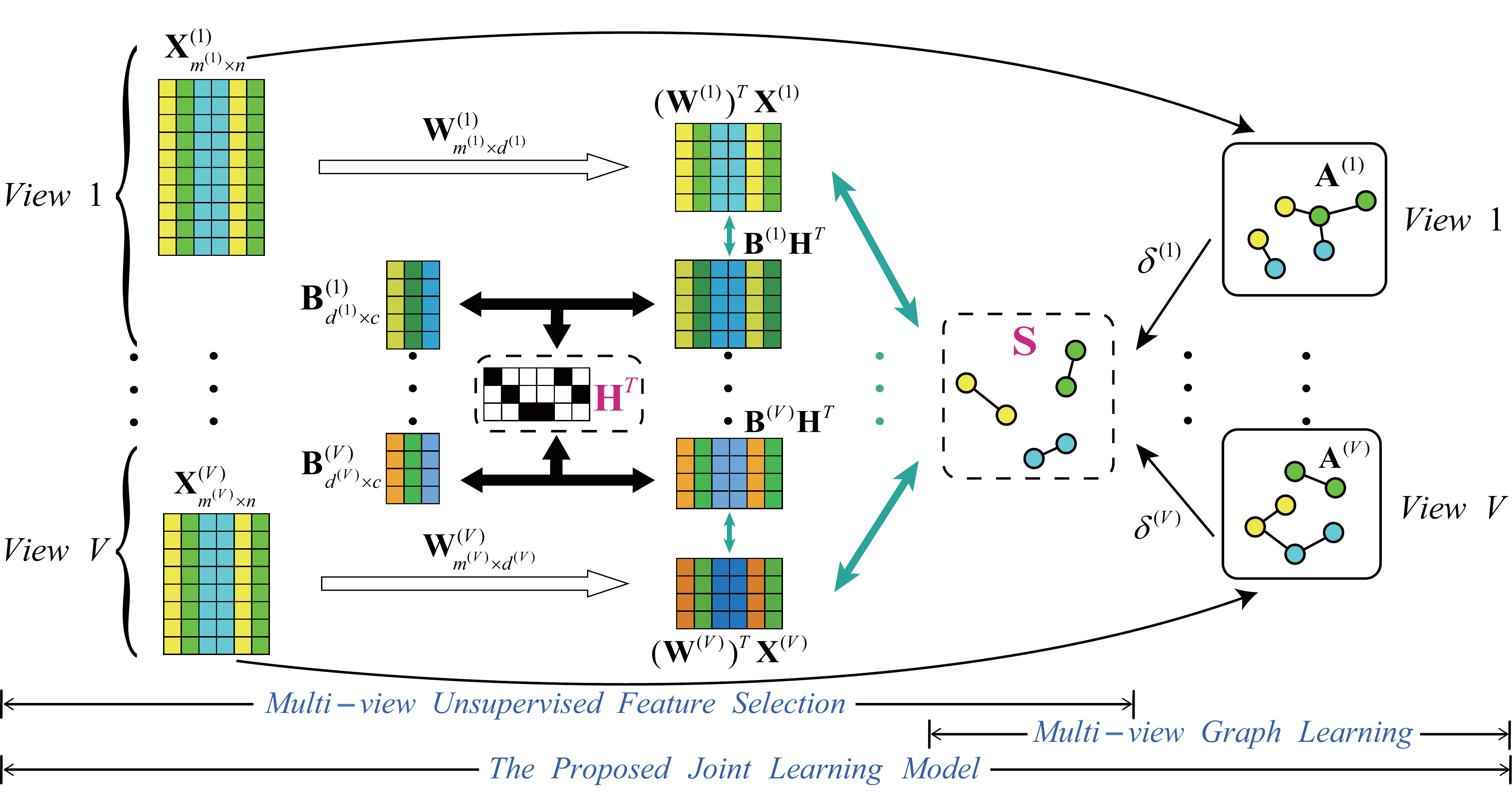}
	\caption{The overall framework of the proposed JMVFG approach.}
	\label{fig_model}
\end{figure*}

\section{Related Work}
\label{sec:work}
In the section, we review the related works on unsupervised feature selection and multi-view clustering (especially via graph learning) in Sections~\ref{sec:ufs} and \ref{sec:mvc}, respectively.

\subsection{Unsupervised Feature Selection}
\label{sec:ufs}
Unsupervised feature selection is an important technique for high-dimensional data analysis \cite{LS,SOCFS,SRCFS,PDPR}. It aims to select a subset of informative features while removing the redundant and noisy ones in an unsupervised manner.
For single-view unsupervised feature selection, a na\"ive method is to rank the feature importance by the feature variance \cite{MaxVar}, where a feature with greater variance is deemed to have higher importance. To exploit the information of local structure, He et al. \cite{LS} proposed the Laplacian score (LS) method, which performs unsupervised feature selection with the help of the Laplacian matrix of the $K$-nearest neighbor ($K$-NN) graph. To incorporate the local structures of multiple subspaces,
Han and Kim \cite{SOCFS} introduced the orthogonal basis clustering into unsupervised feature selection with a target matrix decomposed into an orthogonal basis matrix and an encoding matrix. To improve the orthogonal basis clustering based feature selection method in \cite{SOCFS},
Lim and Kim \cite{PDPR} incorporated a pairwise dependence term into the objective function. And  Lin et al. \cite{OCLSP} utilized a locality preserving term and a graph regularization term in the orthogonal decomposition based formulation. These methods \cite{MaxVar,LS,SOCFS,PDPR,OCLSP,xu22_tetci} are typically designed for single-view unsupervised feature section, but lack the ability to leverage the rich and complementary information in data with multiple views.

For multi-view unsupervised feature selection, quite a few attempts have been made, which typically resort to \emph{either} cluster structure (via pseudo-labels)\cite{MVFS,RMFS} \emph{or} similarity structure (via local or global relationships) \cite{AMFS,OMVFS,ACSL,CvLP-DCL,ASE-UMFS,NSGL,MAMFS}. Tang et al. \cite{MVFS} utilized a pseudo-class label matrix to guide the feature selection on multiple views. Liu et al. \cite{RMFS}  utilized the pseudo-labels generated by multi-view $k$-means to guide the feature selection.
Wang et al. \cite{AMFS} proposed a multi-view unsupervised feature selection method with a local linear regression model to automatically learn multiple view-based Laplacian graphs for locality preservation.
Shao et al. \cite{OMVFS} adopted the nonnegative matrix factorization (NMF) in multi-view feature selection, which processes streaming/large-scale multi-view data in a online fashion.
Dong et al. \cite{ACSL} imposed a rank constraint on the Laplacian matrix to learn a collaborative similarity structure for guiding the feature selection.
Tang et al. \cite{CvLP-DCL} utilized both diversity and consensus across different views to select features, with both local and global similarity structures considered.
Wan et al. \cite{ASE-UMFS} extended the common embedding projection to learn the similarity structure with a sparsity constraint. Bai et al. \cite{NSGL} adopted local structure preserving and rank constraint of Laplacian matrix for multi-view feature selection. Zhang et al. \cite{MAMFS} performed view-specific projection and view-joint projection on features, upon which a similarity graph is learned for unsupervised feature selection.
These methods tend to guide the multi-view unsupervised feature selection by either cluster structure \cite{MVFS,RMFS} or similarity structure \cite{AMFS,OMVFS,ACSL,CvLP-DCL,ASE-UMFS,NSGL,MAMFS}. It remains an open problem how to jointly exploit cluster structure learning and similarity structure learning (i.e., graph learning), coupled with local and global information as well as multi-view consistency and inconsistency, in a multi-view unsupervised feature selection framework.

\subsection{Multi-view Clustering}
\label{sec:mvc}

Multi-view clustering is another popular topic in unsupervised learning, which aims to partition a set of data samples (with multi-view features) into a certain number of clusters \cite{chao21_tai}. In this section, we review the related works on multi-view clustering, with emphasize on the graph learning based multi-view clustering methods.

Many multi-view clustering methods have been designed in recent years \cite{FastMICE,Cai2023,AMGL,DiMSC,liangICDM,LMVSC,FPMVS-CAG,Fang2022,JLMVC,MCGC}, which can be classified into several main categories, such as the subspace learning based methods \cite{DiMSC,LMVSC,FPMVS-CAG}, the multi-kernel learning based methods \cite{JLMVC,CWK2M}, the deep learning based methods \cite{yang2021deep,ke2022efficient}, and the graph learning based methods \cite{AMGL,SwMC,MCGC,Liang2022}.  Specifically, the subspace learning based methods \cite{DiMSC,LMVSC,FPMVS-CAG} aim to discover the cluster structure from a set of low-dimensional subspaces via self-expressive learning. Cao et al. \cite{DiMSC} designed a multi-view subspace clustering based on smoothness and diversity, with the complementarity of multi-view representations explored. Wang et al. \cite{FPMVS-CAG} learned a consistent self-expressive structure via a consensus anchor set and different projection matrices of multiple views.
Alternatively, the multi-kernel learning based methods \cite{JLMVC,CWK2M} seek to combine multiple kernels of different views to learn the cluster structure. Chen et al. \cite{JLMVC} jointly learned the kernel representation tensors and the affinity matrix for multi-view clustering. Liu et al. \cite{CWK2M} proposed a cluster-weighted kernel $k$-means method where the clusters of different views are adaptively weighted.
Additionally, the deep learning based methods \cite{yang2021deep,ke2022efficient} aim to learn clustering-friendly representations by utilizing the deep neural networks. Yang et al. \cite{yang2021deep} employed multiple autoencoder networks and a heterogeneous graph learning module to learn more discriminative latent representations, where intra-view and inter-view collaborative learning are designed to obtain more accurate clustering result. Ke et al. \cite{ke2022efficient} leveraged the pseudo labels generated by the $k$-means algorithm and designed a reconstruction process to learn a common representation for deep multi-view clustering.

Recently the graph learning based methods \cite{AMGL,SwMC,MCGC,Liang2022} have made significant advances, which aim to adaptively learn a robust unified graph from multiple graph structures constructed in multiple views for enhanced clustering. Nie et al. \cite{AMGL} proposed an auto-weighted framework to fuse multiple affinity graphs into a unified graph. To achieve the unified graph with a better cluster structure, Nie et al. \cite{SwMC} further incorporated the Laplacian rank constraint into a self-weighted graph learning framework.
Zhan et al. \cite{MCGC} employed the graph regularization and the Laplacian rank constraint to learn the consensus graph from multiple views. Liang et al. \cite{Liang2022} considered the multi-view consistency and inconsistency in the multi-view graph learning. These graph learning based multi-view clustering methods \cite{AMGL,SwMC,MCGC,Liang2022} typically rely on the multiple graphs built on the original feature space, which restrict their ability to deal with multi-view high-dimensional data with potentially redundant and noisy features.

\begin{table}[!t]
\centering
\caption{Notations}
\label{notations}
\begin{tabular}{m{2.5cm}m{5.5cm}}
\toprule
Notations &Descriptions\\
\midrule
$n$                                     &The number of samples\\
$c$                                     &The number of clusters\\
$V$                                     &The number of views\\
$m^{(v)}$                               &The number of features in the $v$-th view\\
$d^{(v)}$                               &The number of features after projection in the $v$-th view\\
$\textbf{X}^{(v)}\in\mathbb{R}^{m^{(v)}\times n}$       &Data matrix in the $v$-th view\\
$\textbf{W}^{(v)}\in\mathbb{R}^{m^{(v)}\times d^{(v)}}$ &The feature selection and projection matrix\\
$\textbf{B}^{(v)}\in\mathbb{R}^{d^{(v)}\times c}$       &The orthogonal bases matrix of the $v$-th view\\
$\textbf{A}^{(v)}\in\mathbb{R}^{n\times n}$             &The initial similarity matrix of the $v$-th view\\
$\textbf{S}\in\mathbb{R}^{n\times n}$                   &The learned similarity graph\\
$\textbf{H}\in\mathbb{R}^{n\times c}$                   &The cluster indicator matrix\\
\bottomrule
\end{tabular}
\end{table}

\section{Methodology}
\label{sec:method}

In this section, we describe the overall framework of JMVFG. Specifically, the notations are introduced in Section~\ref{sec:notations}. The multi-view feature selection with orthogonality is formulated in Section~\ref{sec:mvufs}. The multi-view graph learning is formulated in Section~\ref{sec:mvgl}. Then, Section~\ref{sec:joint_obj} bridges the gap between multi-view feature selection and graph learning, and presents our joint learning model.

\subsection{Notations}
\label{sec:notations}
In this paper, we denote the matrices by boldface capital letters, the vectors by boldface lower-case letters, and the scalar values by italic letters. For an arbitrary matrix $\textbf{A}\in\mathbb{R}^{n\times m}$, $a_{ij}$ denotes its $(i,j)$-th entry, $\textbf{a}_{i\cdot}$ denotes its $i$-th row vector, and $\textbf{a}_{\cdot j}$ denotes its $j$-th column vector. $Tr(\textbf{A})$ denotes the trace of a matrix $\textbf{A}$ when $\textbf{A}$ is square, and $\textbf{A}^T$ denotes the transpose of $\textbf{A}$. $||\textbf{b}||_2=\sqrt{\sum_{i=1}^nb_i^2}$ denotes the 2-norm of a vector \textbf{b}, where $b_i$ represents the $i$-th entry of $\textbf{b}$. $||\textbf{A}||_F=\sqrt{\sum_{i=1}^n\sum_{j=1}^ma_{ij}^2}=\sqrt{Tr(\textbf{A}\textbf{A}^T)}$ denotes the Frobenius norm of \textbf{A}. $||\textbf{A}||_{2,1}=\sum_{i=1}^n\sqrt{\sum_{j=1}^ma_{ij}^2}$ denotes the $\ell_{2,1}$-norm of $\textbf{A}$. $\textbf{a}\ge 0$ or $\textbf{A}\ge 0$ indicates all of their entries are larger than or equal to zero. $\textbf{I}$ represents the identity matrix. $\textbf{1}=[1,1,...,1]^T$ is a column vector with the corresponding dimension. For clarity, Table \ref{notations} shows some frequently-used notations and their descriptions.

\subsection{Multi-view Feature Selection with Orthogonality}
\label{sec:mvufs}

In this section, we present the model of multi-view unsupervised feature selection with orthogonal cluster structure learning, where the multi-view target matrices are decomposed into multiple view-specific basis matrices and a view-consistent cluster indicator matrix.

For a multi-view dataset, let $\textbf{X}^{(v)}\in\mathbb{R}^{m^{(v)}\times n}$ be the data matrix of its $v$-th view, where $n$ denotes the number of samples, and $m^{(v)}$ denotes the dimension (i.e., the number of features) in the $v$-th view. Let $\textbf{W}^{(v)}\in\mathbb{R}^{m^{(v)}\times d^{(v)}}$ be the projection matrix of the $v$-th view, where $d^{(v)}$ is the number of features after projection. Let $\textbf{T}^{(v)}\in\mathbb{R}^{d^{(v)}\times n}$ be the target matrix. The objective is to make the projected matrix close to the target matrix. Thus, a na\"ive version of the multi-view unsupervised feature selection problem can be written as
\begin{align}
	\label{eq:naive_obj}
\min\limits_{\textbf{W}^{(v)}}\ \sum\limits_{v=1}^V\{||(\textbf{W}^{(v)})^T\textbf{X}^{(v)}-\textbf{T}^{(v)}||_F^2+\eta||\textbf{W}^{(v)}||_{2,1}\},
\end{align}
where $\eta\ge0$ is a hyperparameter to control the regularization term, and $||\textbf{W}^{(v)}||_{2,1}$ is the $\ell_{2,1}$-norm of $\textbf{W}^{(v)}$. The reason for employing the $\ell_{2,1}$-norm are two-fold. First, the $\ell_{2,1}$-norm regularizer can be easily optimized. Second, the $\ell_{2,1}$-norm regularizer induces row sparsity of the feature selection matrix, which can filter out the unimportant features.

The target matrix $\textbf{T}^{(v)}$ is typically unknown in unsupervised scenarios. To enable the cluster structure learning, the orthogonal decomposition can be employed on the target matrix \cite{SOCFS,PDPR,OCLSP}. In a simple situation, we can treat each view \emph{separately} and decompose the target matrix as $\textbf{T}^{(v)}=\textbf{B}^{(v)}(\textbf{H}^{(v)})^T$, where $\textbf{B}^{(v)}\in\mathbb{R}^{d^{(v)}\times c}$ is the orthogonal basis matrix, and $\textbf{H}^{(v)}\in\mathbb{R}^{n\times c}$ is the cluster indicator matrix of the $v$-th view. Then the objective (\ref{eq:naive_obj}) can be written as
\begin{align}
	\label{eq:obj_decomp_separate}
	\min\limits_{\substack{\textbf{W}^{(v)},\textbf{B}^{(v)},\\\textbf{H}^{(v)}}}\ &\sum\limits_{v=1}^V\{||(\textbf{W}^{(v)})^T\textbf{X}^{(v)}-\textbf{B}^{(v)}(\textbf{H}^{(v)})^T||_F^2
	+\eta||\textbf{W}^{(v)}||_{2,1}\}\notag\\
	s.t.\ &\forall v,(\textbf{H}^{(v)})^T\textbf{H}^{(v)}=\textbf{I},\textbf{H}^{(v)}\ge0; (\textbf{B}^{(v)})^T\textbf{B}^{(v)}=\textbf{I},
\end{align}
The orthogonal constraint on $\textbf{B}^{(v)}$ ensures that each column of it is independent. Therefore, the matrix $\textbf{B}^{(v)}$ can be regarded as a set of orthogonal bases (or cluster centers) of the projected space $(\textbf{W}^{(v)})^T\textbf{X}^{(v)}$. The orthogonal and nonnegative constraints on $\textbf{H}^{(v)}$ ensure that each row of it has one non-zero entry, which imposes that each sample can only be associated with one cluster via the cluster indicator matrix $\textbf{H}^{(v)}$.

Note that the formulation of objective (\ref{eq:obj_decomp_separate}) treats each view separately, yet ignores the view-wise relationship. To capture the consistency and complementariness of multiple views, a key concern is how to design the information-sharing mechanism across views. Thereby, we incorporate a shared cluster structure among different views by a view-consistent cluster indicator matrix $\textbf{H}$ (instead of a separate cluster indicator $\textbf{H}^{(v)}$ for each view), while maintaining the orthogonal basis matrix $\textbf{B}^{(v)}$ to reflect the particularity of each view. On the one hand, since the information of samples may be intrinsically different across views, the cluster centers (or orthogonal bases) in different views may be different. Thus the orthogonal basis matrix $\textbf{B}^{(v)}$ is designed to be view-specific to capture the versatile information of multiple views. On the other hand, we seek to build a consistent cluster structure out of the diversity multiple views. Thus the unified cluster indicator matrix $\textbf{H}$ is incorporated to learn the consistent cluster structure from multiple views. By decomposing the target matrix as $\textbf{T}^{(v)}=\textbf{B}^{(v)}\textbf{H}^T$ with both consistency and inconsistency considered, the objective function for multi-view feature selection with orthogonal decomposition can be written as
\begin{align}
\min\limits_{\textbf{W}^{(v)},\textbf{B}^{(v)},\textbf{H}}\ &\sum\limits_{v=1}^V\{||(\textbf{W}^{(v)})^T\textbf{X}^{(v)}-\textbf{B}^{(v)}\textbf{H}^T||_F^2+\eta||\textbf{W}^{(v)}||_{2,1}\}\notag\\
s.t.\ &\textbf{H}^T\textbf{H}=\textbf{I},\textbf{H}\ge0; \forall v, (\textbf{B}^{(v)})^T\textbf{B}^{(v)}=\textbf{I}.
\end{align}
Here, besides ensuring the discriminative ability of the cluster centers (i.e., bases) and the cluster separability of the cluster indicator, the orthogonal constraints on $\textbf{B}^{(v)}$ and $\textbf{H}$ can also avert the model from obtaining the trivial solution with $\textbf{B}^{(v)}=\textbf{0}$ and $\textbf{W}^{(v)}=\textbf{0}$.

\subsection{Multi-view Graph Learning}
\label{sec:mvgl}
With the \emph{cluster structure} learned via orthogonal decomposition, this section proceeds to investigate the \emph{similarity structure} via multi-view graph learning.

Multi-view graph learning seeks to learn a unified graph by fusing the information of multiple graphs built in multiple views. Before delving into the multi-view case, we first consider a single-view data matrix $\textbf{X}\in\mathbb{R}^{m\times n}$. Let $\textbf{A}\in\mathbb{R}^{n\times n}$ be its similarity matrix, where $a_{ij}$ is the $(i,j)$-th entry, corresponding to the similarity between $\textbf{x}_{\cdot i}$ (i.e., the $i$-th sample) and $\textbf{x}_{\cdot j}$ (i.e., the $j$-th sample). The similarity can be computed via the Gaussian kernel function with $K$-nearest neighbors preserved, that is
\begin{align}
\label{eq:3-5}
a_{ij}=\begin{cases} e^{\frac{||\textbf{x}_{\cdot i}-\textbf{x}_{\cdot j}||_2^2}{-2\sigma^2}}, &\text{if~}\textbf{x}_{\cdot i}\in N_K(\textbf{x}_{\cdot j})\text{~or~}\textbf{x}_{\cdot j}\in N_K(\textbf{x}_{\cdot i}),\\
0,&\text{otherwise},\end{cases}
\end{align}
where $N_K(\textbf{x}_{\cdot i})$ denotes the set of $K$-nearest neighbors of sample $\textbf{x}_{\cdot i}$, and $\sigma$ is the Gaussian kernel width which can be set to the median distance between two samples.

For multi-view data, given a data matrix $\textbf{X}^{(v)}\in\mathbb{R}^{m^{(v)}\times n}$ of the $v$-th view, let $\textbf{A}^{(v)}$ denote its similarity matrix. Then the objective function of multi-view graph learning can be formulated as
\begin{align}
	\label{eq:obj_pure_gl}
	\min\limits_{\textbf{S},\bm\delta}\ &\sum\limits_{v=1}^V||\textbf{S}-\delta^{(v)}\textbf{A}^{(v)}||_F^2\notag\\
	s.t.\ &\textbf{S}\textbf{1}=\textbf{1},\textbf{S}\ge0;\bm\delta^T\textbf{1}=1,\bm\delta\ge0,
\end{align}
where $\bm\delta=[\delta^{(1)},\delta^{(2)},...,\delta^{(V)}]^T$ and $\textbf{1}=[1,1,...,1]^T$ are the column vectors with corresponding dimension. Parameter $\bm\delta$ is a learnable parameter, which measures the importance of each view and can be automatically learned during the optimization of the objective function. For different datasets, the learned value of $\bm\delta$ may be different. Since the multiple similarity graphs $\textbf{A}^{(v)}$ can capture the inconsistency or diversity among multiple views, the purpose of the multi-view graph learning is to learn a consistent graph $\textbf{S}$ from the multiple view-specific graphs. Note that $\textbf{S}\textbf{1}=\textbf{1}$ and $\textbf{S}\ge0$ guarantee $0\le s_{ij}\le 1$, and $\bm\delta^T\textbf{1}=1$ and $\bm\delta\ge0$ guarantee $0\le \delta^{(v)}\le 1$, where $\delta^{(v)}$ adjusts the influence of different views. When considering the construction of $\textbf{A}^{(v)}$, it is expected that $\textbf{S}$ and $\delta^{(v)}\textbf{A}^{(v)}$ should be at a similar scale. With $\bm\delta^T\textbf{1}=1$, the average value of each $\delta^{(v)}$ is $1/V$. Thus the $(i,j)$-th entry of $\textbf{A}^{(v)}$ can be defined as $a_{ij}^{(v)}=V\cdot \bar{a}_{ij}^{(v)}/\sum_j\bar{a}_{ij}^{(v)}$, where $\bar{a}_{ij}^{(v)}$ is calculated by \eqref{eq:3-5} with $\textbf{X}^{(v)}$ being the input data. Thereby, the sum of the entries of each row of $\textbf{A}^{(v)}$ is $V$, which on the one hand ensures that $\textbf{S}$ and $\delta^{(v)}\textbf{A}^{(v)}$ are at a similar scale, and on the other hand preserves $\delta^{(v)}$ to adjust the view influence.

With the global graph structures of multiple views fused into a unified graph, the next task is to connect the cluster structure in the previous section and the global similarity structure in this section, together with further graph regularization on the local structures, in a unified framework.

\subsection{Joint Learning Model}
\label{sec:joint_obj}
Note that the cluster structure learning in Section~\ref{sec:mvufs} is formulated in the projected space, while the similarity learning (or graph learning) in Section~\ref{sec:mvgl} in the original space. In this section, we proceed to bridge the gap between the cluster structure and the similarity structure, as well as the gap between the projected space and the original space, and thus present a joint learning model.

To enable the joint learning of the feature selection matrices and the global graph structure, we incorporate a multi-view graph regularization term with \emph{cross-space locality preservation}. The intuition is that if two data points are similar to each other in the original space, then they should also be close to each other in the projected space.
Let $\bar{\textbf{S}}=(\textbf{S}^T+\textbf{S})/2$. If $\textbf{S}$ is symmetric, then it holds that $\bar{\textbf{S}}=\textbf{S}$.  The graph regularization term imposed between the pair-wise distance of each view in the projected space and the fused similarity in the original space is defined as
\begin{align}
\label{eq:3-10}
&\sum\limits_{i=1}^n\sum\limits_{j=1}^n\frac{1}{2}||(\textbf{W}^{(v)})^T\textbf{x}_{\cdot i}^{(v)}-(\textbf{W}^{(v)})^T\textbf{x}_{\cdot j}^{(v)}||_2^2\bar{s}_{ij}\notag\\
=&Tr({\textbf{W}^{(v)}}^T\textbf{X}^{(v)}\textbf{L}{\textbf{X}^{(v)}}^T\textbf{W}^{(v)}).
\end{align}

The projection matrix $\textbf{W}^{(v)}$ can filter out the redundant or noisy features in the original data space to some extent and makes the information of locality preservation more accurate. In the meantime, it  bridges the gap between the cluster structure learning in the projected space and the similarity learning in the original space. Thus the feature selection matrices $\textbf{W}^{(v)}$ of multiple views (with inconsistency) and the unified similarity structure $\bar{\textbf{S}}$ (with multi-view consistency) can be mutually promoted by each other with cross-space locality preservation.

Finally, by unifying the multi-view unsupervised feature selection and the multi-view graph learning through the graph regularization with cross-space locality preservation, the overall objective function can be formulated as
\begin{align}
	\label{eq:3-11}
		&\min\limits_{  \substack{\textbf{W}^{(v)},\textbf{B}^{(v)},\bm\delta,\\\textbf{H},\textbf{S}}  }\sum_{v=1}^V\{\underbrace{||(\textbf{W}^{(v)})^T\textbf{X}^{(v)}-\textbf{B}^{(v)}\textbf{H}^T||_F^2}_{\text{Feature Selection with Orthogonality}}+\underbrace{\eta||\textbf{W}^{(v)}||_{2,1}}_{\text{Regularization}}\notag\\
		&+\underbrace{\gamma Tr[(\textbf{W}^{(v)})^T\textbf{X}^{(v)}\textbf{L}(\textbf{X}^{(v)})^T\textbf{W}^{(v)}]}_{\text{Cross-Space Locality Preservation}}+\underbrace{\beta||\textbf{S}-\delta^{(v)}\textbf{A}^{(v)}||_F^2}_{\text{Graph Learning}}\}\notag\\
		&~~~~s.t.\ \textbf{S1}=\textbf{1},\textbf{S}\ge0;\textbf{H}^T\textbf{H}=\textbf{I},\textbf{H}\ge0;\bm\delta^T\textbf{1}=1,\bm\delta\ge0;\notag\\
		&~~~~~~~~~~~\forall v,(\textbf{B}^{(v)})^T\textbf{B}^{(v)}=\textbf{I},
\end{align}
where $\gamma\ge0$ and $\beta\ge0$ are hyperparameters that control the influences of the cross-space locality preservation term and the graph learning term, respectively.

Therefore, the unified model can jointly utilize the cluster structure with orthogonality, the global structure in the original space, and the local structures in the projected space, where the multi-view consistency and inconsistency are also exploited. The unsupervised feature selection can find a more compact subset from the original features and filter out the redundant, irrelevant, and noisy features, which makes the similarity measurement between samples in graph learning more accurate. In the meantime, the multi-view graph learning can guide the feature selection process by fusing the local structure information of the original space to avoid selecting unreasonable features. By optimizing the objective \eqref{eq:3-11}, the feature selection matrix $\textbf{W}^{(v)}$ can be obtained for feature selection, and at the same time the graph structure $\textbf{S}$ can be learned for further spectral clustering. For clarity, we illustrate the overall framework of JMVFG in Fig.~\ref{fig_model}.

\newtheorem{theorem}{Theorem}
\newtheorem{lemma}{Lemma}
\newtheorem{definition}{Definition}

\section{Optimization and Theoretical Analysis}
\label{sec:Opt}

In this section, we present an alternating optimization algorithm to solve the problem~\eqref{eq:3-11} in Section~\ref{sec:optimize_algo}, and provide the theoretical convergence analysis and the computational complexity analysis in Sections~\ref{sec:theoretical_convergence} and \ref{sec:theoretical_complexity}, respectively.

\subsection{Optimization of Problem}
\label{sec:optimize_algo}
The objective function \eqref{eq:3-11} can be transformed into the following equivalent form:
\begin{align}
\label{eq:32-1}
\min\limits_{  \substack{\textbf{W}^{(v)},\textbf{B}^{(v)},\bm\delta,\\\textbf{H},\textbf{S},\textbf{Z}}  }\sum_{v=1}^V\{|&|(\textbf{W}^{(v)})^T\textbf{X}^{(v)}-\textbf{B}^{(v)}\textbf{H}^T||_F^2+\eta||\textbf{W}^{(v)}||_{2,1}\notag\\
+\gamma Tr[(\textbf{W}^{(v)})^T&\textbf{X}^{(v)}\textbf{L}(\textbf{X}^{(v)})^T\textbf{W}^{(v)}]+\beta||\textbf{S}-\delta^{(v)}\textbf{A}^{(v)}||_F^2\}\notag\\
s.t.\ \textbf{S1}=\textbf{1},\textbf{S}\ge&0;\textbf{H}^T\textbf{H}=\textbf{I},\textbf{H}=\textbf{Z},\textbf{Z}\ge0;\bm\delta^T\textbf{1}=1,\bm\delta\ge0;\notag\\
\forall v,(&\textbf{B}^{(v)})^T\textbf{B}^{(v)}=\textbf{I},
\end{align}
where $\textbf{Z}$ is an auxiliary matrix that satisfies $\textbf{H}=\textbf{Z}$ and $\textbf{Z}\ge0$. Through the auxiliary matrix $\textbf{Z}$, the nonnegative constraint on $\textbf{H}$ is transferred to $\textbf{Z}$. Although the decision variables and constraints are added, the equality constraint $\textbf{H}=\textbf{Z}$ can be integrated into the objective function through the penalty term, which is more conducive to the subsequent optimization solution. The objective function with the penalty term can be written as
\begin{align}
\label{eq:32-2}
\min\limits_{  \substack{\textbf{W}^{(v)},\textbf{B}^{(v)},\bm\delta,\\\textbf{H},\textbf{S},\textbf{Z}}  }\sum_{v=1}^V\{|&|(\textbf{W}^{(v)})^T\textbf{X}^{(v)}-\textbf{B}^{(v)}\textbf{H}^T||_F^2+\eta||\textbf{W}^{(v)}||_{2,1}\notag\\
+\gamma Tr[(\textbf{W}^{(v)})^T&\textbf{X}^{(v)}\textbf{L}(\textbf{X}^{(v)})^T\textbf{W}^{(v)}]+\beta||\textbf{S}-\delta^{(v)}\textbf{A}^{(v)}||_F^2\}\notag\\
+\alpha||\textbf{H}-\textbf{Z}||_F^2~\,&\notag\\
s.t.\ \textbf{S1}=\textbf{1},\textbf{S}&\ge0;\textbf{H}^T\textbf{H}=\textbf{I},\textbf{Z}\ge0;\bm\delta^T\textbf{1}=1,\bm\delta\ge0;\notag\\
\forall v&,(\textbf{B}^{(v)})^T\textbf{B}^{(v)}=\textbf{I};
\end{align}
where $\alpha>0$ is the penalty factor. Provided that the penalty factor $\alpha$ is large enough, which is not necessary to go to infinity, the local optimal solution of the objective \eqref{eq:32-1} can be obtained by minimizing the objective \eqref{eq:32-2}. In the following, we utilize the \emph{alternating direction method of multipliers} (ADMM) to solve the above problem. Particularly, the optimization procedure of JMVFG is given in Algorithm 1.

\subsubsection{Update $\bm\delta$}
With other variables fixed, the subproblem that only relates to $\bm\delta=[\delta^{(1)},\delta^{(2)},...,\delta^{(V)}]^T$ can be written as
\begin{align}
\label{eq:3-36}
&\min_{\bm\delta}~\mathcal{L}(\bm\delta)=\sum_{v=1}^V~||\textbf{S}-\delta^{(v)}\textbf{A}^{(v)}||_F^2\notag\\
&~s.t.~~\bm\delta^T\textbf{1}=1,\bm\delta\ge0.
\end{align}
The Lagrange function of the objective \eqref{eq:3-36} is
\begin{align}
L&=\frac{1}{2}\sum_v||\textbf{S}-\delta^{(v)}\textbf{A}^{(v)}||_F^2-\lambda(\bm\delta^T\textbf{1}-1)-\bm\mu^T\bm\delta\notag\\
&=\sum_v\{\frac{1}{2}[Tr(\textbf{S}\textbf{S}^T)+(\delta^{(v)})^2Tr(\textbf{A}^{(v)}(\textbf{A}^{(v)})^T)]\notag\\
&~~~~-\delta^{(v)}Tr(\textbf{A}^{(v)}\textbf{S}^T)\}-\lambda(\bm\delta^T\textbf{1}-1)-\bm\mu^T\bm\delta,
\end{align}
where $\lambda$ is a Lagrange multiplier scalar and $\bm\mu=[\mu^{(1)},\mu^{(2)},...,\mu^{(V)}]^T$ is a Lagrange multiplier vector. Then we have the partial derivative of the decision variable $\delta^{(v)}$ (for $v=1,\cdots,V$) as follows
\begin{align}
\frac{\partial L}{\partial\delta^{(v)}}&=\delta^{(v)}Tr(\textbf{A}^{(v)}(\textbf{A}^{(v)})^T)-Tr(\textbf{A}^{(v)}\textbf{S}^T)-\lambda-\mu^{(v)}.
\end{align}
Assume that the optimal solution of \eqref{eq:3-36} is $\bm\delta^\ast=[\delta^{(1)\ast},\delta^{(2)\ast},...,\delta^{(V)\ast}]^T$, and the corresponding Lagrange multipliers are $\lambda^\ast$ and $\bm\mu^\ast$. Let $Tr(\textbf{A}^{(v)}\textbf{S}^T)=p^{(v)}$ and $Tr[\textbf{A}^{(v)}(\textbf{A}^{(v)})^T]=q^{(v)}>0$. According to the \emph{Karush-Kuhn-Tucker} (KKT) conditions \cite{COPT} and the original constraints, we can have
\begin{numcases}{}
\forall v,~\delta^{(v)\ast}q^{(v)}-p^{(v)}-\lambda^\ast-\mu^{(v)\ast}=0; \label{eq:3-39a}\\
\forall v,~\mu^{(v)\ast}\ge0,~\mu^{(v)\ast}\delta^{(v)\ast}=0; \label{eq:3-39b}\\
(\bm\delta^\ast)^T\textbf{1}=1,~\bm\delta^\ast\ge0. \label{eq:3-39c}
\end{numcases}
From Eq.~\eqref{eq:3-39a}, we have $\delta^{(v)\ast}=(p^{(v)}+\lambda^\ast+\mu^{(v)\ast})/q^{(v)}$. With $(\bm\delta^\ast)^T\textbf{1}=1$, we know that
\begin{align}
\label{eq:3-42}
\lambda^\ast=\frac{1-\sum_v\frac{p^{(v)}+\mu^{(v)\ast}}{q^{(v)}}}{\sum_v\frac{1}{q^{(v)}}}.
\end{align}
According to Eqs.~\eqref{eq:3-39a} and \eqref{eq:3-42}, we have
\begin{align}
\delta^{(v)\ast}=\frac{p^{(v)}}{q^{(v)}}+\frac{1-\sum_v\frac{p^{(v)}}{q^{(v)}}}{q^{(v)}\sum_v\frac{1}{q^{(v)}}}-\frac{\sum_v\frac{\mu^{(v)\ast}}{q^{(v)}}}{q^{(v)}\sum_v\frac{1}{q^{(v)}}}+\frac{\mu^{(v)\ast}}{q^{(v)}}.
\end{align}
Let $k^{(v)}=\frac{p^{(v)}}{q^{(v)}}+\frac{1-\sum_v\frac{p^{(v)}}{q^{(v)}}}{q^{(v)}\sum_v\frac{1}{q^{(v)}}}$. From Eq.~\eqref{eq:3-39b}, we can obtain
\begin{align}
\label{eq:3-44}
\delta^{(v)\ast}&=k^{(v)}-\frac{\sum_v\frac{\mu^{(v)\ast}}{q^{(v)}}}{q^{(v)}\sum_v\frac{1}{q^{(v)}}}+\frac{\mu^{(v)\ast}}{q^{(v)}}\notag\\
&=(k^{(v)}-\frac{\sum_v\frac{\mu^{(v)\ast}}{q^{(v)}}}{q^{(v)}\sum_v\frac{1}{q^{(v)}}})_+,
\end{align}
where $x_+=max(x,0)$. Thereby, we can obtain the optimal solution $\bm\delta^\ast$ if we know $\sum_v\frac{\mu^{(v)\ast}}{q^{(v)}}$.

Let $\widetilde{\mu}^\ast=\sum_v\frac{\mu^{(v)\ast}}{q^{(v)}}$. From Eqs.~\eqref{eq:3-39b} and \eqref{eq:3-44}, we have
\begin{align}
\frac{\mu^{(v)\ast}}{q^{(v)}}&=-k^{(v)}+\frac{\widetilde{\mu}^\ast}{q^{(v)}\sum_v\frac{1}{q^{(v)}}}+\delta^{(v)\ast}\notag\\
&=(-k^{(v)}+\frac{\widetilde{\mu}^\ast}{q^{(v)}\sum_v\frac{1}{q^{(v)}}})_+.
\end{align}
Then we have $\widetilde{\mu}^\ast=\sum_v(-k^{(v)}+\frac{\widetilde{\mu}^\ast}{q^{(v)}\sum_v\frac{1}{q^{(v)}}})_+$. Further, we define a function as
\begin{align}
f(\widetilde{\mu})=\sum_v(-k^{(v)}+\frac{\widetilde{\mu}}{q^{(v)}\sum_v\frac{1}{q^{(v)}}})_+-\widetilde{\mu},
\end{align}
The root of $f(\widetilde{\mu})=0$ is $\widetilde{\mu}^\ast$. The root $\widetilde{\mu}^\ast$ can be obtained by the Newton method, that is, $\widetilde{\mu}_{(t+1)}=\widetilde{\mu}_{(t)}-f(\widetilde{\mu}_{(t)})/f^{\prime}(\widetilde{\mu}_{(t)})$. Here, the index $t$ represents the $t$-th iteration of the Newton method. Thus we have completed the update of $\bm\delta$.

Regarding the monotonicity, let $\bm\delta_{(t+1)}$ be the objective value obtained in the $(t+1)$-th iteration, we have Theorem~\ref{the:delta}.
\begin{theorem}
\label{the:delta}
In iteration $t+1$, $\mathcal{L}(\bm\delta_{(t+1)})\le\mathcal{L}(\bm\delta_{(t)})$ after solving the subproblem \eqref{eq:3-36}.
\end{theorem}
Please see supplementary material for detailed proof.

\subsubsection{Update $\textbf{W}^{(v)}$}

With other variables fixed, the subproblem that only relates to $\textbf{W}^{(v)}$ can be written as
\begin{align}
\min\limits_{\textbf{W}^{(v)}}\ \mathcal{L}(\textbf{W}^{(v)})=&||(\textbf{W}^{(v)})^T\textbf{X}^{(v)}-\textbf{B}^{(v)}\textbf{H}^T||_F^2+\eta||\textbf{W}^{(v)}||_{2,1}\notag\\
&+\gamma Tr[(\textbf{W}^{(v)})^T\textbf{X}^{(v)}\textbf{L}(\textbf{X}^{(v)})^T\textbf{W}^{(v)}].
\end{align}

Let $\mathcal{L}(\textbf{W}^{(v)})=\mathcal{L}_1(\textbf{W}^{(v)})+\eta\mathcal{L}_2(\textbf{W}^{(v)})+\gamma\mathcal{L}_3(\textbf{W}^{(v)})$. The first term $\mathcal{L}_1(\textbf{W}^{(v)})$ can be written as
\begin{align}
&\mathcal{L}_{1}(\textbf{W}^{(v)})=||(\textbf{W}^{(v)})^T\textbf{X}^{(v)}-\textbf{B}^{(v)}\textbf{H}^{T}||_{F}^{2}\notag\\
&=Tr\{[(\textbf{W}^{(v)})^T\textbf{X}^{(v)}-\textbf{B}^{(v)}\textbf{H}^{T}]^{T}[(\textbf{W}^{(v)})^T\textbf{X}^{(v)}-\textbf{B}^{(v)}\textbf{H}^{T}]\}\notag\\ &=Tr\{[(\textbf{X}^{(v)})^T\textbf{W}^{(v)}-\textbf{H}(\textbf{B}^{(v)})^T][(\textbf{X}^{(v)})^T\textbf{W}^{(v)}-\textbf{H}(\textbf{B}^{(v)})^T]^T\}.
\end{align}
Then we have
\begin{align}
\frac{d\mathcal{L}_{1}(\textbf{W}^{(v)})}{d\textbf{W}^{(v)}}&=2\textbf{X}^{(v)}[(\textbf{X}^{(v)})^{T}\textbf{W}^{(v)}-\textbf{H}(\textbf{B}^{(v)})^{T}]\notag\\
&=2\textbf{X}^{(v)}(\textbf{X}^{(v)})^{T}\textbf{W}^{(v)}-2\textbf{X}^{(v)}\textbf{H}(\textbf{B}^{(v)})^{T}.
\end{align}
The second term $\mathcal{L}_2(\textbf{W}^{(v)})$ can be written as
\begin{align}
\mathcal{L}_2(\textbf{W}^{(v)})=||\textbf{W}^{(v)}||_{2,1}=\sum_{i}\sqrt{\sum_{j}(w_{ij}^{(v)})^2}.
\end{align}
Then we have
\begin{align}
&\frac{d\mathcal{L}_2(\textbf{W}^{(v)})}{d\textbf{W}^{(v)}}=(\frac{\partial\sum_{i}\sqrt{\sum_{j}(w_{ij}^{(v)})^2}}{\partial w_{ij}^{(v)}})_{m^{(v)}\times d^{(v)}}\notag\\
&=\begin{pmatrix}
\frac{2w_{11}^{(v)}}{2||\textbf{W}^{(v)}_{1\cdot}||_2}&\cdots&\frac{2w_{1d^{(v)}}^{(v)}}{2||\textbf{W}^{(v)}_{1\cdot}||_2}\\
\vdots&&\vdots\\
\frac{2w_{m^{(v)}1}^{(v)}}{2||\textbf{W}^{(v)}_{m^{(v)}\cdot}||_2}&\cdots&\frac{2w_{m^{(v)}d^{(v)}}^{(v)}}{2||\textbf{W}^{(v)}_{m^{(v)}\cdot}||_2}
\end{pmatrix}=2\textbf{D}^{(v)}\textbf{W}^{(v)},
\end{align}
where $\textbf{D}^{(v)}\in\mathbb{R}^{m^{(v)}\times m^{(v)}}$ is a diagonal matrix, whose diagonal entry is computed as $d_{ii}^{(v)}=1/(2||\textbf{W}^{(v)}_{i\cdot}||_2)$, for $i=1,\cdots,m^{(v)}$.

The derivative of the third term $\mathcal{L}_3(\textbf{W}^{(v)})$ can be obtained as
\begin{align}
\frac{d\mathcal{L}_3(\textbf{W}^{(v)})}{d\textbf{W}^{(v)}}&=[\textbf{X}^{(v)}\textbf{L}(\textbf{X}^{(v)})^T+\textbf{X}^{(v)}\textbf{L}^T(\textbf{X}^{(v)})^T]\textbf{W}^{(v)}\notag\\
&=2\textbf{X}^{(v)}\textbf{L}(\textbf{X}^{(v)})^T\textbf{W}^{(v)}.
\end{align}

By setting the derivative of the objective function $\mathcal{L}(\textbf{W}^{(v)})$ to zero, we can have
\begin{align}
&\textbf{X}^{(v)}(\textbf{X}^{(v)})^T\textbf{W}^{(v)}-\textbf{X}^{(v)}\textbf{H}(\textbf{B}^{(v)})^T+\eta\textbf{D}^{(v)}\textbf{W}^{(v)}\notag\\
&+\gamma\textbf{X}^{(v)}\textbf{L}(\textbf{X}^{(v)})^T\textbf{W}^{(v)}=0.
\end{align}

Then the variable $\textbf{W}^{(v)}$ can be updated as
\begin{align}
\label{eq:3-34}
\textbf{W}^{(v)}=&[\textbf{X}^{(v)}(\textbf{X}^{(v)})^T+\gamma\textbf{X}^{(v)}\textbf{L}(\textbf{X}^{(v)})^T\notag\\
&+\eta\textbf{D}^{(v)}]^{-1}\textbf{X}^{(v)}\textbf{H}(\textbf{B}^{(v)})^T.
\end{align}

Regarding the monotonicity, Theorem 2 is given below.
\begin{theorem}
\label{the:W}
In iteration $t+1$, $\mathcal{L}(\textbf{W}^{(v)}_{(t+1)})\le\mathcal{L}(\textbf{W}^{(v)}_{(t)})$ after updating with Eq. \eqref{eq:3-34}.
\end{theorem}
Please see supplementary material for detailed proof.

\subsubsection{Update $\textbf{B}^{(v)}$}
With other variables fixed, the subproblem that only relates to $\textbf{B}^{(v)}$ can be written as
\begin{align}
\label{eq:322-1}
\min\limits_{\textbf{B}^{(v)}}\ &\mathcal{L}(\textbf{B}^{(v)})=||(\textbf{W}^{(v)})^T\textbf{X}^{(v)}-\textbf{B}^{(v)}\textbf{H}^T||_F^2\notag\\
s.t.\ &(\textbf{B}^{(v)})^T\textbf{B}^{(v)}=\textbf{I}.
\end{align}
According to \cite{OMF}, the solution of \eqref{eq:322-1} is
\begin{align}
\label{eq:3-n36}
\textbf{B}^{(v)}=\textbf{V}_{\textbf{B}^{(v)}}\textbf{I}_{d^{(v)}\times c}\textbf{U}_{\textbf{B}^{(v)}}^T,
\end{align}
where $\textbf{U}_{\textbf{B}^{(v)}}$ and $\textbf{V}_{\textbf{B}^{(v)}}$ are obtained from the singular value decomposition (SVD) of $\textbf{H}^T(\textbf{X}^{(v)})^T\textbf{W}^{(v)}=\textbf{U}_{\textbf{B}^{(v)}}\bm\Sigma\textbf{V}_{\textbf{B}^{(v)}}^T$.

The subproblem \eqref{eq:322-1} is an \emph{orthogonal Procrustes problem} (OPP) \cite{OMF}, so the updating of $\textbf{B}^{(v)}$ in each iteration can make the objective value decrease monotonically. Please see supplementary material for further analysis.

\subsubsection{Update $\textbf{Z}$}
With other variables fixed, the subproblem that only relates to $\textbf{Z}$ can be written as
\begin{align}
\min\limits_\textbf{Z}~\mathcal{L}(\textbf{Z})=||\textbf{H}-\textbf{Z}||_F^2\quad s.t.~\textbf{Z}\ge0.
\end{align}
The solution to the subproblem is easily written as
\begin{align}
\label{eq:3-38}
\textbf{Z}=(z_{ij})_{n\times c},\quad z_{ij}=max(h_{ij},0),
\end{align}
where $h_{ij}$ represents $(i,j)$-th entry of the matrix $\textbf{H}$.
It is obvious that $\mathcal{L}(\textbf{Z}_{(t+1)})\le\mathcal{L}(\textbf{Z}_{(t)})$ in iteration $t+1$.

\subsubsection{Update $\textbf{H}$}
With other variables fixed, the subproblem that only relates to $\textbf{H}$ can be written as
\begin{align}
\label{eq:323-1}
\min\limits_{\textbf{H}}\ \mathcal{L}(\textbf{H})&=\sum\limits_{v=1}^V||(\textbf{W}^{(v)})^T\textbf{X}^{(v)}-\textbf{B}^{(v)}\textbf{H}^T||_F^2+\alpha||\textbf{H}-\textbf{Z}||_F^2\notag\\
&\ s.t.\ ~\textbf{H}^T\textbf{H}=\textbf{I}.
\end{align}
It can further be written as
\begin{align}
\mathcal{L}(\textbf{H})=&\sum\limits_{v=1}^VTr\{[(\textbf{W}^{(v)})^T\textbf{X}^{(v)}-\textbf{B}^{(v)}\textbf{H}^T][(\textbf{W}^{(v)})^T\textbf{X}^{(v)}\notag\\
&-\textbf{B}^{(v)}\textbf{H}^T]^T\}+\alpha Tr[(\textbf{H}-\textbf{Z})(\textbf{H}-\textbf{Z})^T]\notag\\
=&Tr\{\sum\limits_{v=1}^V[(\textbf{W}^{(v)})^T\textbf{X}^{(v)}(\textbf{X}^{(v)})^T\textbf{W}^{(v)}\notag\\
&-2(\textbf{X}^{(v)})^T\textbf{W}^{(v)}\textbf{B}^{(v)}\textbf{H}^T+\textbf{B}^{(v)}\textbf{H}^T\textbf{H}(\textbf{B}^{(v)})^T]\notag\\
&+\alpha(\textbf{H}^T\textbf{H}-2\textbf{Z}\textbf{H}^T+\textbf{Z}\textbf{Z}^T)\},
\end{align}
Then we have
\begin{align}
&\textbf{H}=\mathop{\arg\min}\limits_{\textbf{H}^T\textbf{H}=\textbf{I}}\mathcal{L}(\textbf{H})\notag\\
\Rightarrow\ &\textbf{H}=\mathop{\arg\min}\limits_{\textbf{H}^T\textbf{H}=\textbf{I}}Tr[-\sum\limits_{v=1}^V(\textbf{X}^{(v)})^T\textbf{W}^{(v)}\textbf{B}^{(v)}\textbf{H}^T-\alpha\textbf{Z}\textbf{H}^T]\notag\\
\Rightarrow\ &\textbf{H}=\mathop{\arg\min}\limits_{\textbf{H}^T\textbf{H}=\textbf{I}}-Tr\{[\sum\limits_{v=1}^V(\textbf{X}^{(v)})^T\textbf{W}^{(v)}\textbf{B}^{(v)}+\alpha\textbf{Z}]\textbf{H}^T\}\notag\\
\label{eq:40}
\Rightarrow\ &\textbf{H}=\mathop{\arg\min}\limits_{\textbf{H}^T\textbf{H}=\textbf{I}}||\textbf{H}-[\sum\limits_{v=1}^V(\textbf{X}^{(v)})^T\textbf{W}^{(v)}\textbf{B}^{(v)}+\alpha\textbf{Z}]||_F^2\\
\label{eq:323-3}
\Rightarrow\ &\textbf{H}=\textbf{V}_\textbf{H}\textbf{I}_{n\times c}\textbf{U}_\textbf{H}^T.
\end{align}
Similar to updating $\textbf{B}^{(v)}$, according to \cite{OMF}, Eq.~\eqref{eq:323-3} is the solution to \eqref{eq:323-1}, where $\textbf{U}_\textbf{H}$ and $\textbf{V}_\textbf{H}$ are obtained from the SVD of $[\sum\limits_{v=1}^V(\textbf{X}^{(v)})^T\textbf{W}^{(v)}\textbf{B}^{(v)}+\alpha\textbf{Z}]^T=\textbf{U}_\textbf{H}\bm\Sigma\textbf{V}_\textbf{H}^T$.

The subproblem \eqref{eq:40} is also an OPP. Therefore, we can have $\mathcal{L}(\textbf{H}_{(t+1)})\le\mathcal{L}(\textbf{H}_{(t)})$ for iteration $t+1$.

\subsubsection{Update $\textbf{S}$}
With other variables fixed, let $\textbf{Y}^{(v)}=(\textbf{W}^{(v)})^T\textbf{X}^{(v)}\in\mathbb{R}^{d^{(v)}\times n}$, then the subproblem that only relates to $\textbf{S}$ can be written as
\begin{align}
\min\limits_\textbf{S}\ \mathcal{L}(\textbf{S})&=\sum\limits_{v=1}^V[\gamma Tr(\textbf{Y}^{(v)}\textbf{L}(\textbf{Y}^{(v)})^T)+\beta||\textbf{S}-\delta^{(v)}\textbf{A}^{(v)}||_F^2]\notag\\
&s.t.\ ~\textbf{S1}=\textbf{1},\textbf{S}\ge0.
\end{align}
where $\textbf{L}=\textbf{P}-\bar{\textbf{S}}$ and $\bar{\textbf{S}}=(\textbf{S}^T+\textbf{S})/2$. For convenience, assuming that $\bar{\textbf{S}}=\textbf{S}$, the objective function can be written as
\begin{align}
&\min\limits_\textbf{S}~\sum_{v,i,j}[\frac{\gamma}{2}||\textbf{y}_{\cdot i}^{(v)}-\textbf{y}_{\cdot j}^{(v)}||_2^2s_{ij}+\beta(s_{ij}-\delta^{(v)}a_{ij}^{(v)})^2].
\end{align}
Let $g_{ij}^{(v)}=||\textbf{y}_{\cdot i}^{(v)}-\textbf{y}_{\cdot j}^{(v)}||_2^2$, then it is equivalent to
\begin{align}
&\min\limits_\textbf{S}~\sum_{i,j}\sum_{v}[\frac{\gamma}{2\beta}g_{ij}^{(v)}s_{ij}+s_{ij}^2-2s_{ij}\delta^{(v)}a_{ij}^{(v)}]\notag\\
\Leftrightarrow&\min\limits_\textbf{S}~\sum_{i,j}[s_{ij}(\frac{\gamma}{2\beta}\sum_{v}g_{ij}^{(v)}-2\sum_{v}\delta^{(v)}a_{ij}^{(v)})+Vs_{ij}^2]\notag\\
\Leftrightarrow&\min\limits_\textbf{S}~\sum_{i,j}(s_{ij}-\frac{2\sum_{v}\delta^{(v)}a_{ij}^{(v)}-\frac{\gamma}{2\beta}\sum_{v}g_{ij}^{(v)}}{2V})^2\notag\\
\Leftrightarrow&\min\limits_\textbf{S}~\sum_{i}||\textbf{s}_{i\cdot}-\frac{2\sum_{v}\delta^{(v)}\textbf{a}_{i\cdot}^{(v)}-\frac{\gamma}{2\beta}\sum_{v}\textbf{g}_{i\cdot}^{(v)}}{2V}||_2^2.\notag\\
\end{align}
Let $\textbf{r}_{i\cdot}=\frac{2\sum_{v}\delta^{(v)}\textbf{a}_{i\cdot}^{(v)}-\frac{\gamma}{2\beta}\sum_{v}\textbf{g}_{i\cdot}^{(v)}}{2V}$, the subproblem can be re-written as
\begin{equation}
\label{eq:3-nn46}
{\rm for~each~}i~\left\{
\begin{aligned}
&\min~||\textbf{s}_{i\cdot}-\textbf{r}_{i\cdot}||_2^2\\
&~s.t.~~\textbf{s}_{i\cdot}\textbf{1}=1,\textbf{s}_{i\cdot}\ge0.
\end{aligned}
\right.
\end{equation}
The subproblem \eqref{eq:3-nn46} can be solved with a closed form solution according to \cite{Anew}.

Regarding the monotonicity, Theorem 3 is given below. Specifically, if the affinity matrices $\textbf{A}^{(v)}$ (for $v=1,\cdots,V$) are initialized to be symmetric, then the objective function monotonically decreases when updating $\textbf{S}$.

\begin{theorem}
\label{the:S}
If $\textbf{A}^{(v)}$ is symmetric for each $v$, in iteration $t+1$, $\mathcal{L}(\textbf{S}_{(t+1)})\le\mathcal{L}(\textbf{S}_{(t)})$ after solving the subproblem \eqref{eq:3-nn46}.
\end{theorem}
Please see supplementary material for detailed proof.

\begin{table}[!h]
	\centering
	\begin{tabular}{m{8.4cm}}
		\toprule
		\textbf{Algorithm 1} Joint Multi-view Unsupervised Feature Selection and Graph Learning (JMVFG)\\
		\midrule
		\textbf{Input:} Multi-view data matrices $\{\textbf{X}^{(v)}\}_{v=1}^V$, the parameters $\eta$, $\beta$, and $\gamma$, and the number of clusters $c$.\\
		\textbf{Preparation:} The number of nearest neighbors $K=5$. The number of features after projection $d^{(1)}=d^{(2)}=...=d^{(V)}=c$. Maximum and minimum normalization of $\{\textbf{X}^{(v)}\}_{v=1}^V$. Construction of $\textbf{A}^{(v)}$.\\
		\textbf{Initialization:} $\bm\delta=\textbf{1}/V$. Initialize $\textbf{S}=\sum_v\delta^{(v)}\textbf{A}^{(v)}$. Use $k$-means to initialize $\textbf{H}$. The feature selection matrix $\textbf{W}^{(v)}=[\textbf{I}_{d^{(v)}\times d^{(v)}}~\textbf{0}]^T$. Initialize $\textbf{D}^{(v)}$ as an identity matrix. Initialize $\textbf{B}^{(v)}$ by Eq.~\eqref{eq:3-n36}.\\
		\textbf{repeat}\\
		\begin{enumerate}
			\item For each $v$, update $\delta^{(v)}$ by Eq.~\eqref{eq:3-44}.
			\item For each $v$, update $\textbf{W}^{(v)}$ by Eq.~\eqref{eq:3-34}, and update $\textbf{D}^{(v)}$ with the diagonal entry $d_{ii}^{(v)}=1/(2||\textbf{W}^{(v)}_{i\cdot}||_2)$.
			\item For each $v$, update $\textbf{B}^{(v)}$ by \eqref{eq:3-n36}.
			\item Update $\textbf{Z}$ by Eq.~\eqref{eq:3-38}.
			\item Update $\textbf{H}$ by Eq.~\eqref{eq:323-3}.
			\item Update $\textbf{S}$ by solving subproblem \eqref{eq:3-nn46}.
		\end{enumerate}
		\textbf{until} Convergence\\
		\textbf{Output:} For the feature selection task, calculate the feature scores $||\textbf{w}_{i\cdot}^{(v)}||_2^2$, $i=1,2,...,m^{(v)}$, for each $v$. The corresponding feature indexes are obtained after the feature scores are sorted from high to low. For the graph learning task, the similarity graph $\textbf{S}$ is obtained and can be used for spectral clustering.\\
		\bottomrule
	\end{tabular}
\end{table}

\subsection{Theoretical Convergence Analysis}
\label{sec:theoretical_convergence}
From the theoretical analysis of the monotonicity at the end of each subsection in Section~\ref{sec:optimize_algo}, we know that the overall objective function \eqref{eq:32-2} monotonically decreases in the updating of each variable under Algorithm 1, except that the monotonicity of updating $\textbf{S}$ requires the affinity matrix $\textbf{A}^{(v)}$ to be symmetric, which can be easily satisfied (e.g., by using a symmetric $K$-NN graph).  Despite this,  we empirically observe that even when $\textbf{A}^{(v)}$ is not initialized to be symmetric (e.g., by performing the row normalization on $\textbf{A}^{(v)}$), the proposed algorithm can still achieve fast and good convergence.
The convergence condition of JMVFG can be defined as $|(obj_{i-1}-obj_i)/obj_{i-1}|\le\epsilon$, where $obj_i$ is the objective function value of problem \eqref{eq:3-11} in the $i$-th iteration, and $\epsilon>0$ is a small constant.
Please see supplementary material for more theoretical details, and Section~\ref{sec:empirical_convergence} for empirical evaluations.

\subsection{Computational Complexity Analysis}
\label{sec:theoretical_complexity}
This section analyzes the computational complexity of the proposed JMVFG algorithm.

Let $T$ denote the number of iterations. For each iteration, the computational complexity of updating $\bm\delta$ is $\mathcal{O}(n^2V)$. Note that $d^{(1)}=d^{(2)}...=d^{(V)}=c$. The cost of updating $\textbf{W}^{(v)}$ (for all $v$) is $\mathcal{O}(({m^{(1)}}^3+...+{m^{(V)}}^3)+({m^{(1)}}^2+...+{m^{(V)}}^2)n+(m^{(1)}+...+m^{(V)})n^2)$. The cost of updating $\textbf{D}^{(v)}$ (for all $v$) is $\mathcal{O}((m^{(1)}+...+m^{(V)})c)$. The matrix manipulations and singular value decomposition for updating $\textbf{B}^{(v)}$ (for all $v$) take $\mathcal{O}((m^{(1)}+...+m^{(V)})(nc+c^2)+c^3V)$ time. It takes $\mathcal{O}(nc)$ time to update $\textbf{Z}$, and $\mathcal{O}((m^{(1)}+...+m^{(V)})nc+n^2c)$ time to update $\textbf{H}$.
The costs of updating $\textbf{S}$ and $\textbf{L}$ are $\mathcal{O}(n\log n)$ and $\mathcal{O}(n^2)$, respectively. With $c,V\ll n,m^{(1)},...,m^{(V)}$, the computational complexity of the JMVFG algorithm can be written as $\mathcal{O}(T(({m^{(1)}}^3+...+{m^{(V)}}^3)+({m^{(1)}}^2+...+{m^{(V)}}^2)n+(m^{(1)}+...+m^{(V)})n^2))$.

\section{Experiments}
\label{sec:experiments}
In this section, we conduct experiments to compare the proposed JMVFG approach against the other multi-view unsupervised feature selection approaches for \emph{the feature selection task}, and against the other multi-view clustering approaches for \emph{the clustering task} (as multi-view graph learning can naturally lead to multi-view clustering by performing spectral clustering on the learned graph).

\subsection{Datasets and Evaluation Measures}
In our experiments, eight real-world multi-view datasets are used. Specifically, the MSRC-v1 dataset \cite{CvLP-DCL} consists of $210$ images, which includes seven object classes (such as building, tree, cow, airplane, car and so forth) and is associated with four views, namely, CM(24D), GIST(512D), LBP(256D), and GENT(254D).
The ORL dataset \cite{MAMFS} consists of $400$ face images from $40$ people with varying facial expressions, angles, illuminations, taking times, and facial wears, which is associated with three views, namely,   Intensity(4096D), LBP(3304D), and Gabor(6750D). The WebKB-Texas dataset is a text dataset, consisting of 187 documents with two views (which are 187-dimensional and 1703-dimensional, respectively). The Caltech101 dataset \cite{li04_caltech101} consists of 9144 images captured for the object recognition problem. In the experiments, we use two widely-used subsets of Caltech101, i.e., Caltech101-7 and Caltech101-20, which include $7$ classes and $20$ classes, respectively, and are associated with three views, namely, GIST(512D), HOG(1984D), and LBP(928D)
The Handwritten dataset \cite{LMVSC} consists of $2000$ images of handwritten digits, which includes ten classes and is associated with six views, namely, PIX(240D), FOU(76D), FAC(216D), ZER(47D), KAR(64D), and MOR(6D). The Mfeat dataset \cite{lin18_dasfaa} is another widely-used version of the Handwritten dataset, consisting of three views, namely, FOU(76D), FAC(216D), and ZER(47D). The Outdoor-Scene dataset \cite{ACSL} consists of $2688$ color images, which includes eight outdoor scene categories and is associated with four views, namely, GIST(512D), HOG(432D), LBP(256D), and GABOR(48D).

Notably, our framework can simultaneously perform multi-view unsupervised feature selection and multi-view graph learning. Therefore, its performance will be evaluated for two corresponding tasks.
For the first task of multi-view unsupervised feature selection, a commonly-adopted protocol \cite{OMVFS,ACSL} is to perform $k$-means on the data matrix with the selected features (by different multi-view feature selection methods) and then evaluate the quality of the selected features by the clustering performance. For the second task of multi-view clustering via graph learning, a general practice \cite{MVGL,Liang2022} is to perform spectral clustering on the learned graph and then evaluate the clustering result.

For both tasks, to compare the clustering results, three widely-used evaluation measures are adopted, namely,  normalized mutual information (NMI) \cite{Huang2018}, accuracy (ACC) \cite{zhang20_pami}, and purity (PUR) \cite{ch20_tcyb}. For all the three evaluation measures, larger values indicate better performance.

\begin{table*}[!t]
\centering
\caption{Average performances ($mean\%_{\pm std\%}$) over 20 runs by different multi-view unsupervised feature selection algorithms.
On each dataset, the best score is highlighted in bold.}
\label{table:FS}

\begin{tabular}{m{0.75cm}<{\centering}m{0.93cm}<{\centering}m{1.25cm}<{\centering}m{1.15cm}<{\centering}m{1.25cm}<{\centering}m{1.55cm}<{\centering}m{1.66cm}<{\centering}m{1.15cm}<{\centering}m{1.3cm}<{\centering}m{1.15cm}<{\centering}m{1.12cm}<{\centering}}
\toprule
Metric &Method &MSRC-v1 &ORL &WebKB-Texas  &Caltech101-7  &Caltech101-20 &Mfeat &Handwritten  &Outdoor-Scene  &Avg. rank\\
\midrule
\multirow{5}{*}{NMI}
&ALLfea	&$56.74_{\pm5.69}^*$	&$75.17_{\pm2.50}^*$	&$14.55_{\pm10.90}^*$	&$47.37_{\pm2.94}^*$	&$57.27_{\pm1.79}^*$	&$72.83_{\pm4.66}^*$	&$72.12_{\pm4.01}^*$	&$53.65_{\pm1.59}^*$	&4.25\\
&OMVFS	&$41.79_{\pm2.71}^*$	&$75.26_{\pm2.02}^*$	&$28.03_{\pm4.64}^*$	&$50.55_{\pm3.28}^*$	&$56.98_{\pm1.95}^*$	&$68.91_{\pm4.48}^*$	&$66.53_{\pm2.65}^*$	&$47.30_{\pm1.82}^*$	&4.50\\
&ACSL	&$74.90_{\pm4.66}$	&$78.55_{\pm1.82}$	&$\textbf{32.93}_{\pm3.71}$	&$54.54_{\pm3.67}$	&$60.05_{\pm1.64}$	&$80.19_{\pm3.89}$	&$85.73_{\pm3.09}$	&$55.85_{\pm1.45}$	&1.88\\
&NSGL	&$59.93_{\pm5.04}^*$	&$76.42_{\pm1.90}^*$	&$32.06_{\pm3.45}$	&$49.22_{\pm3.14}^*$	&$59.31_{\pm1.62}^*$	&$74.84_{\pm3.47}^*$	&$74.40_{\pm3.84}^*$	&$52.06_{\pm0.56}^*$	&3.13\\
&\textbf{JMVFG}	&$\textbf{74.94}_{\pm3.59}$	&$\textbf{78.90}_{\pm2.43}$	&$31.55_{\pm5.03}$	&$\textbf{57.74}_{\pm6.77}$	&$\textbf{60.49}_{\pm1.62}$	&$\textbf{81.83}_{\pm3.58}$	&$\textbf{87.43}_{\pm5.13}$	&$\textbf{56.00}_{\pm1.61}$	&\textbf{1.25}\\
\midrule
\multirow{5}{*}{ACC}
&ALLfea	&$64.88_{\pm8.28}^*$	&$56.40_{\pm4.01}^*$	&$55.67_{\pm6.85}^*$	&$50.16_{\pm5.98}^*$	&$43.29_{\pm3.52}^*$	&$73.84_{\pm9.38}^*$	&$70.23_{\pm7.13}^*$	&$61.87_{\pm4.82}^*$	&4.25\\
&OMVFS	&$48.21_{\pm4.40}^*$	&$56.15_{\pm3.81}^*$	&$57.27_{\pm8.12}$	&$55.61_{\pm6.13}^*$	&$45.65_{\pm5.07}^*$	&$70.28_{\pm9.11}^*$	&$65.48_{\pm4.43}^*$	&$58.28_{\pm2.54}^*$	&4.50\\
&ACSL	&$\textbf{80.64}_{\pm6.46}$	&$\textbf{61.03}_{\pm3.18}$	&$62.33_{\pm4.05}$	&$58.32_{\pm5.24}$	&$49.46_{\pm4.38}$	&$81.96_{\pm6.47}$	&$\textbf{85.87}_{\pm7.76}$	&$67.53_{\pm5.59}$	&\textbf{1.63}\\
&NSGL	&$70.52_{\pm5.71}^*$	&$58.31_{\pm3.37}$	&$\textbf{62.57}_{\pm5.08}$	&$53.75_{\pm5.20}^*$	&$46.72_{\pm4.15}^*$	&$80.32_{\pm6.61}$	&$76.01_{\pm6.88}^*$	&$61.25_{\pm3.45}^*$	&3.00\\
&\textbf{JMVFG}	&$80.55_{\pm7.14}$	&$60.74_{\pm4.23}$	&$59.76_{\pm4.44}$	&$\textbf{62.54}_{\pm9.73}$	&$\textbf{50.78}_{\pm4.33}$	&$\textbf{83.45}_{\pm8.12}$	&$85.81_{\pm10.33}$	&$\textbf{68.37}_{\pm4.84}$	&\textbf{1.63}\\
\midrule
\multirow{5}{*}{PUR}
&ALLfea	&$68.38_{\pm7.41}^*$	&$61.59_{\pm3.46}^*$	&$61.28_{\pm5.68}^*$	&$85.79_{\pm1.36}^*$	&$76.67_{\pm1.33}^*$	&$77.30_{\pm7.03}^*$	&$73.38_{\pm5.93}^*$	&$64.97_{\pm2.92}^*$	&4.13\\
&OMVFS	&$51.71_{\pm3.92}^*$	&$60.75_{\pm3.23}^*$	&$68.34_{\pm3.13}^*$	&$87.70_{\pm1.80}^*$	&$76.30_{\pm1.68}^*$	&$73.90_{\pm6.50}^*$	&$70.07_{\pm4.13}^*$	&$59.09_{\pm1.28}^*$	&4.63\\
&ACSL	&$82.19_{\pm5.28}$	&$\textbf{65.40}_{\pm3.21}$	&$70.51_{\pm2.04}$	&$89.33_{\pm1.69}$	&$78.51_{\pm1.50}^*$	&$83.75_{\pm5.05}$	&$88.36_{\pm5.47}$	&$69.23_{\pm3.46}$	&2.00\\
&NSGL	&$72.57_{\pm5.44}^*$	&$63.01_{\pm2.94}^*$	&$70.61_{\pm1.89}$	&$86.85_{\pm2.24}^*$	&$78.22_{\pm1.24}^*$	&$81.29_{\pm5.82}^*$	&$78.04_{\pm6.41}^*$	&$62.50_{\pm1.52}^*$	&3.13\\
&\textbf{JMVFG}	&$\textbf{82.21}_{\pm5.56}$	&$65.34_{\pm3.52}$	&$\textbf{71.31}_{\pm3.75}$	&$\textbf{90.01}_{\pm1.21}$	&$\textbf{79.41}_{\pm1.26}$	&$\textbf{85.66}_{\pm5.90}$	&$\textbf{88.79}_{\pm7.61}$	&$\textbf{69.56}_{\pm3.04}$	&\textbf{1.13}\\
\bottomrule
\end{tabular}

\begin{tablenotes}
		\footnotesize
		\item[1] * The symbol ``*'' indicates statistically significant improvement (of JMVFG over a method on this dataset) w.r.t. Student's t-test with $p<0.05$.
	\end{tablenotes}

\end{table*}

\subsection{Baseline Methods and Experimental Settings}

For the multi-view unsupervised feature selection task, we compare our JMVFG method with four baseline methods, namely, $k$-means with all features (\textbf{ALLfea}), online unsupervised multi-view feature selection (\textbf{OMVFS}) \cite{OMVFS}, adaptive collaborative similarity learning for unsupervised multi-view feature selection (\textbf{ACSL}) \cite{ACSL},
and multi-view feature selection via nonnegative structured graph learning (\textbf{NSGL}) \cite{NSGL}.

For the multi-view clustering task, we compare our JMVFG method with thirteen baseline methods, namely,
multi-view spectral clustering (\textbf{MVSC}) \cite{MVSC},
diversity-induced multi-view subspace clustering (\textbf{DiMSC}) \cite{DiMSC},
auto-weighted multiple graph learning (\textbf{AMGL}) \cite{AMGL},
multi-view learning with adaptive neighbors (\textbf{MLAN}) \cite{MLAN},
self-weighted multi-view clustering (\textbf{SwMC}) \cite{SwMC},
highly-economized scalable image clustering (\textbf{HSIC}) \cite{HSIC},
multi-view clustering with graph learning (\textbf{MVGL}) \cite{MVGL},
multi-view consensus graph clustering (\textbf{MCGC}) \cite{MCGC},
similarity graph fusion (\textbf{SGF}) \cite{Liang2022},
binary multi-view clustering (\textbf{BMVC}) \cite{BMVC},
large-scale multi-view subspace clustering (\textbf{LMVSC}) \cite{LMVSC},
scalable multi-view subspace clustering (\textbf{SMVSC}) \cite{SMVSC},
and fast parameter-free multi-view subspace clustering with consensus anchor guidance (\textbf{FPMVS-CAG}) \cite{FPMVS-CAG}.

For all test methods, their parameters are tuned in the range of $\{10^{-3},10^{-2},...,10^{3}\}$, unless the value  (or range) of the parameter is specified by the corresponding paper. If the $K$-NN graph is involved in a method,  the number of nearest neighbors $K=5$ will be used. For each experiment, the average performance over 20 runs will be reported.

\begin{figure}[!t]
	\begin{center}
		{\subfigure[{\scriptsize MSRC-v1}]
			{\includegraphics[width=0.24\columnwidth]{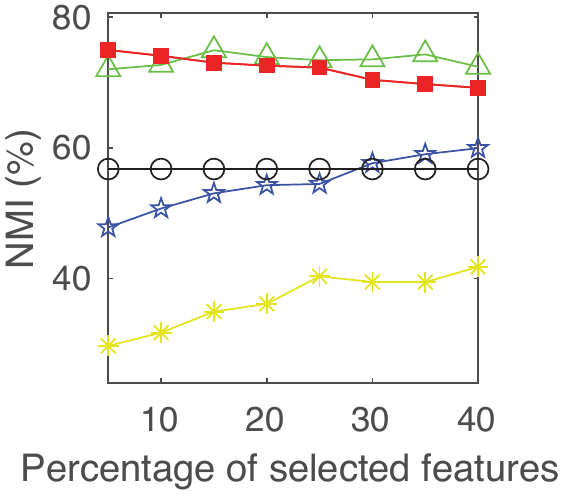}}}
		{\subfigure[{\scriptsize ORL}]
			{\includegraphics[width=0.24\columnwidth]{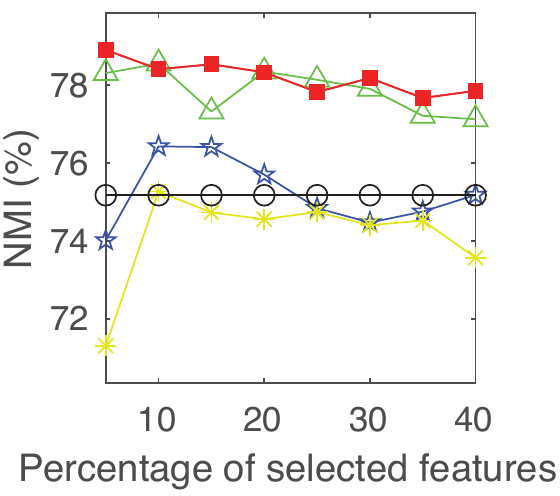}}}
		{\subfigure[{\scriptsize WebKB-Texas}]
			{\includegraphics[width=0.24\columnwidth]{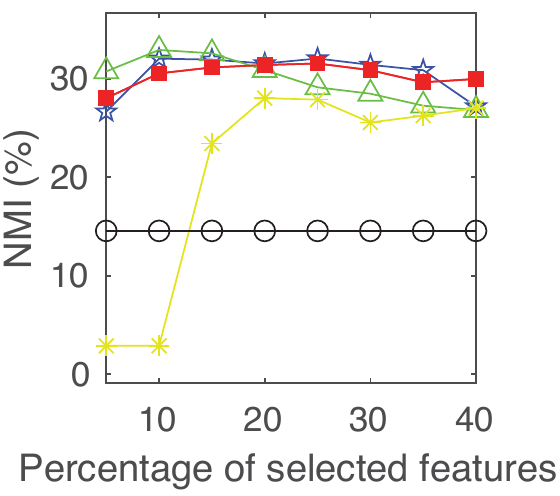}}}
		{\subfigure[{\scriptsize Caltech101-7}]
			{\includegraphics[width=0.24\columnwidth]{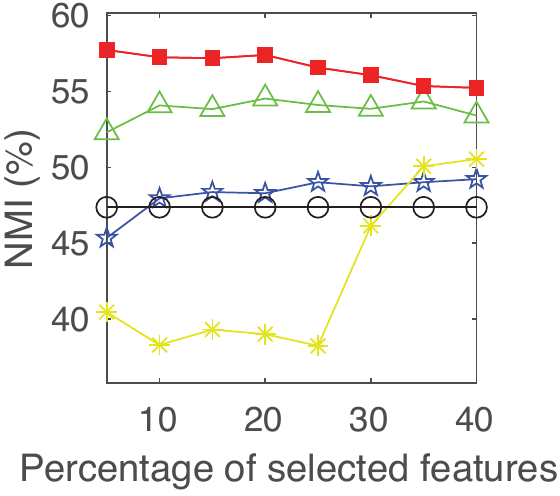}}}\\
		{\subfigure[{\scriptsize Caltech101-20}]
			{\includegraphics[width=0.24\columnwidth]{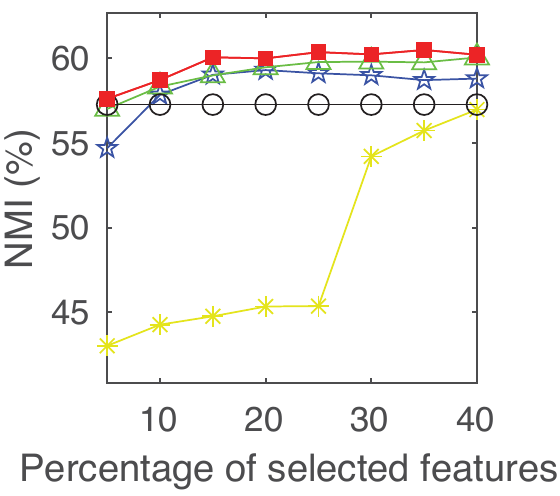}}}
		{\subfigure[{\scriptsize Mfeat}]
			{\includegraphics[width=0.24\columnwidth]{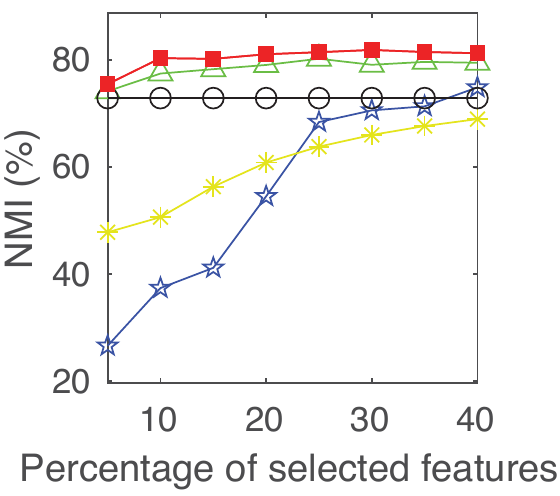}}}
		{\subfigure[{\scriptsize Handwritten}]
			{\includegraphics[width=0.24\columnwidth]{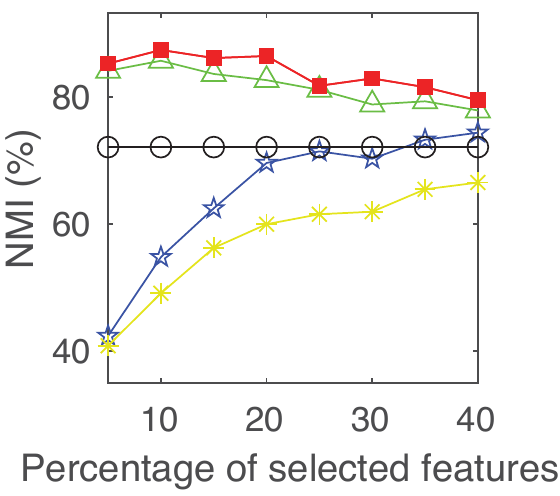}}}
		{\subfigure[{\scriptsize Outdoor-Scene}]
			{\includegraphics[width=0.24\columnwidth]{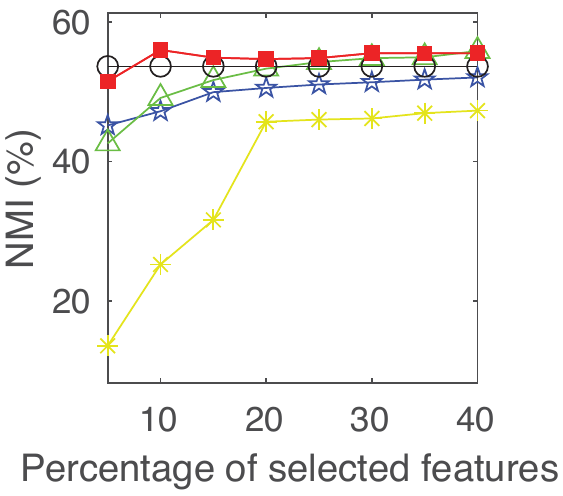}}}
		{\subfigure
			{\includegraphics[width=0.68\columnwidth]{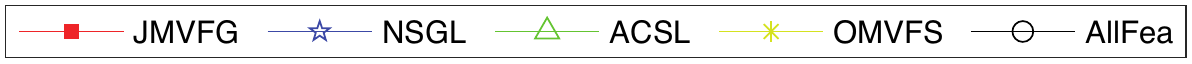}}}
		\caption{Average NMI(\%) scores of different multi-view unsupervised feature selection methods with different percentages of selected features.}
		\label{fig:NMI}
	\end{center}
\end{figure}

\begin{table*}[!th]
	\centering
	\newcommand{\tabincell}[2]{\begin{tabular}{@{}#1@{}}#2\end{tabular}}
	\caption{Average performances ($mean\%_{\pm std\%}$) over 20 runs by different multi-view clustering algorithms.
		The best score in each row is highlighted in bold.}
	\label{t:CL}
	\begin{tabular}{m{0.55cm}<{\centering}m{1.7cm}<{\centering}m{0.6cm}<{\centering}m{0.6cm}<{\centering}m{0.6cm}<{\centering}m{0.6cm}<{\centering}m{0.6cm}<{\centering}m{0.6cm}<{\centering}m{0.6cm}<{\centering}m{0.6cm}<{\centering}m{0.6cm}<{\centering}m{0.6cm}<{\centering}m{0.7cm}<{\centering}m{0.7cm}<{\centering}m{0.8cm}<{\centering}m{0.8cm}<{\centering}}
		\toprule
		Metric	&Datasets	&MVSC	&DiMSC	&AMGL	&MLAN	&SwMC	&HSIC	&MVGL	&MCGC	&SGF	&BMVC	&LMVSC	&SMVSC	&FPMVS-CAG	&\textbf{JMVFG}\\
		\midrule
		\multirow{10}{*}{NMI}
		&MSRC-v1	&\tabincell{c}{$49.34^*$\\$^{\pm3.78}$}	&\tabincell{c}{$64.09^*$\\$^{\pm3.32}$}	&\tabincell{c}{$58.84^*$\\$^{\pm6.33}$}	&\tabincell{c}{$73.46^*$\\$^{\pm0.37}$}	&\tabincell{c}{$62.69^*$\\$^{\pm0.00}$}	&\tabincell{c}{$34.18^*$\\$^{\pm0.00}$}	&\tabincell{c}{$54.95^*$\\$^{\pm0.00}$}	&\tabincell{c}{$63.14^*$\\$^{\pm0.00}$}	&\tabincell{c}{$70.20^*$\\$^{\pm0.01}$}	&\tabincell{c}{$36.73^*$\\$^{\pm0.00}$}	&\tabincell{c}{$24.60^*$\\$^{\pm0.00}$}	&\tabincell{c}{$61.09^*$\\$^{\pm0.00}$}	&\tabincell{c}{$56.62^*$\\$^{\pm0.00}$}	&\tabincell{c}{$\textbf{76.87}$\\$^{\pm0.00}$}\\
& ORL	&\tabincell{c}{$84.69^*$\\$^{\pm2.06}$}	&\tabincell{c}{$90.63^*$\\$^{\pm1.73}$}	&\tabincell{c}{$85.29^*$\\$^{\pm1.78}$}	&\tabincell{c}{$78.58^*$\\$^{\pm0.16}$}	&\tabincell{c}{$83.31^*$\\$^{\pm0.00}$}	&\tabincell{c}{$36.74^*$\\$^{\pm0.00}$}	&\tabincell{c}{$85.59^*$\\$^{\pm0.00}$}	&\tabincell{c}{$89.41^*$\\$^{\pm0.00}$}	&\tabincell{c}{$91.28^*$\\$^{\pm0.37}$}	&\tabincell{c}{$36.54^*$\\$^{\pm0.00}$}	&\tabincell{c}{$76.26^*$\\$^{\pm0.00}$}	&\tabincell{c}{$75.26^*$\\$^{\pm0.00}$}	&\tabincell{c}{$74.30^*$\\$^{\pm0.00}$}	&\tabincell{c}{$\textbf{92.46}$\\$^{\pm0.00}$}\\
&WebKB-Texas	&\tabincell{c}{$1.98^*$\\$^{\pm0.25}$}	&\tabincell{c}{$17.46^*$\\$^{\pm1.13}$}	&\tabincell{c}{$9.89^*$\\$^{\pm2.61}$}	&\tabincell{c}{$7.11^*$\\$^{\pm1.72}$}	&\tabincell{c}{$9.37^*$\\$^{\pm0.00}$}	&\tabincell{c}{$19.33^*$\\$^{\pm0.00}$}	&\tabincell{c}{$6.55^*$\\$^{\pm0.00}$}	&\tabincell{c}{$19.85^*$\\$^{\pm0.00}$}	&\tabincell{c}{$24.97^*$\\$^{\pm0.05}$}	&\tabincell{c}{$23.71^*$\\$^{\pm0.00}$}	&\tabincell{c}{$20.86^*$\\$^{\pm0.00}$}	&\tabincell{c}{$21.42^*$\\$^{\pm0.00}$}	&\tabincell{c}{$22.76^*$\\$^{\pm0.00}$}	&\tabincell{c}{$\textbf{25.69}$\\$^{\pm0.00}$}\\
&Caltech101-7	&\tabincell{c}{$54.46^*$\\$^{\pm11.37}$}	&\tabincell{c}{$50.83^*$\\$^{\pm1.68}$}	&\tabincell{c}{$50.56^*$\\$^{\pm8.41}$}	&\tabincell{c}{$63.58^*$\\$^{\pm0.00}$}	&\tabincell{c}{$52.97^*$\\$^{\pm0.00}$}	&\tabincell{c}{$37.48^*$\\$^{\pm0.00}$}	&\tabincell{c}{$55.98^*$\\$^{\pm0.00}$}	&\tabincell{c}{$50.28^*$\\$^{\pm0.00}$}	&\tabincell{c}{$58.81^*$\\$^{\pm0.00}$}	&\tabincell{c}{$45.51^*$\\$^{\pm0.00}$}	&\tabincell{c}{$43.53^*$\\$^{\pm0.00}$}	&\tabincell{c}{$47.79^*$\\$^{\pm0.00}$}	&\tabincell{c}{$47.32^*$\\$^{\pm0.00}$}	&\tabincell{c}{$\textbf{77.37}$\\$^{\pm0.00}$}\\
&Caltech101-20	&\tabincell{c}{$52.30^*$\\$^{\pm5.02}$}	&\tabincell{c}{$56.00^*$\\$^{\pm1.00}$}	&\tabincell{c}{$53.56^*$\\$^{\pm3.28}$}	&\tabincell{c}{$41.62^*$\\$^{\pm2.65}$}	&\tabincell{c}{$49.88^*$\\$^{\pm0.00}$}	&\tabincell{c}{$58.01^*$\\$^{\pm0.00}$}	&\tabincell{c}{$43.90^*$\\$^{\pm0.00}$}	&\tabincell{c}{$61.04^*$\\$^{\pm0.00}$}	&\tabincell{c}{$64.73^*$\\$^{\pm0.65}$}	&\tabincell{c}{$59.73^*$\\$^{\pm0.00}$}	&\tabincell{c}{$55.40^*$\\$^{\pm0.00}$}	&\tabincell{c}{$55.89^*$\\$^{\pm0.00}$}	&\tabincell{c}{$58.49^*$\\$^{\pm0.00}$}	&\tabincell{c}{$\textbf{69.26}$\\$^{\pm0.00}$}\\
&Mfeat	&\tabincell{c}{$70.06^*$\\$^{\pm6.30}$}	&\tabincell{c}{$69.14^*$\\$^{\pm0.13}$}	&\tabincell{c}{$81.19^*$\\$^{\pm3.40}$}	&\tabincell{c}{$82.65^*$\\$^{\pm2.48}$}	&\tabincell{c}{$87.78^*$\\$^{\pm0.00}$}	&\tabincell{c}{$72.36^*$\\$^{\pm0.00}$}	&\tabincell{c}{$85.02^*$\\$^{\pm0.00}$}	&\tabincell{c}{$78.04^*$\\$^{\pm0.00}$}	&\tabincell{c}{$88.53^*$\\$^{\pm0.00}$}	&\tabincell{c}{$68.30^*$\\$^{\pm0.00}$}	&\tabincell{c}{$63.78^*$\\$^{\pm0.00}$}	&\tabincell{c}{$74.82^*$\\$^{\pm0.00}$}	&\tabincell{c}{$71.63^*$\\$^{\pm0.00}$}	&\tabincell{c}{$\textbf{95.47}$\\$^{\pm0.00}$}\\
&handwritten	&\tabincell{c}{$67.02^*$\\$^{\pm7.72}$}	&\tabincell{c}{$46.64^*$\\$^{\pm0.22}$}	&\tabincell{c}{$84.20^*$\\$^{\pm4.76}$}	&\tabincell{c}{$93.88^*$\\$^{\pm0.09}$}	&\tabincell{c}{$84.16^*$\\$^{\pm0.00}$}	&\tabincell{c}{$23.97^*$\\$^{\pm0.00}$}	&\tabincell{c}{$89.16^*$\\$^{\pm0.00}$}	&\tabincell{c}{$83.12^*$\\$^{\pm0.00}$}	&\tabincell{c}{$88.74^*$\\$^{\pm0.00}$}	&\tabincell{c}{$60.19^*$\\$^{\pm0.00}$}	&\tabincell{c}{$75.60^*$\\$^{\pm0.00}$}	&\tabincell{c}{$77.80^*$\\$^{\pm0.00}$}	&\tabincell{c}{$78.14^*$\\$^{\pm0.00}$}	&\tabincell{c}{$\textbf{96.34}$\\$^{\pm0.00}$}\\
&Outdoor-Scene	&\tabincell{c}{$23.39^*$\\$^{\pm14.73}$}	&\tabincell{c}{$36.19^*$\\$^{\pm0.51}$}	&\tabincell{c}{$44.19^*$\\$^{\pm3.96}$}	&\tabincell{c}{$37.38^*$\\$^{\pm0.00}$}	&\tabincell{c}{$43.69^*$\\$^{\pm0.00}$}	&\tabincell{c}{$38.12^*$\\$^{\pm0.00}$}	&\tabincell{c}{$41.63^*$\\$^{\pm0.00}$}	&\tabincell{c}{$39.93^*$\\$^{\pm0.00}$}	&\tabincell{c}{$52.06^*$\\$^{\pm0.07}$}	&\tabincell{c}{$21.06^*$\\$^{\pm0.00}$}	&\tabincell{c}{$45.97^*$\\$^{\pm0.00}$}	&\tabincell{c}{$51.67^*$\\$^{\pm0.00}$}	&\tabincell{c}{$53.04^*$\\$^{\pm0.00}$}	&\tabincell{c}{$\textbf{63.69}$\\$^{\pm0.00}$}\\
\cline{2-16}
		&Avg. rank	&10.38	&8.38	&7.38	&7.13	&7.38	&10.88	&7.50	&6.38	&2.63	&10.50	&10.13	&7.75	&7.63	&\textbf{1.00}\\
		\midrule
		\multirow{10}{*}{ACC}
		&MSRC-v1	&\tabincell{c}{$56.98^*$\\$^{\pm5.33}$}	&\tabincell{c}{$73.62^*$\\$^{\pm2.91}$}	&\tabincell{c}{$64.90^*$\\$^{\pm8.10}$}	&\tabincell{c}{$73.81^*$\\$^{\pm0.00}$}	&\tabincell{c}{$70.48^*$\\$^{\pm0.00}$}	&\tabincell{c}{$43.33^*$\\$^{\pm0.00}$}	&\tabincell{c}{$60.95^*$\\$^{\pm0.00}$}	&\tabincell{c}{$76.67^*$\\$^{\pm0.00}$}	&\tabincell{c}{$75.36^*$\\$^{\pm0.52}$}	&\tabincell{c}{$49.52^*$\\$^{\pm0.00}$}	&\tabincell{c}{$34.29^*$\\$^{\pm0.00}$}	&\tabincell{c}{$67.14^*$\\$^{\pm0.00}$}	&\tabincell{c}{$60.00^*$\\$^{\pm0.00}$}	&\tabincell{c}{$\textbf{81.90}$\\$^{\pm0.00}$}\\
& ORL	&\tabincell{c}{$72.36^*$\\$^{\pm3.81}$}	&\tabincell{c}{$80.66^*$\\$^{\pm3.33}$}	&\tabincell{c}{$72.61^*$\\$^{\pm2.25}$}	&\tabincell{c}{$68.40^*$\\$^{\pm0.13}$}	&\tabincell{c}{$70.75^*$\\$^{\pm0.00}$}	&\tabincell{c}{$15.50^*$\\$^{\pm0.00}$}	&\tabincell{c}{$72.50^*$\\$^{\pm0.00}$}	&\tabincell{c}{$79.50^*$\\$^{\pm0.00}$}	&\tabincell{c}{$82.59^*$\\$^{\pm0.81}$}	&\tabincell{c}{$16.00^*$\\$^{\pm0.00}$}	&\tabincell{c}{$55.50^*$\\$^{\pm0.00}$}	&\tabincell{c}{$56.25^*$\\$^{\pm0.00}$}	&\tabincell{c}{$56.00^*$\\$^{\pm0.00}$}	&\tabincell{c}{$\textbf{86.25}$\\$^{\pm0.00}$}\\
&WebKB-Texas	&\tabincell{c}{$55.13^*$\\$^{\pm0.30}$}	&\tabincell{c}{$49.97^*$\\$^{\pm0.80}$}	&\tabincell{c}{$51.79^*$\\$^{\pm5.36}$}	&\tabincell{c}{$49.63^*$\\$^{\pm0.59}$}	&\tabincell{c}{$54.55^*$\\$^{\pm0.00}$}	&\tabincell{c}{$56.15^*$\\$^{\pm0.00}$}	&\tabincell{c}{$51.87^*$\\$^{\pm0.00}$}	&\tabincell{c}{$55.61^*$\\$^{\pm0.00}$}	&\tabincell{c}{$59.63^*$\\$^{\pm0.27}$}	&\tabincell{c}{$55.61^*$\\$^{\pm0.00}$}	&\tabincell{c}{$\textbf{60.96}$\\$^{\pm0.00}$}	&\tabincell{c}{$56.68^*$\\$^{\pm0.00}$}	&\tabincell{c}{$57.75^*$\\$^{\pm0.00}$}	&\tabincell{c}{$52.94$\\$^{\pm0.00}$}\\
&Caltech101-7	&\tabincell{c}{$63.68^*$\\$^{\pm6.14}$}	&\tabincell{c}{$57.48^*$\\$^{\pm2.23}$}	&\tabincell{c}{$70.72^*$\\$^{\pm6.97}$}	&\tabincell{c}{$78.09^*$\\$^{\pm0.00}$}	&\tabincell{c}{$64.79^*$\\$^{\pm0.00}$}	&\tabincell{c}{$70.15^*$\\$^{\pm0.00}$}	&\tabincell{c}{$66.28^*$\\$^{\pm0.00}$}	&\tabincell{c}{$84.87^*$\\$^{\pm0.00}$}	&\tabincell{c}{$67.23^*$\\$^{\pm0.00}$}	&\tabincell{c}{$45.52^*$\\$^{\pm0.00}$}	&\tabincell{c}{$45.05^*$\\$^{\pm0.00}$}	&\tabincell{c}{$46.68^*$\\$^{\pm0.00}$}	&\tabincell{c}{$54.55^*$\\$^{\pm0.00}$}	&\tabincell{c}{$\textbf{88.87}$\\$^{\pm0.00}$}\\
&Caltech101-20	&\tabincell{c}{$51.23^*$\\$^{\pm6.24}$}	&\tabincell{c}{$45.32^*$\\$^{\pm1.84}$}	&\tabincell{c}{$55.19^*$\\$^{\pm2.39}$}	&\tabincell{c}{$42.32^*$\\$^{\pm1.91}$}	&\tabincell{c}{$57.21^*$\\$^{\pm0.00}$}	&\tabincell{c}{$49.79^*$\\$^{\pm0.00}$}	&\tabincell{c}{$53.10^*$\\$^{\pm0.00}$}	&\tabincell{c}{$54.78^*$\\$^{\pm0.00}$}	&\tabincell{c}{$\textbf{59.44}$\\$^{\pm0.92}$}	&\tabincell{c}{$49.54^*$\\$^{\pm0.00}$}	&\tabincell{c}{$50.67^*$\\$^{\pm0.00}$}	&\tabincell{c}{$52.05^*$\\$^{\pm0.00}$}	&\tabincell{c}{$53.94^*$\\$^{\pm0.00}$}	&\tabincell{c}{$52.56$\\$^{\pm0.00}$}\\
&Mfeat	&\tabincell{c}{$68.68^*$\\$^{\pm7.90}$}	&\tabincell{c}{$76.63^*$\\$^{\pm0.09}$}	&\tabincell{c}{$77.01^*$\\$^{\pm5.87}$}	&\tabincell{c}{$77.61^*$\\$^{\pm4.09}$}	&\tabincell{c}{$87.10^*$\\$^{\pm0.00}$}	&\tabincell{c}{$70.75^*$\\$^{\pm0.00}$}	&\tabincell{c}{$83.65^*$\\$^{\pm0.00}$}	&\tabincell{c}{$82.40^*$\\$^{\pm0.00}$}	&\tabincell{c}{$83.10^*$\\$^{\pm0.00}$}	&\tabincell{c}{$63.35^*$\\$^{\pm0.00}$}	&\tabincell{c}{$65.50^*$\\$^{\pm0.00}$}	&\tabincell{c}{$82.20^*$\\$^{\pm0.00}$}	&\tabincell{c}{$73.05^*$\\$^{\pm0.00}$}	&\tabincell{c}{$\textbf{98.00}$\\$^{\pm0.00}$}\\
&handwritten	&\tabincell{c}{$65.49^*$\\$^{\pm7.78}$}	&\tabincell{c}{$57.09^*$\\$^{\pm0.18}$}	&\tabincell{c}{$79.81^*$\\$^{\pm8.26}$}	&\tabincell{c}{$97.29^*$\\$^{\pm0.05}$}	&\tabincell{c}{$77.20^*$\\$^{\pm0.00}$}	&\tabincell{c}{$24.55^*$\\$^{\pm0.00}$}	&\tabincell{c}{$87.55^*$\\$^{\pm0.00}$}	&\tabincell{c}{$82.60^*$\\$^{\pm0.00}$}	&\tabincell{c}{$85.35^*$\\$^{\pm0.00}$}	&\tabincell{c}{$63.95^*$\\$^{\pm0.00}$}	&\tabincell{c}{$80.30^*$\\$^{\pm0.00}$}	&\tabincell{c}{$82.05^*$\\$^{\pm0.00}$}	&\tabincell{c}{$82.25^*$\\$^{\pm0.00}$}	&\tabincell{c}{$\textbf{98.55}$\\$^{\pm0.00}$}\\
&Outdoor-Scene	&\tabincell{c}{$31.82^*$\\$^{\pm10.36}$}	&\tabincell{c}{$46.74^*$\\$^{\pm1.66}$}	&\tabincell{c}{$49.91^*$\\$^{\pm5.12}$}	&\tabincell{c}{$40.55^*$\\$^{\pm0.00}$}	&\tabincell{c}{$45.46^*$\\$^{\pm0.00}$}	&\tabincell{c}{$46.09^*$\\$^{\pm0.00}$}	&\tabincell{c}{$43.30^*$\\$^{\pm0.00}$}	&\tabincell{c}{$48.03^*$\\$^{\pm0.00}$}	&\tabincell{c}{$62.41^*$\\$^{\pm0.07}$}	&\tabincell{c}{$34.52^*$\\$^{\pm0.00}$}	&\tabincell{c}{$59.86^*$\\$^{\pm0.00}$}	&\tabincell{c}{$63.32^*$\\$^{\pm0.00}$}	&\tabincell{c}{$63.58^*$\\$^{\pm0.00}$}	&\tabincell{c}{$\textbf{74.63}$\\$^{\pm0.00}$}\\
\cline{2-16}
		&Avg. rank	&10.13	&9.25	&6.88	&8.13	&6.88	&10.25	&7.00	&4.38	&3.25	&11.88	&9.63	&7.13	&7.25	&\textbf{2.88}\\
		\midrule
		\multirow{10}{*}{PUR}
		&MSRC-v1	&\tabincell{c}{$59.31^*$\\$^{\pm5.06}$}	&\tabincell{c}{$75.29^*$\\$^{\pm2.79}$}	&\tabincell{c}{$68.10^*$\\$^{\pm7.17}$}	&\tabincell{c}{$80.48^*$\\$^{\pm0.00}$}	&\tabincell{c}{$73.33^*$\\$^{\pm0.00}$}	&\tabincell{c}{$43.33^*$\\$^{\pm0.00}$}	&\tabincell{c}{$64.76^*$\\$^{\pm0.00}$}	&\tabincell{c}{$77.14^*$\\$^{\pm0.00}$}	&\tabincell{c}{$79.05^*$\\$^{\pm0.00}$}	&\tabincell{c}{$50.48^*$\\$^{\pm0.00}$}	&\tabincell{c}{$40.00^*$\\$^{\pm0.00}$}	&\tabincell{c}{$69.05^*$\\$^{\pm0.00}$}	&\tabincell{c}{$62.86^*$\\$^{\pm0.00}$}	&\tabincell{c}{$\textbf{82.38}$\\$^{\pm0.00}$}\\
& ORL	&\tabincell{c}{$77.05^*$\\$^{\pm2.47}$}	&\tabincell{c}{$83.31^*$\\$^{\pm2.97}$}	&\tabincell{c}{$78.08^*$\\$^{\pm1.87}$}	&\tabincell{c}{$73.41^*$\\$^{\pm0.15}$}	&\tabincell{c}{$76.75^*$\\$^{\pm0.00}$}	&\tabincell{c}{$16.50^*$\\$^{\pm0.00}$}	&\tabincell{c}{$78.50^*$\\$^{\pm0.00}$}	&\tabincell{c}{$83.00^*$\\$^{\pm0.00}$}	&\tabincell{c}{$85.98^*$\\$^{\pm0.62}$}	&\tabincell{c}{$16.25^*$\\$^{\pm0.00}$}	&\tabincell{c}{$62.00^*$\\$^{\pm0.00}$}	&\tabincell{c}{$60.25^*$\\$^{\pm0.00}$}	&\tabincell{c}{$60.00^*$\\$^{\pm0.00}$}	&\tabincell{c}{$\textbf{88.75}$\\$^{\pm0.00}$}\\
&WebKB-Texas	&\tabincell{c}{$56.10^*$\\$^{\pm0.16}$}	&\tabincell{c}{$66.58^*$\\$^{\pm0.41}$}	&\tabincell{c}{$60.00^*$\\$^{\pm1.98}$}	&\tabincell{c}{$56.58^*$\\$^{\pm0.88}$}	&\tabincell{c}{$59.89^*$\\$^{\pm0.00}$}	&\tabincell{c}{$56.68^*$\\$^{\pm0.00}$}	&\tabincell{c}{$57.22^*$\\$^{\pm0.00}$}	&\tabincell{c}{$62.57^*$\\$^{\pm0.00}$}	&\tabincell{c}{$67.38^*$\\$^{\pm0.00}$}	&\tabincell{c}{$\textbf{68.98}$\\$^{\pm0.00}$}	&\tabincell{c}{$62.57^*$\\$^{\pm0.00}$}	&\tabincell{c}{$67.91^*$\\$^{\pm0.00}$}	&\tabincell{c}{$65.78^*$\\$^{\pm0.00}$}	&\tabincell{c}{$67.91$\\$^{\pm0.00}$}\\
&Caltech101-7	&\tabincell{c}{$82.91^*$\\$^{\pm7.85}$}	&\tabincell{c}{$88.32^*$\\$^{\pm0.62}$}	&\tabincell{c}{$83.24^*$\\$^{\pm6.41}$}	&\tabincell{c}{$89.08^*$\\$^{\pm0.00}$}	&\tabincell{c}{$83.65^*$\\$^{\pm0.00}$}	&\tabincell{c}{$75.85^*$\\$^{\pm0.00}$}	&\tabincell{c}{$85.35^*$\\$^{\pm0.00}$}	&\tabincell{c}{$85.21^*$\\$^{\pm0.00}$}	&\tabincell{c}{$89.08^*$\\$^{\pm0.00}$}	&\tabincell{c}{$84.19^*$\\$^{\pm0.00}$}	&\tabincell{c}{$86.16^*$\\$^{\pm0.00}$}	&\tabincell{c}{$87.31^*$\\$^{\pm0.00}$}	&\tabincell{c}{$84.74^*$\\$^{\pm0.00}$}	&\tabincell{c}{$\textbf{89.35}$\\$^{\pm0.00}$}\\
&Caltech101-20	&\tabincell{c}{$67.00^*$\\$^{\pm4.76}$}	&\tabincell{c}{$78.61^*$\\$^{\pm0.84}$}	&\tabincell{c}{$68.36^*$\\$^{\pm2.35}$}	&\tabincell{c}{$62.09^*$\\$^{\pm1.89}$}	&\tabincell{c}{$69.87^*$\\$^{\pm0.00}$}	&\tabincell{c}{$73.97^*$\\$^{\pm0.00}$}	&\tabincell{c}{$64.59^*$\\$^{\pm0.00}$}	&\tabincell{c}{$74.48^*$\\$^{\pm0.00}$}	&\tabincell{c}{$\textbf{80.25}$\\$^{\pm0.97}$}	&\tabincell{c}{$80.09^*$\\$^{\pm0.00}$}	&\tabincell{c}{$76.70^*$\\$^{\pm0.00}$}	&\tabincell{c}{$70.49^*$\\$^{\pm0.00}$}	&\tabincell{c}{$74.98^*$\\$^{\pm0.00}$}	&\tabincell{c}{$77.03$\\$^{\pm0.00}$}\\
&Mfeat	&\tabincell{c}{$69.72^*$\\$^{\pm7.25}$}	&\tabincell{c}{$76.63^*$\\$^{\pm0.09}$}	&\tabincell{c}{$80.27^*$\\$^{\pm4.65}$}	&\tabincell{c}{$80.36^*$\\$^{\pm4.09}$}	&\tabincell{c}{$87.15^*$\\$^{\pm0.00}$}	&\tabincell{c}{$72.50^*$\\$^{\pm0.00}$}	&\tabincell{c}{$86.05^*$\\$^{\pm0.00}$}	&\tabincell{c}{$83.00^*$\\$^{\pm0.00}$}	&\tabincell{c}{$88.00^*$\\$^{\pm0.00}$}	&\tabincell{c}{$65.55^*$\\$^{\pm0.00}$}	&\tabincell{c}{$69.30^*$\\$^{\pm0.00}$}	&\tabincell{c}{$82.20^*$\\$^{\pm0.00}$}	&\tabincell{c}{$73.10^*$\\$^{\pm0.00}$}	&\tabincell{c}{$\textbf{98.00}$\\$^{\pm0.00}$}\\
&handwritten	&\tabincell{c}{$65.50^*$\\$^{\pm7.78}$}	&\tabincell{c}{$59.20^*$\\$^{\pm0.14}$}	&\tabincell{c}{$82.74^*$\\$^{\pm6.50}$}	&\tabincell{c}{$97.29^*$\\$^{\pm0.05}$}	&\tabincell{c}{$78.50^*$\\$^{\pm0.00}$}	&\tabincell{c}{$26.20^*$\\$^{\pm0.00}$}	&\tabincell{c}{$88.10^*$\\$^{\pm0.00}$}	&\tabincell{c}{$84.70^*$\\$^{\pm0.00}$}	&\tabincell{c}{$87.70^*$\\$^{\pm0.00}$}	&\tabincell{c}{$63.95^*$\\$^{\pm0.00}$}	&\tabincell{c}{$80.30^*$\\$^{\pm0.00}$}	&\tabincell{c}{$82.05^*$\\$^{\pm0.00}$}	&\tabincell{c}{$82.25^*$\\$^{\pm0.00}$}	&\tabincell{c}{$\textbf{98.55}$\\$^{\pm0.00}$}\\
&Outdoor-Scene	&\tabincell{c}{$32.03^*$\\$^{\pm10.40}$}	&\tabincell{c}{$48.85^*$\\$^{\pm1.16}$}	&\tabincell{c}{$50.47^*$\\$^{\pm5.32}$}	&\tabincell{c}{$40.85^*$\\$^{\pm0.00}$}	&\tabincell{c}{$45.57^*$\\$^{\pm0.00}$}	&\tabincell{c}{$48.03^*$\\$^{\pm0.00}$}	&\tabincell{c}{$44.01^*$\\$^{\pm0.00}$}	&\tabincell{c}{$51.00^*$\\$^{\pm0.00}$}	&\tabincell{c}{$64.31^*$\\$^{\pm0.08}$}	&\tabincell{c}{$35.23^*$\\$^{\pm0.00}$}	&\tabincell{c}{$59.86^*$\\$^{\pm0.00}$}	&\tabincell{c}{$66.00^*$\\$^{\pm0.00}$}	&\tabincell{c}{$67.08^*$\\$^{\pm0.00}$}	&\tabincell{c}{$\textbf{74.63}$\\$^{\pm0.00}$}\\
\cline{2-16}
		&Avg. rank	&11.75	&6.25	&8.38	&7.63	&8.50	&11.75	&7.88	&5.75	&2.75	&9.75	&8.63	&6.38	&7.75	&\textbf{1.50}\\
		\bottomrule
	\end{tabular}

\begin{tablenotes}
		\footnotesize
		\item[1] * The symbol ``*'' indicates statistically significant improvement (of JMVFG over a method on this dataset) w.r.t. Student's t-test with $p<0.05$.
	\end{tablenotes}

\end{table*}

\subsection{Comparison Against Other Multi-view Unsupervised Feature Selection Methods}

In this section, we evaluate the proposed JMVFG method against the other multi-view unsupervised feature selection methods with the percentage of selected features varying from 5\% to 40\% at intervals of 5\%. As shown in Fig. \ref{fig:NMI}, the proposed method achieves highly competitive performance (w.r.t. NMI) with varying percentages of selected features, when compared with the state-of-the-art methods on the benchmark datasets.

Further, we compare different multi-view unsupervised feature selection methods by considering the optimal feature selection percentage. As shown in Table \ref{table:FS}, the proposed method yields the best NMI scores on seven datasets, the best ACC scores on four datasets, and the best PUR scores on seven datasets, out of the totally eight datasets. For the proposed method and the baseline methods, we conduct the Student's t-test to evaluate the statistical significance of the results, where the symbol ``*'' indicates significant improvement (of JMVFG over a method on this dataset) with $p<0.05$. The experimental results in  Fig. \ref{fig:NMI} and Table \ref{table:FS} demonstrate the robust performance of JMVFG for the multi-view unsupervised feature selection task.

\begin{figure*}[!t]\vskip -0.2 in
	\begin{center}
		{\subfigure[{\scriptsize MSRC-v1}]
			{\includegraphics[width=0.45\columnwidth]{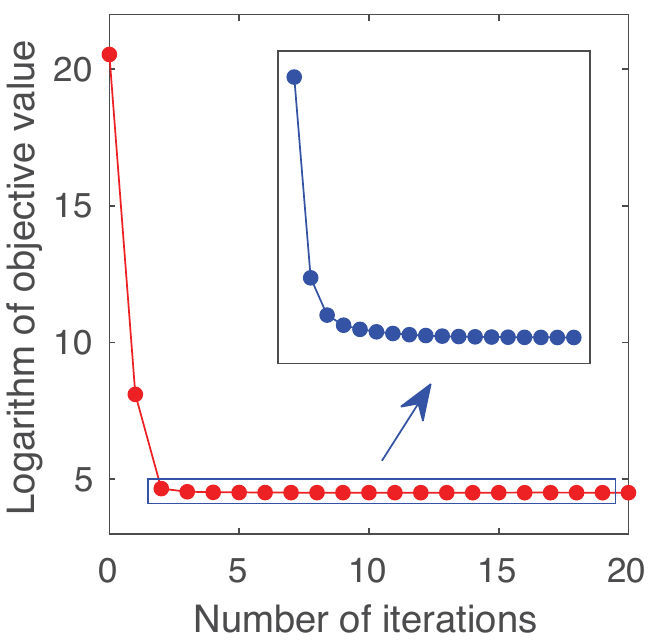}}}
		{\subfigure[{\scriptsize ORL}]
			{\includegraphics[width=0.45\columnwidth]{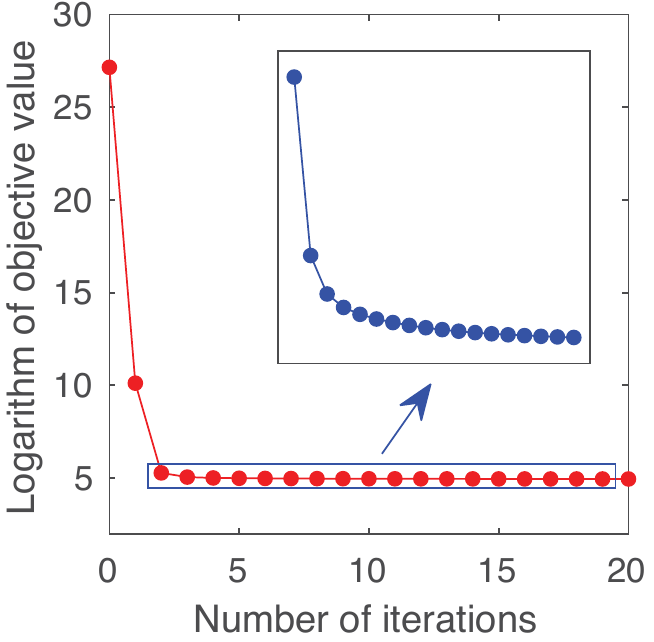}}}
		{\subfigure[{\scriptsize WebKB-Texas}]
			{\includegraphics[width=0.45\columnwidth]{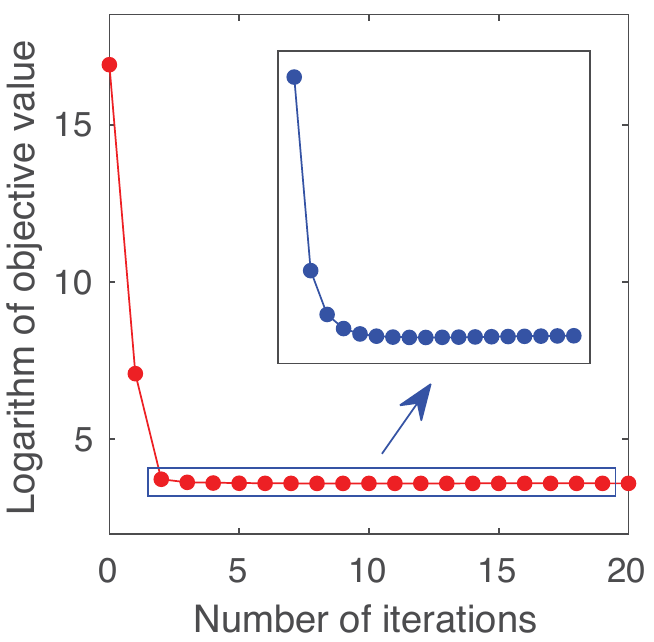}}}
		{\subfigure[{\scriptsize Caltech101-7}]
			{\includegraphics[width=0.45\columnwidth]{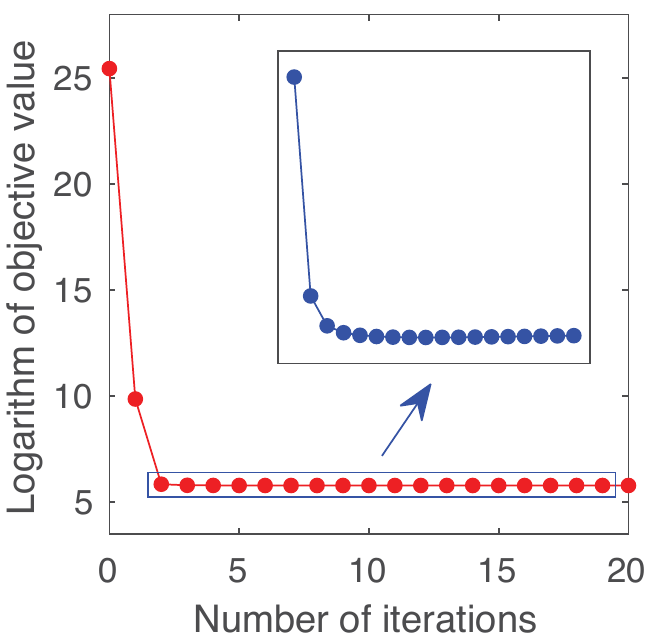}}}\\\vskip -0.03 in
		{\subfigure[{\scriptsize Caltech101-20}]
			{\includegraphics[width=0.45\columnwidth]{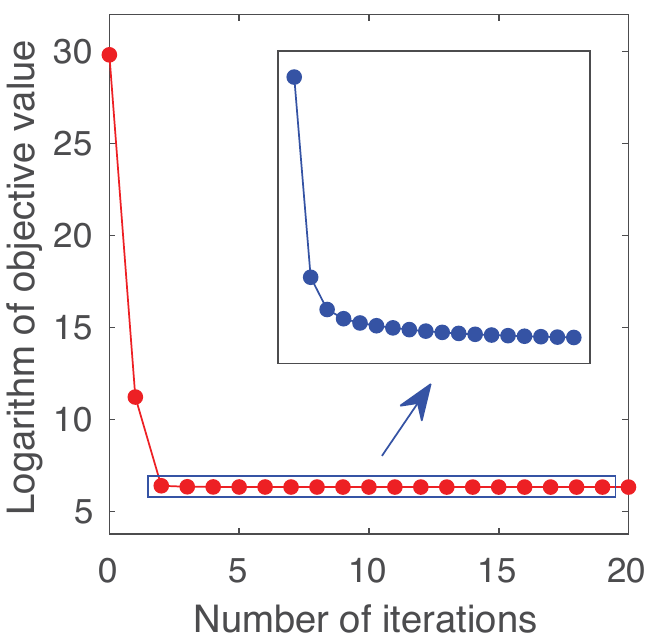}}}
		{\subfigure[{\scriptsize Mfeat}]
			{\includegraphics[width=0.45\columnwidth]{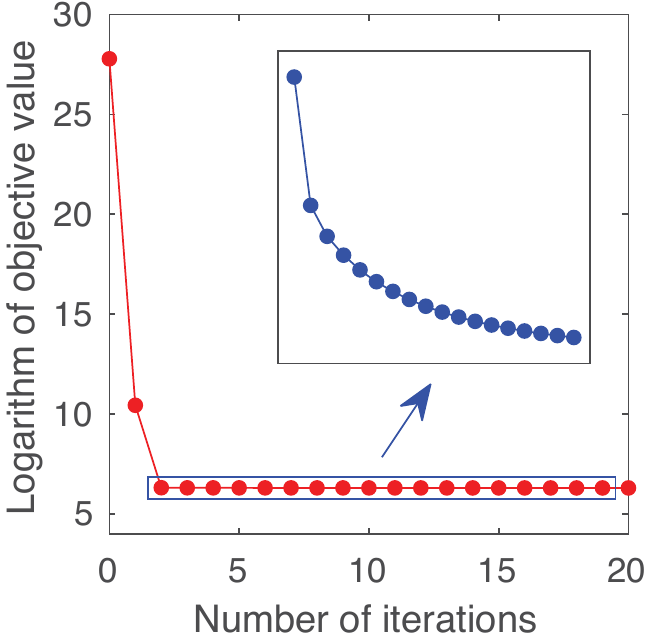}}}
		{\subfigure[{\scriptsize Handwritten}]
			{\includegraphics[width=0.45\columnwidth]{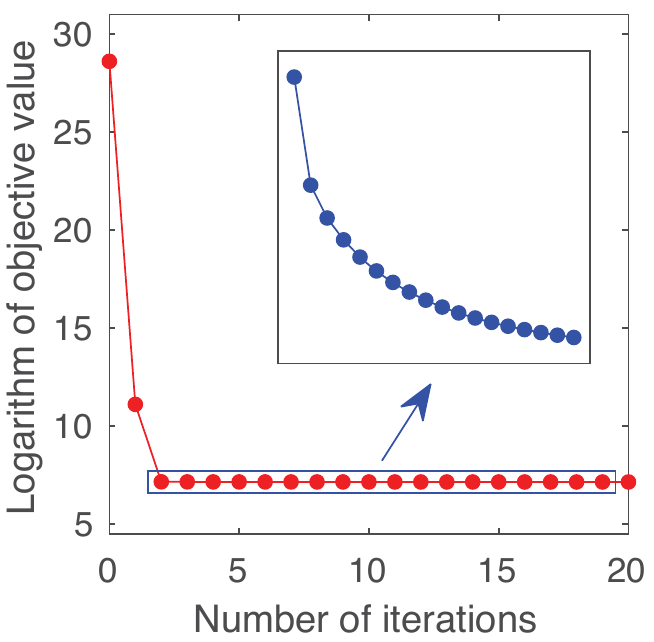}}}
		{\subfigure[{\scriptsize Outdoor-Scene}]
			{\includegraphics[width=0.45\columnwidth]{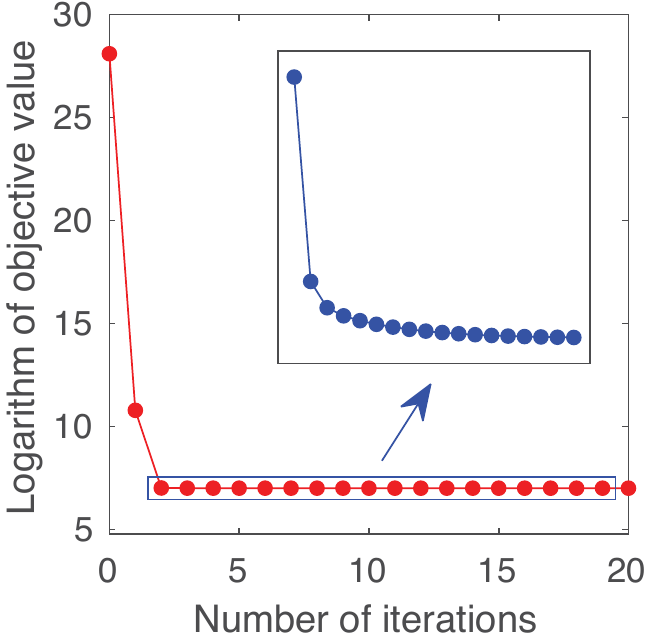}}}\vskip -0.05 in
		\caption{Empirical convergence analysis. The logarithm of the objective function value of JMVFG is illustrated (as the number of iterations increases).}
		\label{fig:conver}\vskip -0.2 in
	\end{center}
\end{figure*}

\subsection{Comparison Against Other Multi-view Clustering Methods}

In the section, we evaluate the proposed JMVFG method for the multi-view clustering task, in comparison with the other thirteen multi-view clustering methods.

As shown in Table \ref{t:CL}, in terms of NMI, JMVFG achieves the best performance on all of the eight datasets.
In terms of ACC and PUR, similar advantages of our JMVFG method can also be observed. Specifically, JMVFG achieves the best ACC scores on six datasets, and the best PUR scores on six datasets, out of the totally eight datasets.
Note that with the same starting status, our optimization algorithm will converge to the same solution (over multiple runs), which means that the same similarity matrix $\textbf{S}$ will be learned over multiple runs. For the multi-view clustering task, the learned similarity matrix $\textbf{S}$ is partitioned via spectral clustering for evaluation, where the stability of the spectral clustering (as well as the optimization process) contributes to the zero standard deviations of our JMVFG method (as shown in Tables~\ref{t:CL}).
	
We also provide the average ranks (across eight datasets) of different methods.  As shown in Table~\ref{t:CL}, the average ranks of JMVFG (w.r.t. NMI, ACC, and PUR) are 1.00, 2.88, and 1.50, respectively, while the second best method obtains the average ranks (w.r.t. NMI, ACC, and PUR) of 2.63, 3.25, and 2.75, respectively, which confirm the effectiveness of JMVFG for the multi-view clustering task.

\subsection{Empirical Convergence Analysis}
\label{sec:empirical_convergence}
In this section, we conduct empirical convergence analysis on the benchmark datasets.
According to the theoretical conclusion in Section \ref{sec:theoretical_convergence}, the monotonicity of each step of our optimization algorithm can be guaranteed provided that the affinity matrix $\textbf{A}^{(v)}$ is initialized to be symmetric, which can be easily satisfied (e.g., by defining it as a symmetric $K$-NN graph). In fact, we empirically find that the convergence property of JMVFG is fast and good even when we remove the symmetric constraint (e.g., by normalizing each row of $\textbf{A}^{(v)}$). Specifically, we illustrate the convergence curves of JMVFG in Fig.~\ref{fig:conver}. As shown in Fig. \ref{fig:conver}, the objective value rapidly and monotonically decreases as the number of iterations grows, where the convergence is generally reached within twenty iterations, which show that the empirical convergence of JMVFG is well consistent with our theoretical analysis.

\subsection{Parameter Analysis}
In this section, we conduct experiments to analyze the influence of the parameters $\beta$, $\gamma$,  and $\eta$, as well as the percentage of selected features. Specifically, the performance of our method with varying parameters $\beta$ and $\gamma$ is illustrated in Fig.~\ref{fig:sen_s}, while that with varying parameter $\eta$ and percentage of selected features is illustrated in Fig.~\ref{fig:sen_p}. As shown in Figs.~\ref{fig:sen_s} and \ref{fig:sen_p}, the performance of the proposed method is consistent and robust with different parameter settings, which suggests that they can be safely set to some moderate values while maintaining high-quality performance on most of the datasets.

\subsection{Ablation Study}

In this section, we provide the ablation study w.r.t. the multi-view unsupervised feature selection task and the multi-view graph learning (for clustering) task in Sections~\ref{sec:ablation_ufs} and \ref{sec:ablation_gl}, respectively.

\subsubsection{Ablation w.r.t. Multi-view Feature Selection}
\label{sec:ablation_ufs}

This section analyzes the influences of different terms in the overall objective function~\eqref{eq:3-11} for the multi-view feature selection task. There are three removable terms for this task, namely, the regularization term of $\textbf{W}^{(v)}$, the cross-space locality preservation term, and the graph learning term. Since the cross-space locality preservation term bridges the gap between feature selection and graph learning, if this \emph{bridge} is removed, then the influence of graph learning also disappears. Therefore, the experimental setting of removing the \emph{bridge} while preserving the graph learning term is not necessary to be tested.
As shown in Table~\ref{t:Ab_fs}, using all the three terms leads to average NMI(\%), ACC(\%), and PUR(\%) scores (across eight datasets) of 66.11, 69.00, and 79.04, respectively, while removing the graph learning term leads to average NMI(\%), ACC(\%), and PUR(\%) scores of 64.02, 65.68, and 76.79, respectively, which confirm the contribution of the graph learning term for the feature selection task. Further, in an extreme case, if all three terms are removed, the average NMI(\%), ACC(\%), and PUR(\%) scores (across eight datasets) would drop to 48.62, 52.06, and 64.76, respectively. The ablation results in Table~\ref{t:Ab_fs} have shown the substantial contribution of different components/terms in the JMVFG method for the feature selection task.

\begin{figure}[!t]
	\begin{center}
		{\subfigure[{\scriptsize MSRC-v1}]
			{\includegraphics[width=0.24\columnwidth]{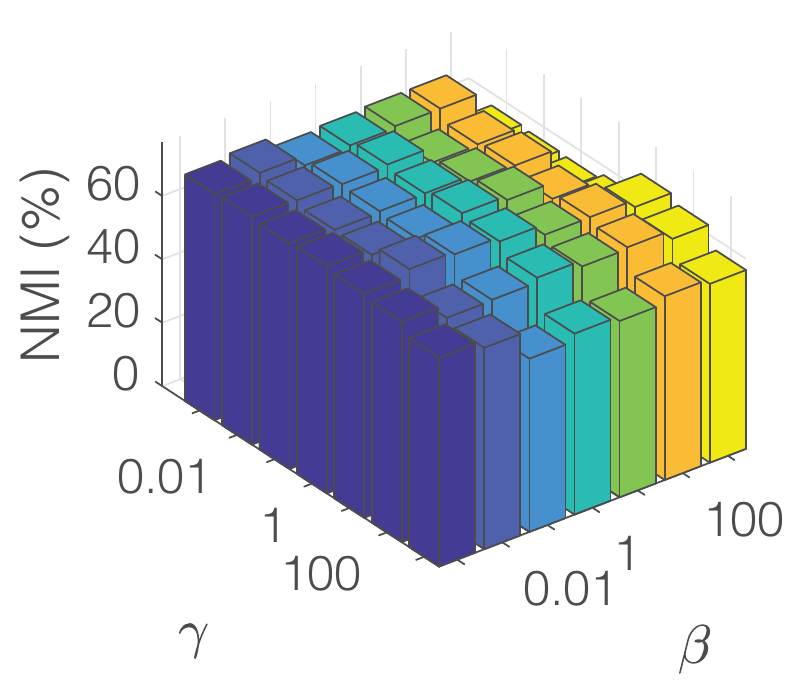}}}
		{\subfigure[{\scriptsize ORL}]
			{\includegraphics[width=0.24\columnwidth]{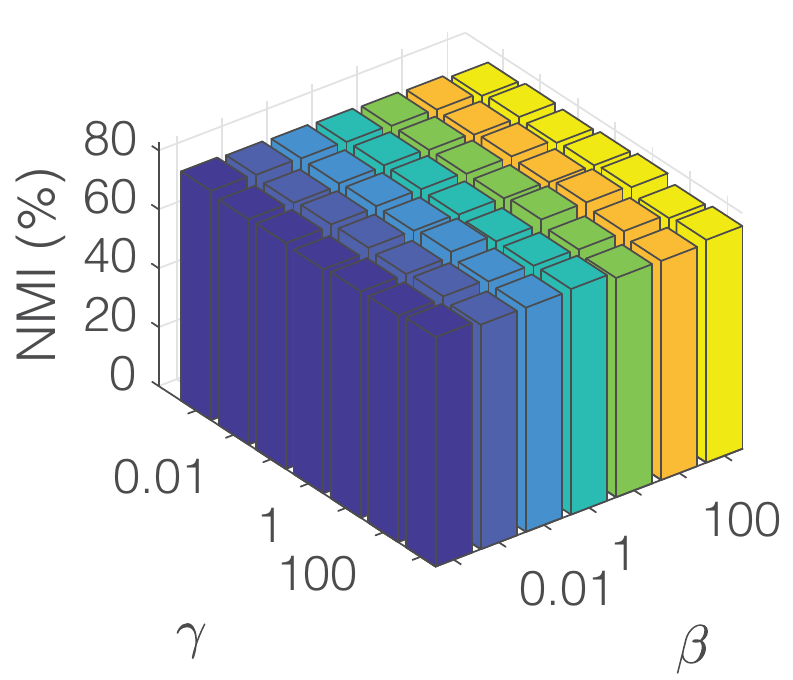}}}
		{\subfigure[{\scriptsize WebKB-Texas}]
			{\includegraphics[width=0.24\columnwidth]{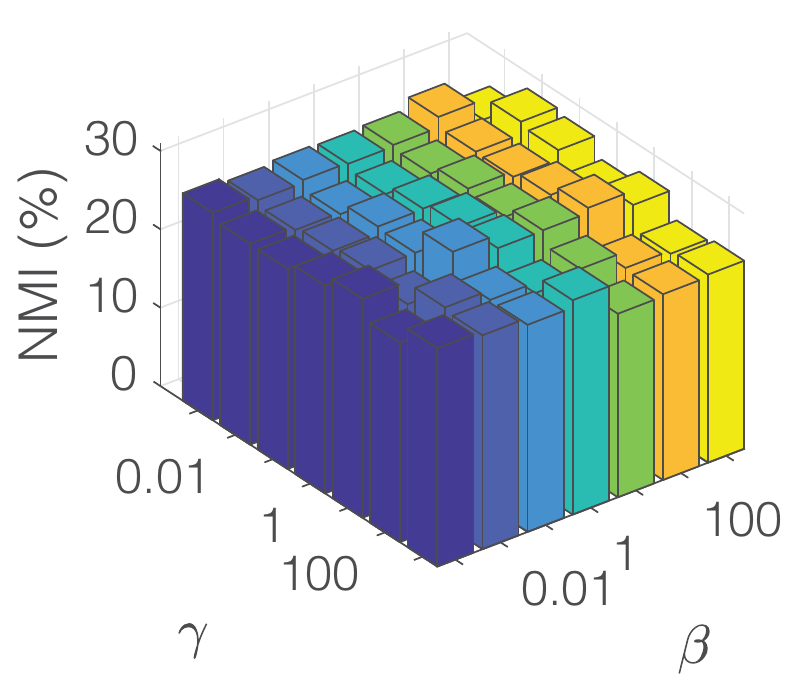}}}
		{\subfigure[{\scriptsize Caltech101-7}]
			{\includegraphics[width=0.24\columnwidth]{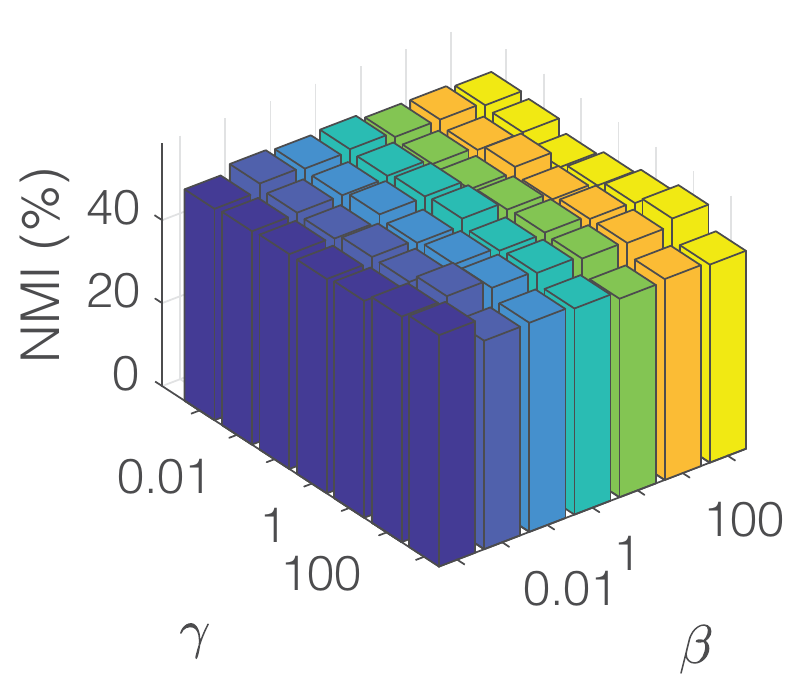}}}\\
		{\subfigure[{\scriptsize Caltech101-20}]
			{\includegraphics[width=0.24\columnwidth]{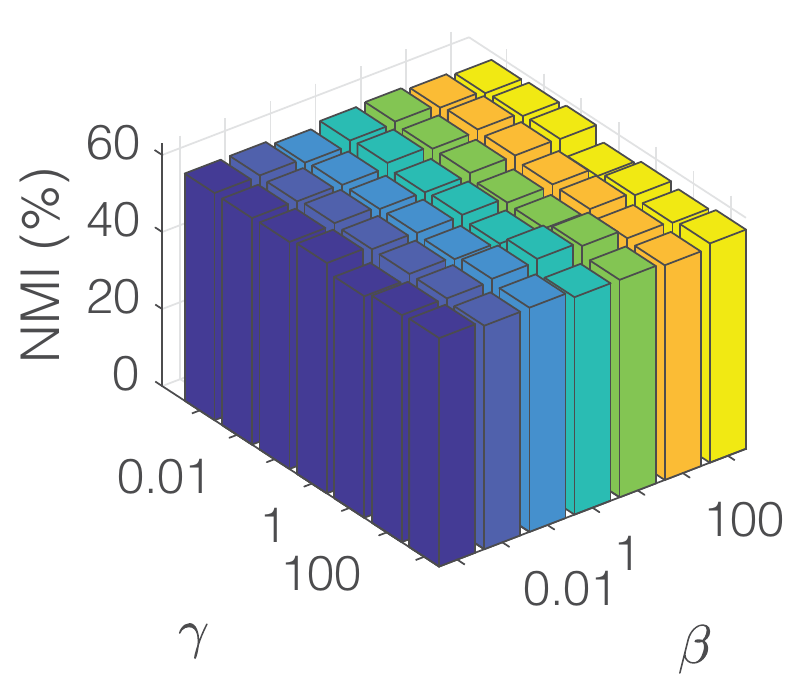}}}
		{\subfigure[{\scriptsize Mfeat}]
			{\includegraphics[width=0.24\columnwidth]{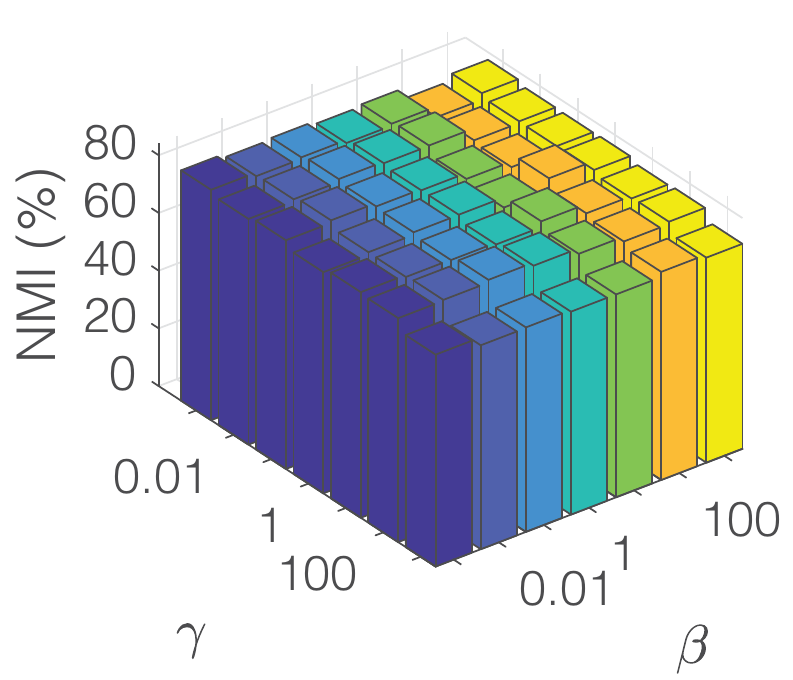}}}
		{\subfigure[{\scriptsize Handwritten}]
			{\includegraphics[width=0.24\columnwidth]{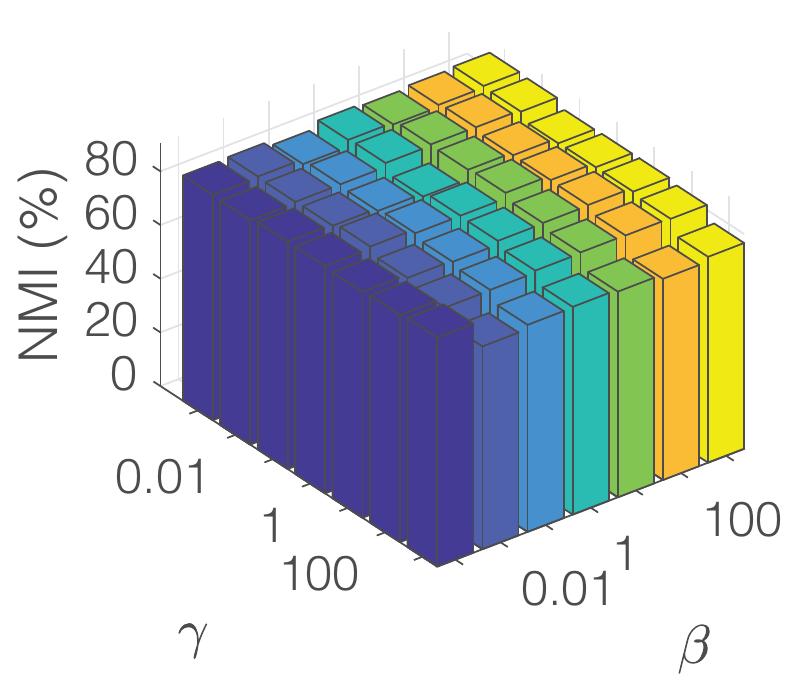}}}
		{\subfigure[{\scriptsize Outdoor-Scene}]
			{\includegraphics[width=0.24\columnwidth]{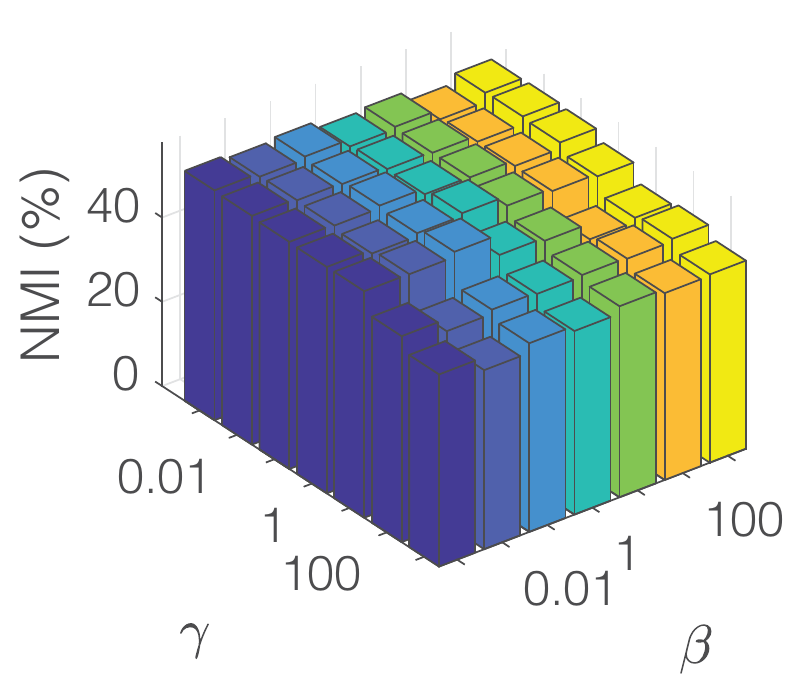}}}
		\caption{Average performance (w.r.t. NMI(\%)) of JMVFG with varying $\beta$ and $\gamma$ while fixing $\eta=1$.}
		\label{fig:sen_s}\vskip -0.1 in
	\end{center}
\end{figure}

\begin{figure}[!t]
	\begin{center}
		{\subfigure[{\scriptsize MSRC-v1}]
			{\includegraphics[width=0.24\columnwidth]{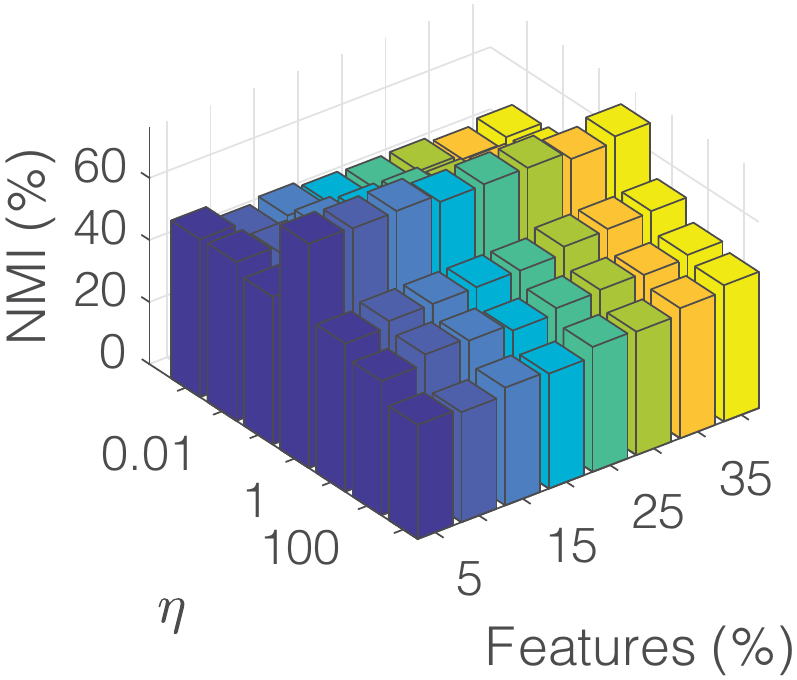}}}
		{\subfigure[{\scriptsize ORL}]
			{\includegraphics[width=0.24\columnwidth]{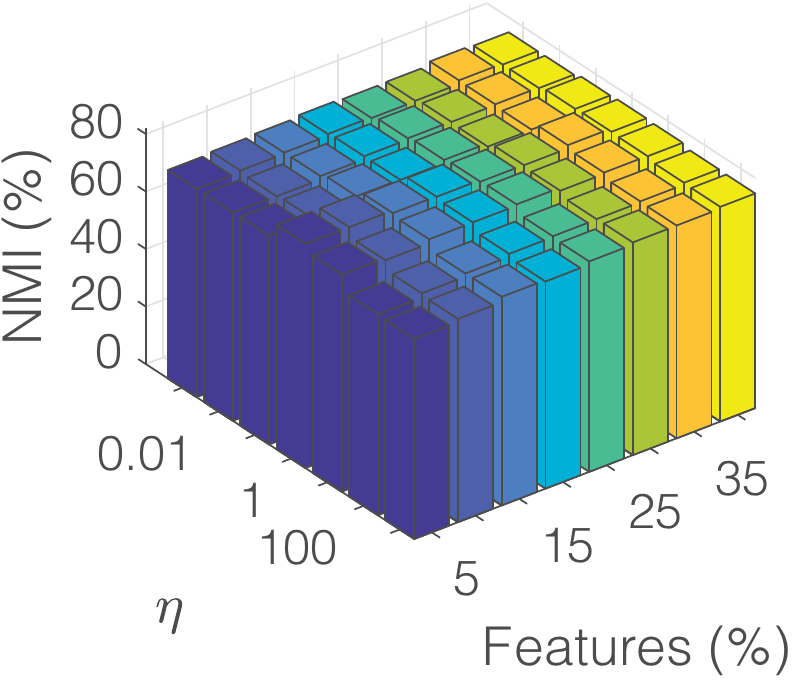}}}
		{\subfigure[{\scriptsize WebKB-Texas}]
			{\includegraphics[width=0.24\columnwidth]{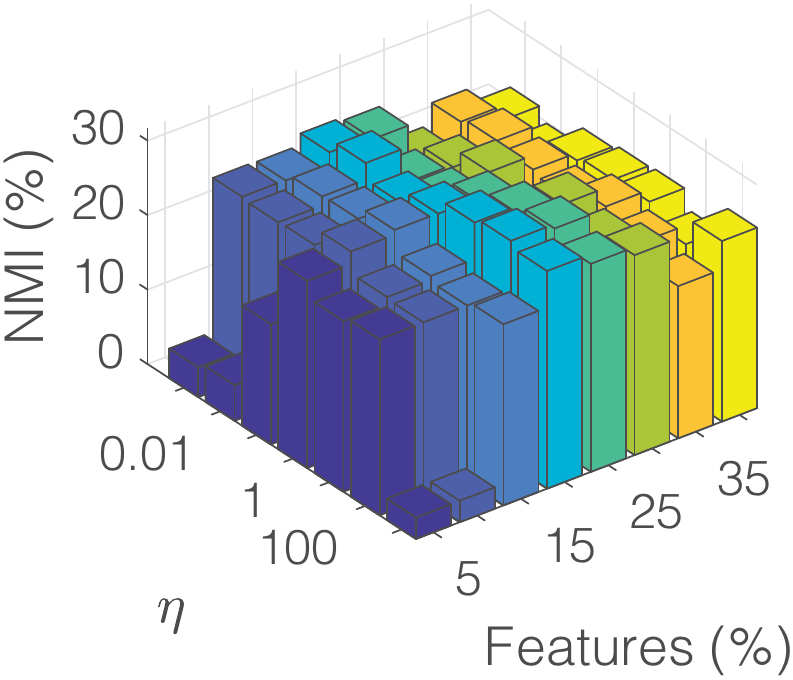}}}
		{\subfigure[{\scriptsize Caltech101-7}]
			{\includegraphics[width=0.24\columnwidth]{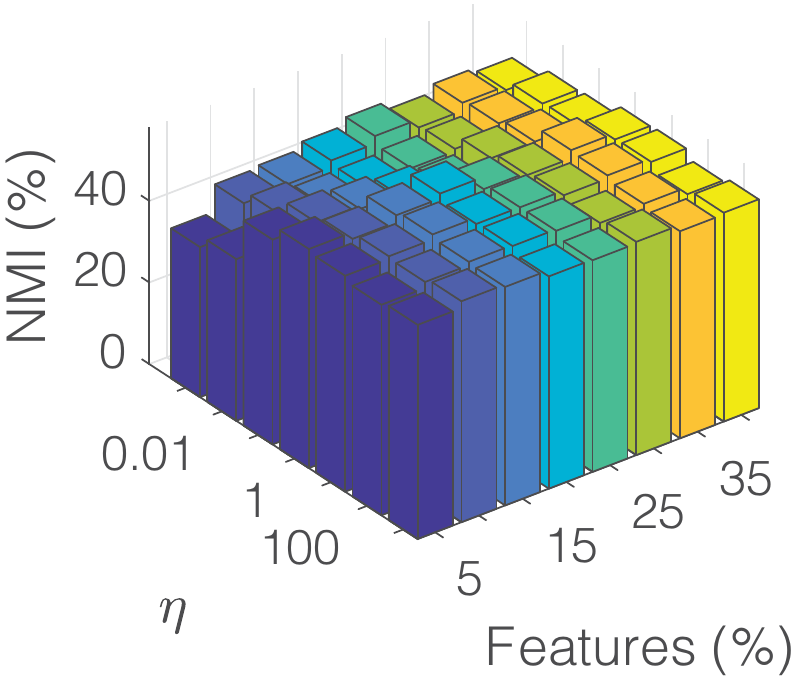}}}\\
		{\subfigure[{\scriptsize Caltech101-20}]
			{\includegraphics[width=0.24\columnwidth]{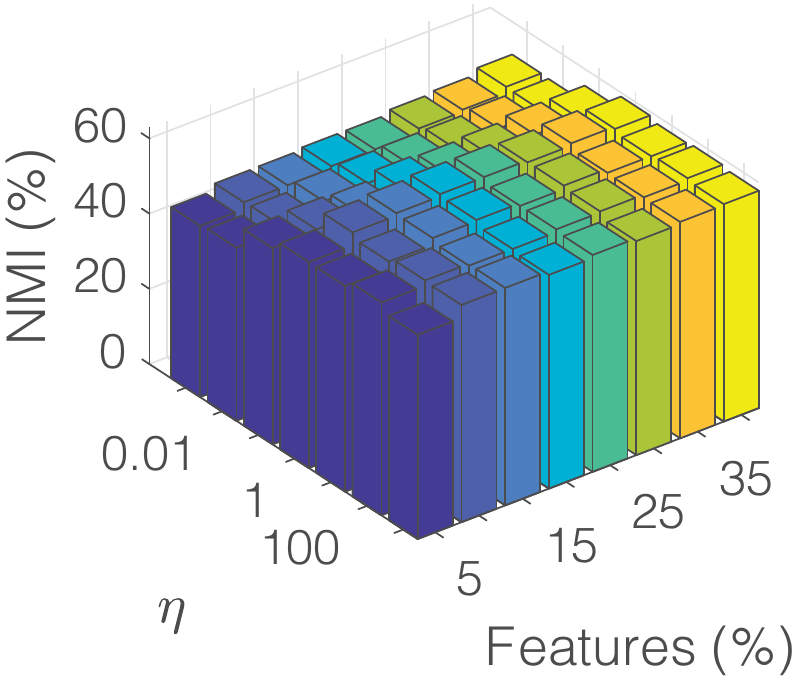}}}
		{\subfigure[{\scriptsize Mfeat}]
			{\includegraphics[width=0.24\columnwidth]{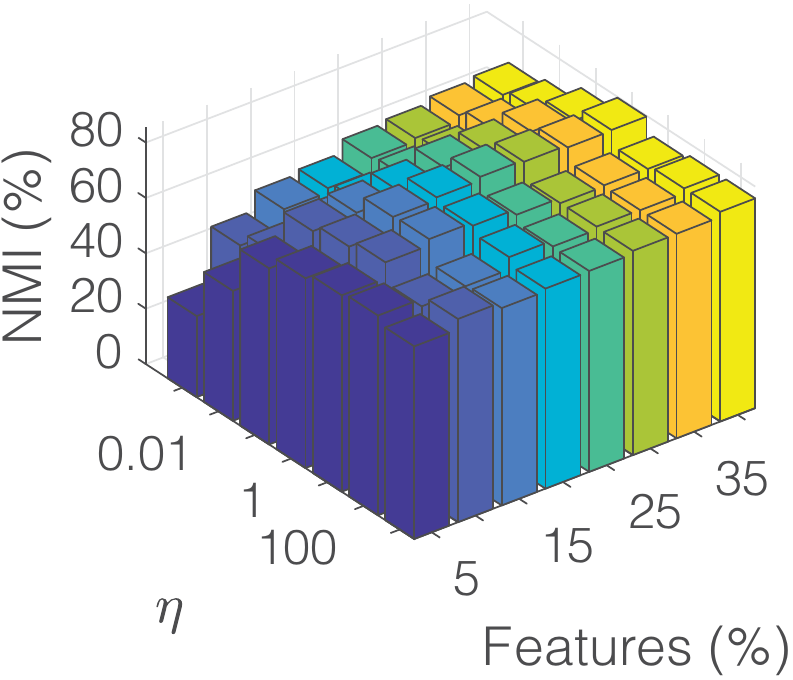}}}
		{\subfigure[{\scriptsize Handwritten}]
			{\includegraphics[width=0.24\columnwidth]{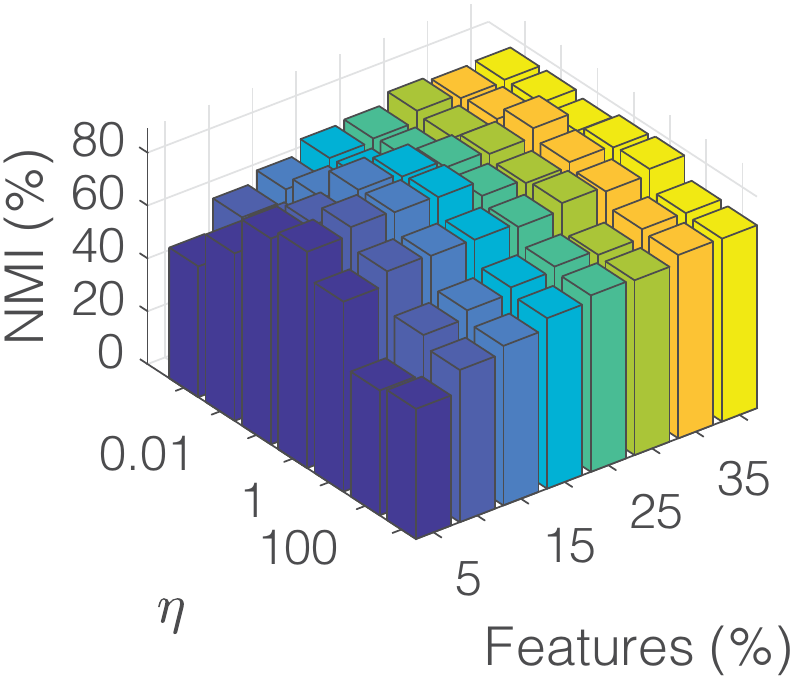}}}
		{\subfigure[{\scriptsize Outdoor-Scene}]
			{\includegraphics[width=0.24\columnwidth]{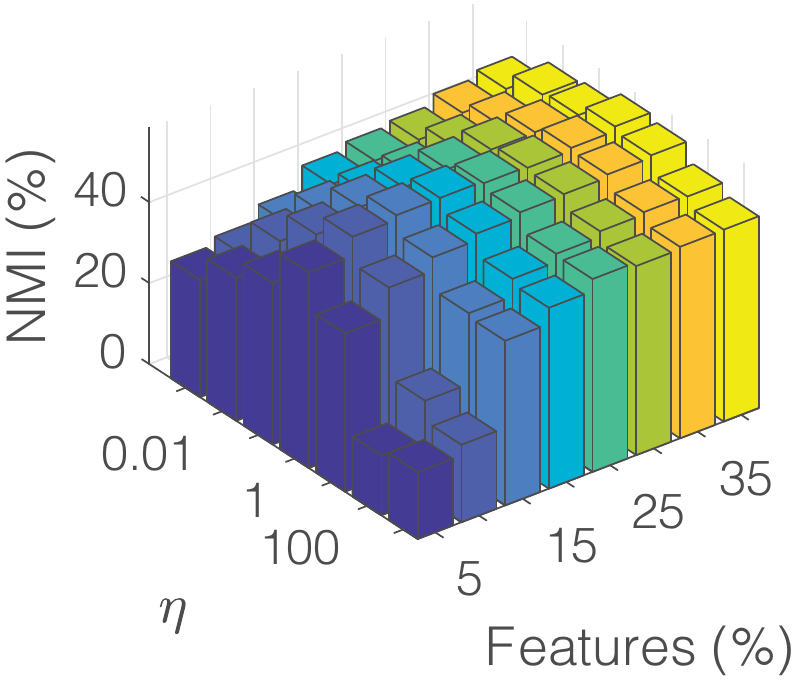}}}
		\caption{Average performance (w.r.t. NMI(\%)) of JMVFG with varying $\eta$ and percentage of selected features while fixing $\beta=1$ and $\gamma=1$.}
		\label{fig:sen_p}
	\end{center}\vskip -0.1 in
\end{figure}

\begin{table*}[!t]
	\caption{Ablation analysis w.r.t. the \emph{unsupervised feature selection} performance ($mean\%_{\pm std\%}$) over 20 runs on the benchmark datasets. The best score in each row is highlighted in bold.}
	\label{t:Ab_fs}
	\centering \begin{tabular}{|m{0.9cm}<{\centering}|m{4.24cm}<{\centering}|m{1.6cm}<{\centering}|m{1.5cm}<{\centering}m{1.5cm}<{\centering}|m{1.5cm}<{\centering}m{1.5cm}<{\centering}|m{1.5cm}<{\centering}|}
		\hline
		\multirow{4}{*}{Metric}
		&Method        &JMVFG     &\multicolumn{5}{c|}{Leave one/two/three components out}\\
		\cline{2-8}
		&Regularization of $\textbf{W}^{(v)}$    &\checkmark     &\checkmark         &\tiny{\XSolidBrush}                &\checkmark    &\tiny{\XSolidBrush}                   &\tiny{\XSolidBrush}      \\
		&Cross-Space Locality Preservation                     &\checkmark     &\checkmark         &\checkmark      &\tiny{\XSolidBrush}               &\checkmark         &\tiny{\XSolidBrush}         \\
		&Graph Learning                          &\checkmark     &\tiny{\XSolidBrush}                    &\checkmark      &\tiny{\XSolidBrush}               &\tiny{\XSolidBrush}                    &\tiny{\XSolidBrush}    \\
		\hline
		\multirow{9}{*}{NMI}
		&MSRC-v1	&$\textbf{74.94}_{\pm3.59}$	&$72.47_{\pm5.11}$	&$46.57_{\pm4.03}$	&$72.01_{\pm3.86}$	&$45.12_{\pm3.33}$	&$49.52_{\pm4.26}$	\\
		& ORL	&$\textbf{78.90}_{\pm2.43}$	&$77.77_{\pm1.95}$	&$71.90_{\pm1.98}$	&$77.95_{\pm2.63}$	&$73.28_{\pm1.93}$	&$69.40_{\pm1.69}$	\\
		&WebKB-Texas	&$\textbf{31.55}_{\pm5.03}$	&$28.96_{\pm5.57}$	&$25.26_{\pm7.97}$	&$27.46_{\pm6.48}$	&$24.85_{\pm5.68}$	&$27.16_{\pm3.73}$	\\
		&Caltech101-7	&$\textbf{57.74}_{\pm6.77}$	&$53.25_{\pm3.88}$	&$36.56_{\pm2.43}$	&$51.51_{\pm3.31}$	&$39.01_{\pm3.15}$	&$39.39_{\pm2.65}$	\\
		&Caltech101-20	&$\textbf{60.49}_{\pm1.62}$	&$58.83_{\pm1.35}$	&$49.94_{\pm1.77}$	&$59.57_{\pm2.17}$	&$49.84_{\pm1.33}$	&$49.69_{\pm2.06}$	\\
		&Mfeat	&$\textbf{81.83}_{\pm3.58}$	&$80.48_{\pm3.24}$	&$62.10_{\pm2.22}$	&$80.41_{\pm2.26}$	&$61.71_{\pm2.28}$	&$61.33_{\pm1.96}$	\\
		&handwritten	&$\textbf{87.43}_{\pm5.13}$	&$86.07_{\pm3.42}$	&$59.21_{\pm1.49}$	&$85.45_{\pm4.12}$	&$64.67_{\pm2.20}$	&$61.77_{\pm1.34}$	\\
		&Outdoor-Scene	&$\textbf{56.00}_{\pm1.61}$	&$54.38_{\pm1.97}$	&$30.73_{\pm1.47}$	&$54.79_{\pm0.95}$	&$30.44_{\pm1.26}$	&$30.70_{\pm1.50}$	\\
		\cline{2-8}
		&Avg. score	&$\textbf{66.11}$	&$64.02$	&$47.79$	&$63.64$	&$48.61$	&$48.62$	\\
		\hline
		\multirow{9}{*}{ACC}
		&MSRC-v1	&$\textbf{80.55}_{\pm7.14}$	&$77.02_{\pm6.74}$	&$48.19_{\pm3.47}$	&$76.95_{\pm7.08}$	&$54.98_{\pm4.90}$	&$53.86_{\pm5.87}$	\\
		& ORL	&$\textbf{60.74}_{\pm4.23}$	&$58.24_{\pm3.11}$	&$55.31_{\pm3.69}$	&$57.54_{\pm2.89}$	&$53.94_{\pm2.89}$	&$53.49_{\pm3.13}$	\\
		&WebKB-Texas	&$\textbf{59.76}_{\pm4.44}$	&$55.27_{\pm6.58}$	&$58.90_{\pm4.02}$	&$54.41_{\pm7.06}$	&$48.77_{\pm9.38}$	&$54.89_{\pm3.56}$	\\
		&Caltech101-7	&$\textbf{62.54}_{\pm9.73}$	&$58.15_{\pm6.81}$	&$42.62_{\pm2.43}$	&$55.04_{\pm5.81}$	&$47.79_{\pm3.91}$	&$44.46_{\pm2.79}$	\\
		&Caltech101-20	&$\textbf{50.78}_{\pm4.33}$	&$47.19_{\pm5.39}$	&$40.23_{\pm4.29}$	&$45.16_{\pm4.32}$	&$38.54_{\pm3.58}$	&$39.20_{\pm3.43}$	\\
		&Mfeat	&$\textbf{83.45}_{\pm8.12}$	&$81.03_{\pm6.57}$	&$65.35_{\pm4.36}$	&$79.96_{\pm5.26}$	&$64.81_{\pm4.99}$	&$63.67_{\pm4.13}$	\\
		&handwritten	&$\textbf{85.81}_{\pm10.33}$	&$84.08_{\pm7.42}$	&$63.75_{\pm4.32}$	&$82.46_{\pm7.54}$	&$65.48_{\pm3.68}$	&$62.85_{\pm2.67}$	\\
		&Outdoor-Scene	&$\textbf{68.37}_{\pm4.84}$	&$64.48_{\pm4.97}$	&$44.06_{\pm2.22}$	&$64.77_{\pm4.94}$	&$43.87_{\pm2.51}$	&$44.10_{\pm2.80}$	\\
		\cline{2-8}
		&Avg. score	&$\textbf{69.00}$	&$65.68$	&$52.30$	&$64.54$	&$52.27$	&$52.06$	\\
		\hline
		\multirow{9}{*}{PUR}
		&MSRC-v1	&$\textbf{82.21}_{\pm5.56}$	&$79.33_{\pm5.50}$	&$50.98_{\pm2.99}$	&$79.69_{\pm5.31}$	&$58.29_{\pm5.07}$	&$57.40_{\pm4.57}$	\\
		& ORL	&$\textbf{65.34}_{\pm3.52}$	&$63.38_{\pm3.20}$	&$58.79_{\pm3.19}$	&$62.86_{\pm3.57}$	&$56.56_{\pm2.58}$	&$57.11_{\pm2.66}$	\\
		&WebKB-Texas	&$\textbf{71.31}_{\pm3.75}$	&$67.99_{\pm5.22}$	&$67.83_{\pm2.23}$	&$67.75_{\pm5.60}$	&$68.07_{\pm2.02}$	&$68.02_{\pm2.57}$	\\
		&Caltech101-7	&$\textbf{90.01}_{\pm1.21}$	&$89.09_{\pm1.24}$	&$85.60_{\pm2.04}$	&$88.88_{\pm1.49}$	&$84.97_{\pm1.78}$	&$84.74_{\pm1.78}$	\\
		&Caltech101-20	&$\textbf{79.41}_{\pm1.26}$	&$77.68_{\pm1.61}$	&$72.54_{\pm1.26}$	&$77.76_{\pm1.57}$	&$72.40_{\pm1.61}$	&$72.32_{\pm1.53}$	\\
		&Mfeat	&$\textbf{85.66}_{\pm5.90}$	&$83.04_{\pm5.01}$	&$68.57_{\pm3.65}$	&$82.67_{\pm3.58}$	&$68.20_{\pm3.76}$	&$67.12_{\pm3.36}$	\\
		&handwritten	&$\textbf{88.79}_{\pm7.61}$	&$87.36_{\pm5.36}$	&$66.38_{\pm3.45}$	&$86.11_{\pm5.96}$	&$68.57_{\pm3.25}$	&$65.75_{\pm2.13}$	\\
		&Outdoor-Scene	&$\textbf{69.56}_{\pm3.04}$	&$66.44_{\pm3.34}$	&$45.81_{\pm1.54}$	&$66.51_{\pm3.29}$	&$45.76_{\pm2.06}$	&$45.58_{\pm1.70}$	\\
		\cline{2-8}
		&Avg. score	&$\textbf{79.04}$	&$76.79$	&$64.56$	&$76.53$	&$65.35$	&$64.76$	\\
		\hline
	\end{tabular}
\end{table*}

\begin{table*}[!t]
	\caption{Ablation analysis w.r.t. the \emph{graph learning} (for multi-view clustering) performance ($mean\%_{\pm std\%}$) over 20 runs on the benchmark datasets. }
	\label{t:Ab_gl}
	\centering \begin{tabular}{|m{2cm}<{\centering}|m{1.8cm}<{\centering}m{2.3cm}<{\centering}|m{1.8cm}<{\centering}m{2.3cm}<{\centering}|m{1.8cm}<{\centering}m{2.3cm}<{\centering}|}
		\hline
		\multirow{2}{*}{Datasets}
		&\multicolumn{2}{c|}{NMI}    &\multicolumn{2}{c|}{ACC}  &\multicolumn{2}{c|}{PUR}    \\
		\cline{2-7}
		&JMVFG  &Graph learning with objective~(\ref{eq:obj_pure_gl}) &JMVFG  &Graph learning with objective~(\ref{eq:obj_pure_gl})  &JMVFG  &Graph learning with objective~(\ref{eq:obj_pure_gl}) \\
		\hline
		MSRC-v1	&$76.87_{\pm0.00}$	&$72.46_{\pm0.00}$	&$81.90_{\pm0.00}$	&$77.14_{\pm0.00}$	&$82.38_{\pm0.00}$	&$80.95_{\pm0.00}$	\\
		ORL	&$92.46_{\pm0.00}$	&$91.31_{\pm0.00}$	&$86.25_{\pm0.00}$	&$84.50_{\pm0.00}$	&$88.75_{\pm0.00}$	&$87.00_{\pm0.00}$	\\
		WebKB-Texas	&$25.69_{\pm0.00}$	&$17.04_{\pm0.00}$	&$52.94_{\pm0.00}$	&$55.08_{\pm0.00}$	&$67.91_{\pm0.00}$	&$62.03_{\pm0.00}$	\\
		Caltech101-7	&$77.37_{\pm0.00}$	&$49.76_{\pm0.00}$	&$88.87_{\pm0.00}$	&$63.30_{\pm0.00}$	&$89.35_{\pm0.00}$	&$86.30_{\pm0.00}$	\\
		Caltech101-20	&$69.26_{\pm0.00}$	&$62.70_{\pm0.00}$	&$52.56_{\pm0.00}$	&$58.84_{\pm0.00}$	&$77.03_{\pm0.00}$	&$75.23_{\pm0.00}$	\\
		Mfeat	&$95.47_{\pm0.00}$	&$84.98_{\pm0.00}$	&$98.00_{\pm0.00}$	&$86.40_{\pm0.00}$	&$98.00_{\pm0.00}$	&$86.40_{\pm0.00}$	\\
		handwritten	&$96.34_{\pm0.00}$	&$86.83_{\pm0.00}$	&$98.55_{\pm0.00}$	&$86.75_{\pm0.00}$	&$98.55_{\pm0.00}$	&$86.75_{\pm0.00}$	\\
		Outdoor-Scene	&$63.69_{\pm0.00}$	&$59.02_{\pm0.00}$	&$74.63_{\pm0.00}$	&$66.33_{\pm0.00}$	&$74.63_{\pm0.00}$	&$68.30_{\pm0.00}$	\\
		\hline
		Avg.score	&$74.64$	&$65.51$	&$79.21$	&$72.29$	&$84.58$	&$79.12$	\\
		\hline
	\end{tabular}
\end{table*}

\subsubsection{Ablation w.r.t. Multi-view Graph Learning}
\label{sec:ablation_gl}
This section conducts ablation analysis w.r.t. the multi-view graph learning (for clustering) task. Note that \emph{the first two terms} in objective \eqref{eq:3-11} are designed for feature selection, \emph{the fourth term} is designed for graph learning, and \emph{the third term} bridges the gap between feature selection and graph learning. To verify how the feature selection part (via the first three terms) affects the graph learning performance, we compare the proposed JMVFG method with all terms, corresponding to the objective \eqref{eq:3-11}, against the variant with only the graph learning term, corresponding to the objective \eqref{eq:obj_pure_gl}. As shown in Table~\ref{t:Ab_gl}, the incorporation of the feature selection terms leads to significant improvements on the graph learning performance, where the average NMI(\%), ACC(\%), and PUR(\%) scores (across eight datasets) go from 65.51, 72.29, and 79.12 to 74.64, 79.21, 84.58, respectively.

\section{Conclusion and Future Work}
\label{sec:conclusion}
In this paper, we formulate the multi-view feature selection and the multi-view graph learning in a unified framework termed JMVFG, which is capable of simultaneously leveraging the cluster structure learning (via orthogonal decomposition), the cross-space locality preservation, and the global multi-graph fusion. Specifically, to learn the discriminant cluster structure, the projection matrix is constrained via orthogonal decomposition, where the target matrix is decomposed into a view-specific basis matrix and a view-consistent cluster indicator matrix. A multi-view graph learning term is incorporated to learn the global graph structure in the original space, which is further bridged to the multi-view feature selection through a cross-space locality preservation term. By seamlessly modeling multi-view feature selection and multi-view graph learning in a joint objective function, we then design an efficient optimization algorithm with theoretically-proved convergence. Extensive experiments on eight multi-view datasets have confirmed the superiority of our approach.

In the future work, we plan to extend our framework from the general graph structure to the bipartite graph structure \cite{huang19_tkde,Fang2023} in order to enhance its scalability for large-scale datasets. Besides, as the deep clustering has been a popular research topic in recent years, the integration of our joint learning methodology into the deep clustering framework \cite{zhong23_npl,Deng2023} may be another interesting research direction in the future work.

\section*{Acknowledgments}

This work was supported by the NSFC (61976097, 62276277, \& U1811263) and the Natural Science Foundation of Guangdong Province (2021A1515012203).

\bibliographystyle{IEEEtran}
\bibliography{myref}

\begin{IEEEbiography}[{\includegraphics[width=1in,height=1.25in,clip,keepaspectratio]{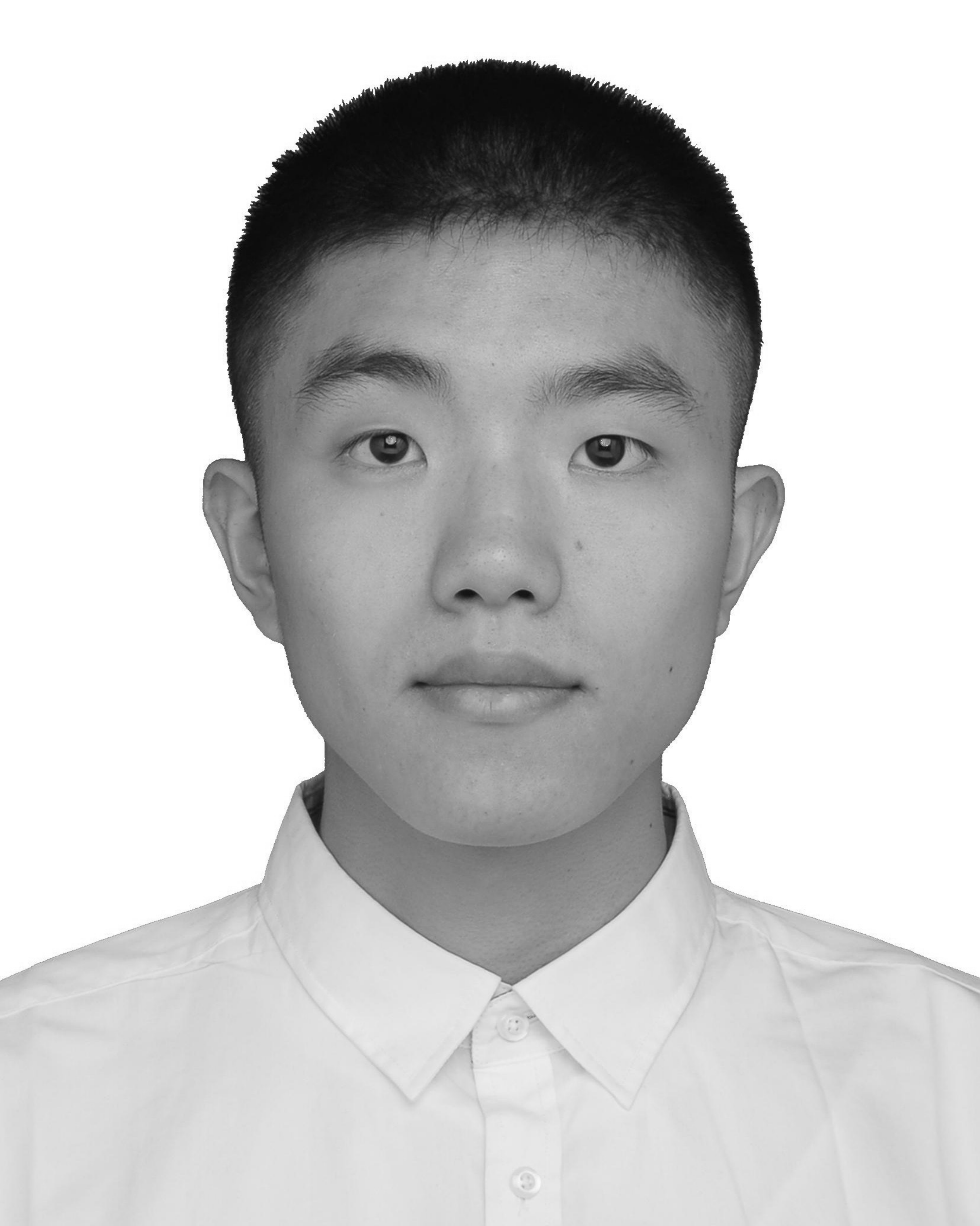}}]{Siguo Fang}
	received the B.S. degree in information and computing sciences from the China Jiliang University, Hangzhou, China, in 2021. He is currently pursuing the master degree in computer science with the College of Mathematics and Informatics, South China Agricultural University, Guangzhou, China. His research interests include data mining and machine learning. His work has appeared in IEEE TNNLS and IEEE TETCI.
\end{IEEEbiography}

\begin{IEEEbiography}[{\includegraphics[width=1in,height=1.25in,clip,keepaspectratio]{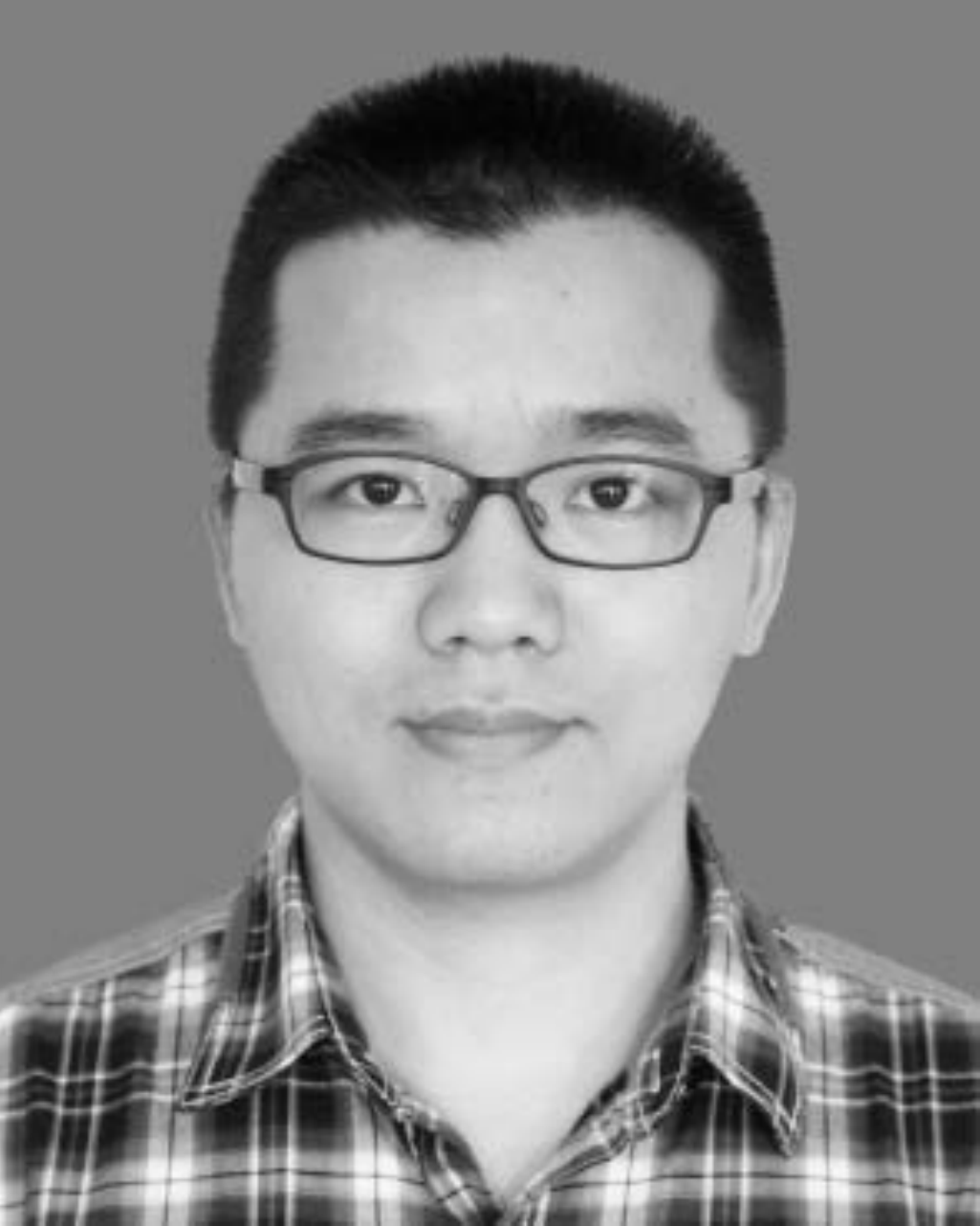}}]{Dong Huang}
	received the B.S. degree in computer science in 2009 from South China University of Technology, Guangzhou, China. He received the M.Sc. degree in computer science in 2011 and the Ph.D. degree in computer science in 2015, both from Sun Yat-sen University, Guangzhou, China. He joined South China Agricultural University in 2015, where he is currently an Associate Professor with the College of Mathematics and Informatics. From July 2017 to July 2018, he was a visiting fellow with the School of Computer Science and Engineering, Nanyang Technological University, Singapore. His research interests include data mining and machine learning. He has published more than 70 papers in refereed journals and conferences, such as IEEE TKDE, IEEE TNNLS, IEEE TCYB, IEEE TSMC-S, IEEE TETCI, ACM TKDD, SIGKDD, AAAI, and ICDM. He was the recipient of the 2020 ACM Guangzhou Rising Star Award.
\end{IEEEbiography}

\begin{IEEEbiography}[{\includegraphics[width=1in,height=1.25in,clip,keepaspectratio]{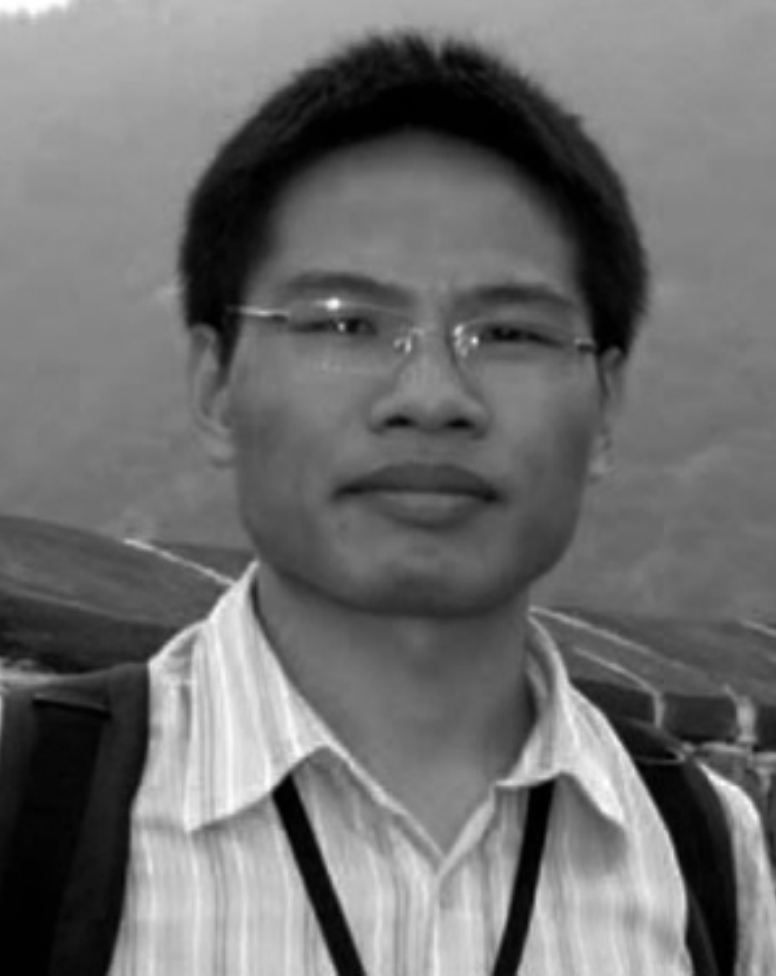}}]{Chang-Dong Wang}
	received the B.S. degree in applied mathematics in 2008, the M.Sc. degree in computer science in 2010, and the Ph.D. degree in computer science in 2013, all from Sun Yat-sen University, Guangzhou, China. He was a visiting student at the University of Illinois at Chicago from January 2012 to November 2012. He is currently an Associate Professor with the School of Data and Computer Science, Sun Yat-sen University, Guangzhou, China. His current research interests include machine learning and data mining. He has published more than 100 scientific papers in international journals and conferences such as IEEE TPAMI, IEEE TKDE, IEEE TNNLS, IEEE TSMC-C, ACM TKDD, Pattern Recognition, SIGKDD, ICDM and SDM. His ICDM 2010 paper won the Honorable Mention for Best Research Paper Award. He was awarded 2015 Chinese Association for Artificial Intelligence (CAAI) Outstanding Dissertation.
\end{IEEEbiography}

\begin{IEEEbiography}[{\includegraphics[width=1in,height=1.25in,clip,keepaspectratio]{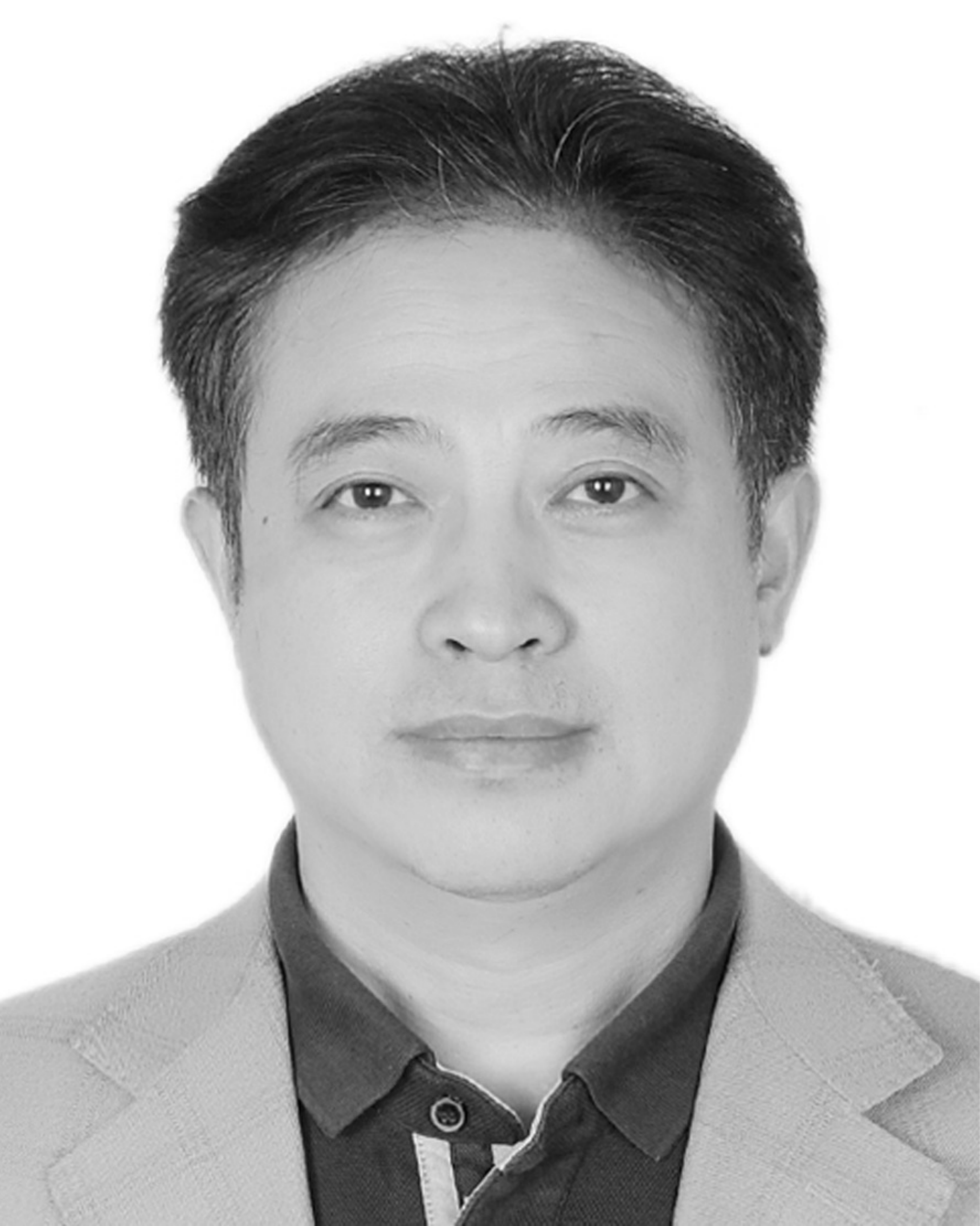}}]{Yong Tang} is the founder of SCHOLAT, a kind of scholar social network. He is now a Professor and Dean of School of Computer Science at South China Normal University. He got his BS and MSc degrees from Wuhan University in 1985 and 1990 respectively, and PhD degree from University of Science and Technology of China in 2001, all in computer science. Before joining South China Normal University in 2009, he was vice Dean of School of Information of Science and Technology at Sun Yat-Sen University. He has published more than 200 papers and books. He has supervised more than 40 PhD students since 2003 and more than 100 Master students since 1996. His main research areas include data and knowledge engineering, social networking and collaborative computing. He currently serves as the director of technical committee on collaborative computing of China Computer Federation (CCF) and the executive vice president of Guangdong Computer Academy. For more information, please visit \url{https://scholat.com/ytang}.
\end{IEEEbiography}

\appendix

In the Appendix, we provide the theoretical details for the convergence analysis in Section \ref{TheoreticalAnalysis}, and report additional experimental results in Section \ref{ExperimentalAnalysis}.

\subsection{Theoretical Details of Convergence Analysis}\label{TheoreticalAnalysis}
In this section, we theoretically prove that the objective function monotonically decreases after each variable is updated in each iteration under the mild assumption that $\textbf{A}^{(v)}$ should be initialized to be symmetric.

\subsubsection{Update $\bm\delta$}
The subproblem that only relates to $\bm\delta$ is
\begin{align}
\label{eq:moDelta}
&\min_{\bm\delta}~\mathcal{L}(\bm\delta)=\sum_{v=1}^V~||\textbf{S}-\delta^{(v)}\textbf{A}^{(v)}||_F^2\notag\\
&~s.t.~\bm\delta^T\textbf{1}=1,\bm\delta\ge0.
\end{align}
Let $\bm\delta_{(t+1)}$ denote the value obtained in the $(t+1)$-th iteration. Then we have Theorem 1.
\begin{theorem}
In iteration $t+1$, $\mathcal{L}(\bm\delta_{(t+1)})\le\mathcal{L}(\bm\delta_{(t)})$ after solving the subproblem \eqref{eq:moDelta}.
\end{theorem}
\begin{proof}
According to \cite{COPT}, we can prove that subproblem \eqref{eq:moDelta} is a convex optimization problem. First, the equality constraint function $f(\bm\delta)=\bm\delta^T\textbf{1}-1$ is affine. Second, the inequality constraint functions $f(\delta^{(v)})=-\delta^{(v)}$ are linear functions and therefore also convex (for $v=1,\cdots,V$). Finally, we have
\begin{align}
\mathcal{L}(\bm\delta)=&\sum_v[Tr(\textbf{S}\textbf{S}^T)+(\delta^{(v)})^2Tr(\textbf{A}^{(v)}(\textbf{A}^{(v)})^T)\notag\\
&-2\delta^{(v)}Tr(\textbf{A}^{(v)}\textbf{S}^T)].
\end{align}
Let $Tr(\textbf{A}^{(v)}\textbf{S}^T)=p^{(v)}\ge0$, $Tr[\textbf{A}^{(v)}(\textbf{A}^{(v)})^T]=q^{(v)}>0$ and $Tr(\textbf{S}\textbf{S}^T)=C$. We have
\begin{align}
\mathcal{L}(\bm\delta)=&\sum_v[q^{(v)}(\delta^{(v)})^2-2p^{(v)}\delta^{(v)}+C].
\end{align}
Since $q^{(v)}>0$, we know $f_v(\delta^{(v)})=q^{(v)}(\delta^{(v)})^2-2p^{(v)}\delta^{(v)}+C$ are convex functions (for $v=1,\cdots,V$). Note that convexity is preserved under addition. Therefore, the objective function $\mathcal{L}(\bm\delta)=\sum_vf_v(\delta^{(v)})$ is convex, and the subproblem \eqref{eq:moDelta} is a convex optimization problem.

When the problem is convex, the KKT conditions are sufficient for the points to be optimal \cite{COPT}. In the optimization part, we obtain the point that satisfies the KKT conditions, i.e. the optimal point. Thus, for an arbitrary $\bm\delta$, in iteration $t+1$, we have $\mathcal{L}(\bm\delta_{(t+1)})\le\mathcal{L}(\bm\delta)$, especially $\mathcal{L}(\bm\delta_{(t+1)})\le\mathcal{L}(\bm\delta_{(t)})$.
\end{proof}

\subsubsection{Update $\textbf{W}^{(v)}$}
\begin{lemma}
\label{Le:1}
For an arbitrary matrix $\textbf{A}=(a_{ij})_{m\times d}$ and a diagonal matrix $\textbf{B}=(b_{ii})_{m\times m}$, we have
\begin{align}
	Tr(\textbf{A}^T\textbf{B}\textbf{A})=\sum_{i=1}^m b_{ii}||\textbf{a}_{i\cdot}||_2^2.
\end{align}
\end{lemma}
\begin{proof}
Since
\begin{align}
\textbf{A}^T\textbf{B}\textbf{A}=\begin{pmatrix}\textbf{a}_{\cdot1}^T \\ \textbf{a}_{\cdot2}^T \\ \vdots \\ \textbf{a}_{\cdot d}^T\end{pmatrix}\textbf{B}\begin{pmatrix}\textbf{a}_{\cdot 1}&\textbf{a}_{\cdot 2}&\cdots&\textbf{a}_{\cdot d}\end{pmatrix},
\end{align}
then we have
\begin{align}
	&Tr(\textbf{A}^T\textbf{B}\textbf{A})=\sum_{i=1}^d\textbf{a}_{\cdot i}^T\textbf{B}\textbf{a}_{\cdot i}\notag\\
	&=\sum_{i=1}^d\begin{pmatrix}a_{1i}&a_{2i}&\cdots&a_{mi}\end{pmatrix}\begin{pmatrix}b_{11}&&&\\&b_{22}&&\\&&\ddots&\\&&&b_{mm}\end{pmatrix}\begin{pmatrix}a_{1i}\\a_{2i}\\\vdots\\a_{mi}\end{pmatrix}\notag\\
	&=\sum_{i=1}^d\sum_{j=1}^m{b_{jj}a_{ji}^2}=\sum_{j=1}^m{b_{jj}\sum_{i=1}^da_{ji}^2}=\sum_{i=1}^m{b_{ii}||\textbf{a}_{i\cdot}||_2^2}.
\end{align}
\end{proof}

The subproblem that is only related to $\textbf{W}^{(v)}$ is
\begin{align}
\label{eq:moW}
\min\limits_{\textbf{W}^{(v)}}\ \mathcal{L}(\textbf{W}^{(v)})=&||(\textbf{W}^{(v)})^T\textbf{X}^{(v)}-\textbf{B}^{(v)}\textbf{H}^T||_F^2+\eta||\textbf{W}^{(v)}||_{2,1}\notag\\
&+\gamma Tr[(\textbf{W}^{(v)})^T\textbf{X}^{(v)}\textbf{L}(\textbf{X}^{(v)})^T\textbf{W}^{(v)}].
\end{align}
Let $\textbf{W}^{(v)}_{(t)}$ and $\textbf{D}^{(v)}_{(t)}$ represent the values obtained in the $t$-th iteration. We can obtain the value of $\textbf{W}^{(v)}_{(t+1)}$ for the $(t+1)$-th iteration as follows:
\begin{align}
\label{eq:upW}
\textbf{W}^{(v)}_{(t+1)}=&[\textbf{X}^{(v)}(\textbf{X}^{(v)})^T+\gamma\textbf{X}^{(v)}\textbf{L}(\textbf{X}^{(v)})^T\notag\\
&+\eta\textbf{D}^{(v)}_{(t)}]^{-1}\textbf{X}^{(v)}\textbf{H}(\textbf{B}^{(v)})^T.
\end{align}

\begin{theorem}
In iteration $t+1$, $\mathcal{L}(\textbf{W}^{(v)}_{(t+1)})\le\mathcal{L}(\textbf{W}^{(v)}_{(t)})$ after updating with Eq. \eqref{eq:upW}.
\end{theorem}
\begin{proof}
Let $\overline{\mathcal{L}}_2(\textbf{W}^{(v)})=Tr[(\textbf{W}^{(v)})^T\textbf{D}^{(v)}_{(t)}\textbf{W}^{(v)}]$. Note that $\textbf{D}^{(v)}_{(t)}$ is not related to $\textbf{W}^{(v)}$, but related to $\textbf{W}^{(v)}_{(t)}$. According to \cite{Mcook}, we have
\begin{align}
\frac{d\overline{\mathcal{L}}_2(\textbf{W}^{(v)})}{d\textbf{W}^{(v)}}=2\textbf{D}^{(v)}_{(t)}\textbf{W}^{(v)}.
\end{align}
Thus, we consider the following problem
\begin{align}
\label{eq:moW_n}
\min\limits_{\textbf{W}^{(v)}}\ \overline{\mathcal{L}}(\textbf{W}^{(v)})=&||(\textbf{W}^{(v)})^T\textbf{X}^{(v)}-\textbf{B}^{(v)}\textbf{H}^T||_F^2\notag\\
&+\eta Tr[(\textbf{W}^{(v)})^T\textbf{D}^{(v)}_{(t)}\textbf{W}^{(v)}]\notag\\
&+\gamma Tr[(\textbf{W}^{(v)})^T\textbf{X}^{(v)}\textbf{L}(\textbf{X}^{(v)})^T\textbf{W}^{(v)}].
\end{align}
Let $d\overline{\mathcal{L}}(\textbf{W}^{(v)})/d\textbf{W}^{(v)}=0$, we can know that $\textbf{W}^{(v)}_{(t+1)}$ obtained from Eq. \eqref{eq:upW} is the optimal solution of problem \eqref{eq:moW_n}. Therefore, for an arbitrary $\textbf{W}^{(v)}$, we know that $\overline{\mathcal{L}}(\textbf{W}^{(v)}_{(t+1)})\le\overline{\mathcal{L}}(\textbf{W}^{(v)})$, especially $\overline{\mathcal{L}}(\textbf{W}^{(v)}_{(t+1)})\le\overline{\mathcal{L}}(\textbf{W}^{(v)}_{(t)})$, that is
\begin{align}
\label{eq:17}
&||(\textbf{W}^{(v)}_{(t+1)})^T\textbf{X}^{(v)}-\textbf{B}^{(v)}\textbf{H}^T||_F^2+\eta Tr[(\textbf{W}^{(v)}_{(t+1)})^T\textbf{D}^{(v)}_{(t)}\textbf{W}^{(v)}_{(t+1)}]\notag\\
&+\gamma Tr[(\textbf{W}^{(v)}_{(t+1)})^T\textbf{X}^{(v)}\textbf{L}(\textbf{X}^{(v)})^T\textbf{W}^{(v)}_{(t+1)}]\le\notag\\
&||(\textbf{W}^{(v)}_{(t)})^T\textbf{X}^{(v)}-\textbf{B}^{(v)}\textbf{H}^T||_F^2+\eta Tr[(\textbf{W}^{(v)}_{(t)})^T\textbf{D}^{(v)}_{(t)}\textbf{W}^{(v)}_{(t)}]\notag\\
&+\gamma Tr[(\textbf{W}^{(v)}_{(t)})^T\textbf{X}^{(v)}\textbf{L}(\textbf{X}^{(v)})^T\textbf{W}^{(v)}_{(t)}].
\end{align}
It has been proved in \cite{L21} that for any nonzero vectors \textbf{u}, $\textbf{u}_t\in\mathbb{R}^c$, the following inequality holds
\begin{align}
||\textbf{u}||_2-\frac{||\textbf{u}||_2^2}{2||\textbf{u}_t||_2}\le||\textbf{u}_t||_2-\frac{||\textbf{u}_t||_2^2}{2||\textbf{u}_t||_2}.
\end{align}
We replace $\textbf{u}$ and $\textbf{u}_t$ in the above inequality by $\textbf{w}^{(v)}_{i\cdot,(t+1)}$ and $\textbf{w}^{(v)}_{i\cdot,(t)}$, respectively, and sum them from $i=1$ to $m^{(v)}$, where $\textbf{w}^{(v)}_{i\cdot,(t+1)}$ denotes the $i$-th row vector of $\textbf{W}^{(v)}_{(t+1)}$. Then we have
\begin{align}
&||\textbf{W}^{(v)}_{(t+1)}||_{2,1}-\sum_{i=1}^{m^{(v)}}\frac{||\textbf{w}^{(v)}_{i\cdot,(t+1)}||_2^2}{2||\textbf{w}^{(v)}_{i\cdot,(t)}||_2}\notag\\
\le&||\textbf{W}^{(v)}_{(t)}||_{2,1}-\sum_{i=1}^{m^{(v)}}\frac{||\textbf{w}^{(v)}_{i\cdot,(t)}||_2^2}{2||\textbf{w}^{(v)}_{i\cdot,(t)}||_2}.
\end{align}
Since the diagonal element of $\textbf{D}^{(v)}_{(t)}$ is $d_{ii,(t)}^{(v)}=\frac{1}{2||\textbf{W}^{(v)}_{i\cdot,(t)}||_2}$, according to Lemma \ref{Le:1}, we have
\begin{align}
\label{eq:20}
&||\textbf{W}^{(v)}_{(t+1)}||_{2,1}-Tr[(\textbf{W}^{(v)}_{(t+1)})^T\textbf{D}^{(v)}_{(t)}\textbf{W}^{(v)}_{(t+1)}]\notag\\
\le&||\textbf{W}^{(v)}_{(t)}||_{2,1}-Tr[(\textbf{W}^{(v)}_{(t)})^T\textbf{D}^{(v)}_{(t)}\textbf{W}^{(v)}_{(t)}].
\end{align}
From Eqs. \eqref{eq:17} and \eqref{eq:20}, we have
\begin{align}
&||(\textbf{W}^{(v)}_{(t+1)})^T\textbf{X}^{(v)}-\textbf{B}^{(v)}\textbf{H}^T||_F^2+\eta||\textbf{W}^{(v)}_{(t+1)}||_{2,1}\notag\\
&+\gamma Tr[(\textbf{W}^{(v)}_{(t+1)})^T\textbf{X}^{(v)}\textbf{L}(\textbf{X}^{(v)})^T\textbf{W}^{(v)}_{(t+1)}]\notag\\
\le&||(\textbf{W}^{(v)}_{(t)})^T\textbf{X}^{(v)}-\textbf{B}^{(v)}\textbf{H}^T||_F^2+\eta||\textbf{W}^{(v)}_{(t)}||_{2,1}\notag\\
&+\gamma Tr[(\textbf{W}^{(v)}_{(t)})^T\textbf{X}^{(v)}\textbf{L}(\textbf{X}^{(v)})^T\textbf{W}^{(v)}_{(t)}].
\end{align}
Therefore, we obtain $\mathcal{L}(\textbf{W}^{(v)}_{(t+1)})\le\mathcal{L}(\textbf{W}^{(v)}_{(t)})$.
\end{proof}

\subsubsection{Update $\textbf{B}^{(v)}$}
According to \cite{OMF}, we have the following definition.
\begin{definition}
With $\textbf{Q}\in\mathbb{R}^{m\times n}$ where $n\le m$, Assume $\textbf{X}$ with $rank(\textbf{X})=n$ and $\textbf{B}$ be known real matrices of correct dimensions. The optimization problem
\begin{align}
\min_{\textbf{Q}}||\textbf{Q}\textbf{X}-\textbf{B}||_F^2\quad s.t.~\textbf{Q}^T\textbf{Q}=\textbf{I}_n
\end{align}
is called \emph{orthogonal Procrustes problem (OPP)}.
\end{definition}
The OPP has an analytical solution, which can be derived by using the singular value decomposition (SVD).
This subproblem that is only related to $\textbf{B}^{(v)}$ is an OPP, so the updating $\textbf{B}^{(v)}$ of each iteration can make the value of the objective function decrease monotonically \cite{OMF}.

\subsubsection{Update $\textbf{Z}$}
The subproblem that is simply related to $\textbf{Z}$ is
\begin{align}
\label{eq:moZ}
\min\limits_\textbf{Z}~\mathcal{L}(\textbf{Z})=||\textbf{H}-\textbf{Z}||_F^2\quad s.t.\textbf{Z}\ge0.
\end{align}
In iteration $t+1$, it is obvious that
\begin{align}
\textbf{Z}_{(t+1)}=(z_{ij,(t+1)})_{n\times c},\quad z_{ij,(t+1)}=max(h_{ij},0)
\end{align}
is the optimal solution of subproblem \eqref{eq:moZ}. Thereby, for an arbitrary $\textbf{Z}$, in iteration $t+1$, we know $\mathcal{L}(\textbf{Z}_{(t+1)})\le\mathcal{L}(\textbf{Z})$, especially $\mathcal{L}(\textbf{Z}_{(t+1)})\le\mathcal{L}(\textbf{Z}_{(t)})$.

\subsubsection{Update $\textbf{H}$}
According to the update of $\textbf{H}$ in the main manuscript, similar to updating $\textbf{B}^{(v)}$, the subproblem that is only related to $\textbf{H}$ can also be transformed into an OPP. Thereby, the update of $\textbf{H}$ in each iteration can make the value of the objective function decrease monotonically.

\begin{figure}[!t] 
	\begin{center}
		{\subfigure[{\scriptsize MSRC-v1}]
			{\includegraphics[width=0.24\columnwidth]{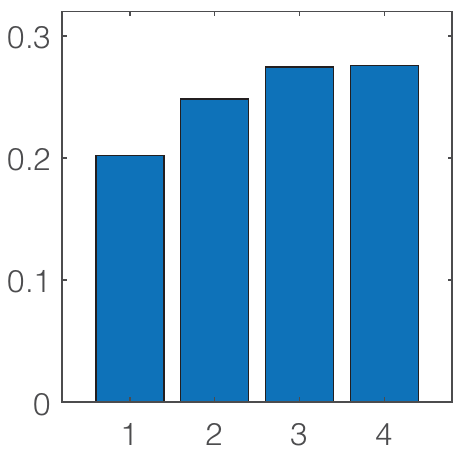}}}
		{\subfigure[{\scriptsize ORL}]
			{\includegraphics[width=0.24\columnwidth]{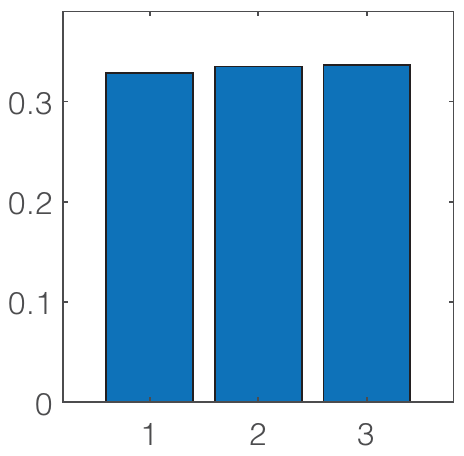}}}
		{\subfigure[{\scriptsize WebKB-Texas}]
			{\includegraphics[width=0.24\columnwidth]{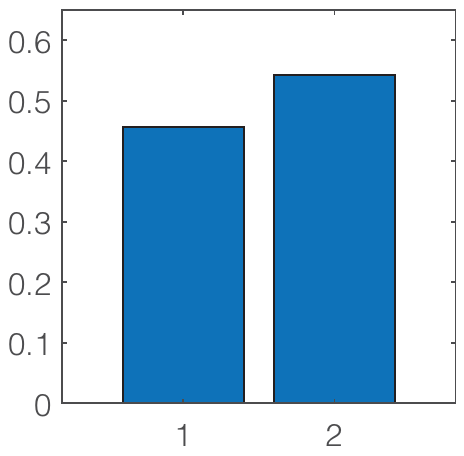}}}
		{\subfigure[{\scriptsize Caltech101-7}]
			{\includegraphics[width=0.24\columnwidth]{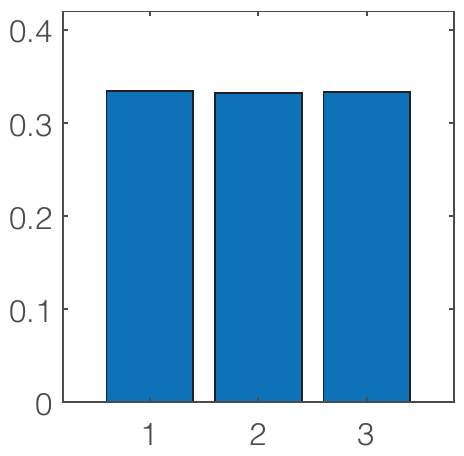}}}\vskip 0.05in
		{\subfigure[{\scriptsize Caltech101-20}]
			{\includegraphics[width=0.24\columnwidth]{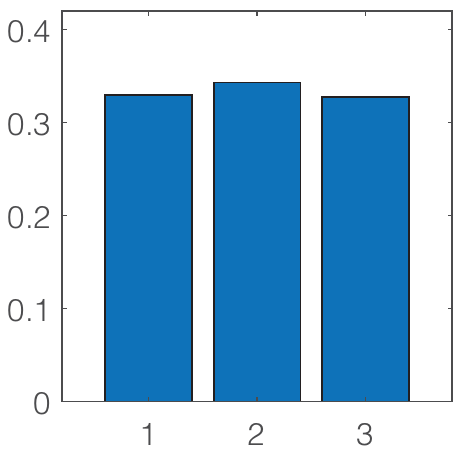}}}
		{\subfigure[{\scriptsize Mfeat}]
			{\includegraphics[width=0.24\columnwidth]{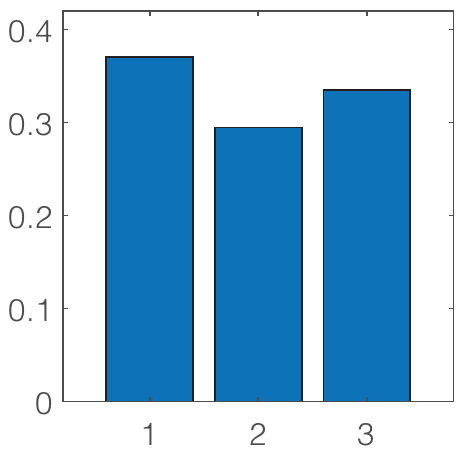}}}
		{\subfigure[{\scriptsize Handwritten}]
			{\includegraphics[width=0.24\columnwidth]{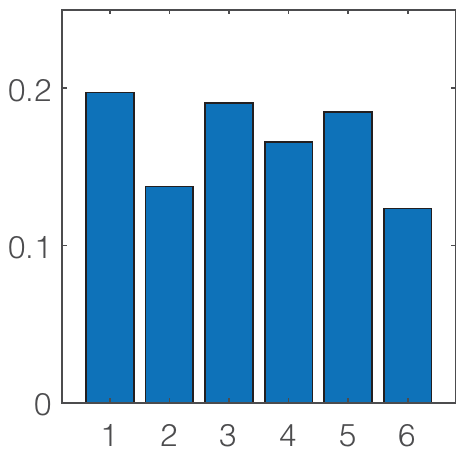}}}
		{\subfigure[{\scriptsize Outdoor-Scene}]
			{\includegraphics[width=0.24\columnwidth]{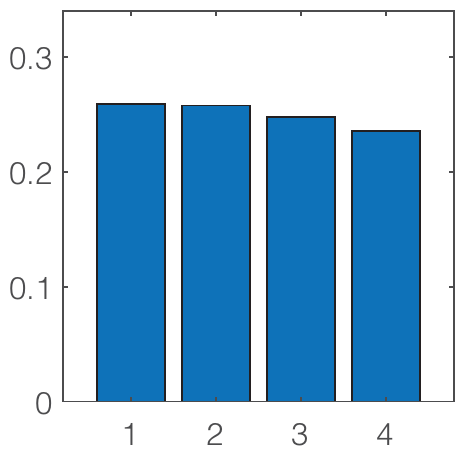}}}
		\caption{The learned view weights on the benchmark datasets.}
		\label{fig:delta}
	\end{center}
\end{figure}

\subsubsection{Update $\textbf{S}$}
The subproblem that is only related to $\textbf{S}$ is
\begin{align}
\label{eq:moS}
\min\limits_\textbf{S}\ &\mathcal{L}(\textbf{S})=\sum\limits_{v=1}^V[\gamma Tr(\textbf{Y}^{(v)}\textbf{L}(\textbf{Y}^{(v)})^T)+\beta||\textbf{S}-\delta^{(v)}\textbf{A}^{(v)}||_F^2]\notag\\
s.t.\ &\textbf{S1}=\textbf{1},\textbf{S}\ge0.
\end{align}
where $\textbf{L}=\textbf{P}-\bar{\textbf{S}}$ and $\bar{\textbf{S}}=(\textbf{S}^T+\textbf{S})/2$. If the affinity matrices $\textbf{A}^{(v)}$, for $v=1,\cdots,V$, are symmetric, then the solution $\textbf{S}$ of the subproblem \eqref{eq:moS} is also symmetric.

\begin{theorem}
If $\textbf{A}^{(v)}$ is symmetric for each $v$, in iteration $t+1$, $\mathcal{L}(\textbf{S}_{(t+1)})\le\mathcal{L}(\textbf{S}_{(t)})$ after solving the subproblem \eqref{eq:moS}.
\end{theorem}

\begin{proof}
According to the update of $\textbf{S}$ in the main manuscript, the above subproblem \eqref{eq:moS} is equivalent to
\begin{equation}
{\rm for~each~}i~\left\{
\begin{aligned}
&\min~||\textbf{s}_{i\cdot}-\textbf{r}_{i\cdot}||_2^2\\
&~s.t.~~\textbf{s}_{i\cdot}\textbf{1}=1,\textbf{s}_{i\cdot}\ge0.
\end{aligned}
\right.
\end{equation}
According to \cite{Anew}, we can obtain the solution that satisfies the KKT conditions. On the one hand, similar to updating $\bm\delta$, the equality constraint function $f(\textbf{s}_{i\cdot})=\textbf{s}_{i\cdot}\textbf{1}-1$ is convex, and the inequality constraint functions $f(s_{ij})=-s_{ij}$ are also convex (for $j=1,\cdots,n$). On the other hand, the objective function $f(\textbf{s}_{i\cdot})=||\textbf{s}_{i\cdot}-\textbf{r}_{i\cdot}||_2^2$ is convex. Thereby, the solution satisfying the KKT conditions is optimal, which indicates that $\mathcal{L}(\textbf{S}_{(t+1)})\le\mathcal{L}(\textbf{S}_{(t)})$.
\end{proof}

\subsection{Additional Experimental Analysis}\label{ExperimentalAnalysis}
In this section, we provide additional experimental analysis w.r.t. (i) the distribution of the learned view weights in Section~\ref{sec:weight_distribution}, (ii) the comparison results via the indicator matrix and the similarity matrix in Section~\ref{sec:comp_indicator}, and (iii) the influence of the number of nearest neighbors in Section~\ref{sec:compK}.

\subsubsection{Distribution of the Learned View Weights}
\label{sec:weight_distribution}
In this section, we illustrate distribution of the view weights learned by our JMVFG approach on the eight benchmark datasets. As mentioned in the main paper, the parameter $\bm\delta=[\delta^{(1)},\delta^{(2)},...,\delta^{(V)}]^T$ is a learnable parameter in JMVFG, whose value can measure the influence of each view and is iteratively learned during the optimization. As shown in Fig.~\ref{fig:delta}, the contributions of multiple views are different, which lead to different learned view weights in the process of minimizing the objective function of our framework.

\begin{figure}[!t] 
	\begin{center}
		{\subfigure[{\scriptsize MSRC-v1}]
			{\includegraphics[width=0.24\columnwidth]{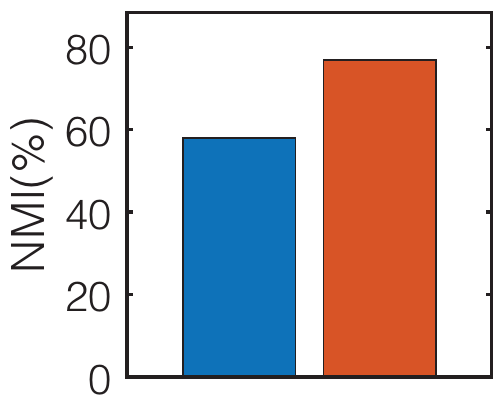}}}
		{\subfigure[{\scriptsize ORL}]
			{\includegraphics[width=0.24\columnwidth]{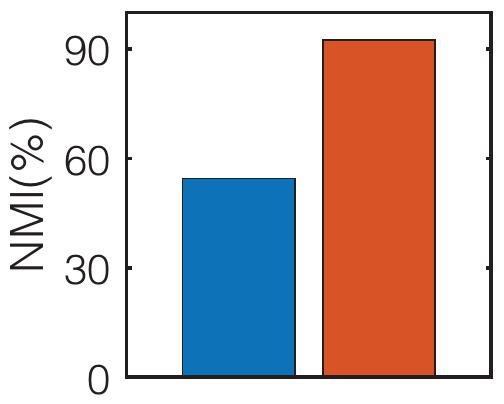}}}
		{\subfigure[{\scriptsize WebKB-Texas}]
			{\includegraphics[width=0.24\columnwidth]{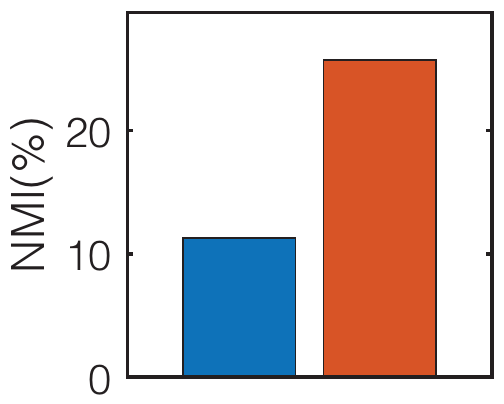}}}
		{\subfigure[{\scriptsize Caltech101-7}]
			{\includegraphics[width=0.24\columnwidth]{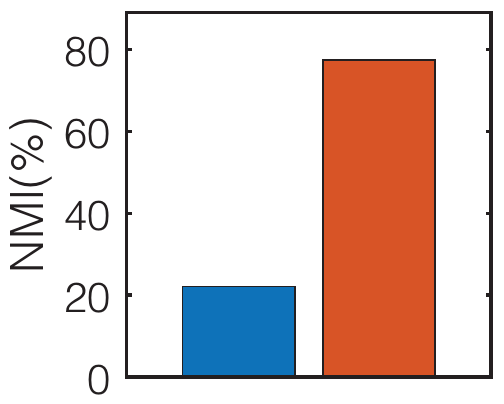}}}\vskip -0.05in
		{\subfigure[{\scriptsize Caltech101-20}]
			{\includegraphics[width=0.24\columnwidth]{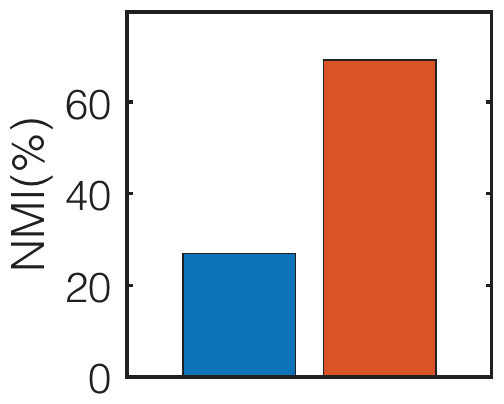}}}
		{\subfigure[{\scriptsize Mfeat}]
			{\includegraphics[width=0.24\columnwidth]{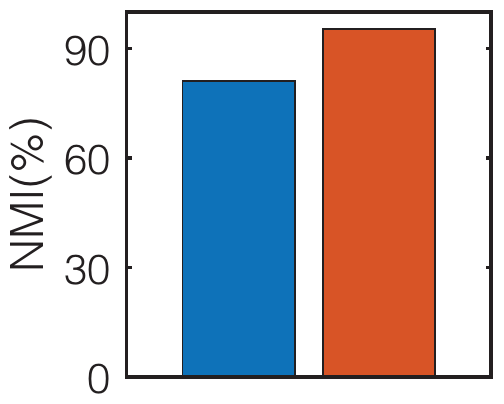}}}
		{\subfigure[{\scriptsize Handwritten}]
			{\includegraphics[width=0.24\columnwidth]{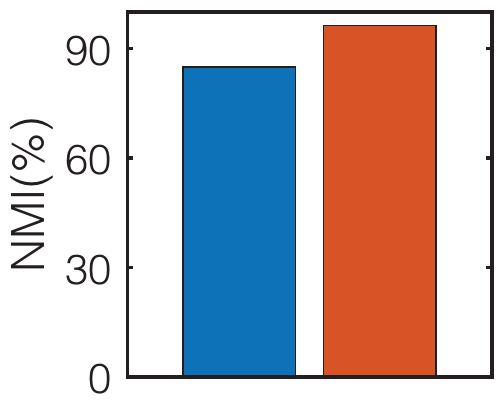}}}
		{\subfigure[{\scriptsize Outdoor-Scene}]
			{\includegraphics[width=0.24\columnwidth]{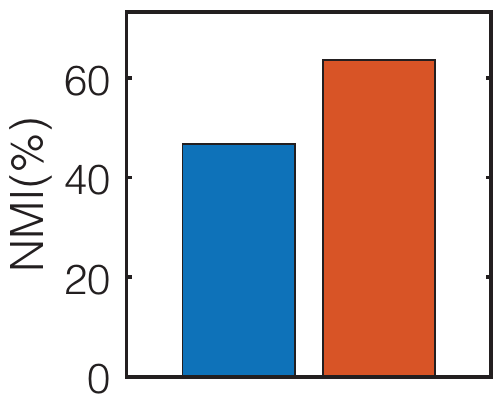}}}
		{\subfigure
			{\includegraphics[width=0.4\columnwidth]{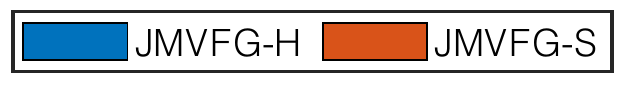}}}\vskip -0.03in
		\caption{The NMI(\%) scores by JMVFG using the indicator matrix against using the similarity matrix.}\vskip -0.1in
		\label{fig:H_vs_S}
	\end{center}
\end{figure}

\begin{figure}[!t] 
	\begin{center}
		{\subfigure[{\scriptsize MSRC-v1}]
			{\includegraphics[width=0.24\columnwidth]{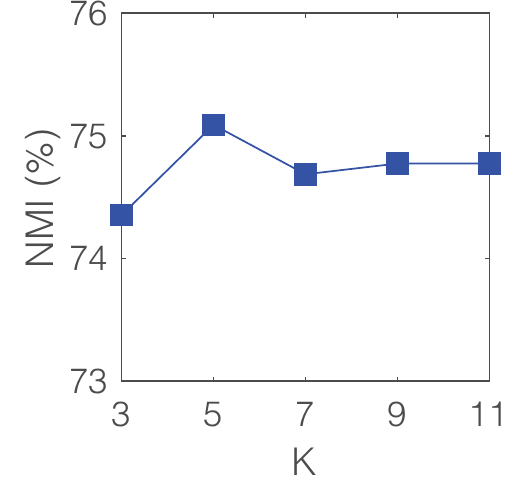}}}
		{\subfigure[{\scriptsize ORL}]
			{\includegraphics[width=0.24\columnwidth]{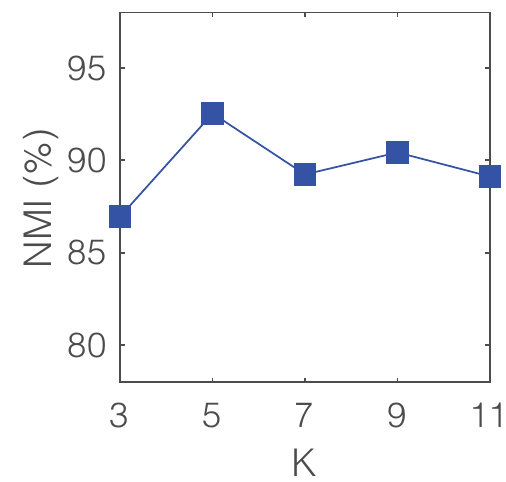}}}
		{\subfigure[{\scriptsize WebKB-Texas}]
			{\includegraphics[width=0.24\columnwidth]{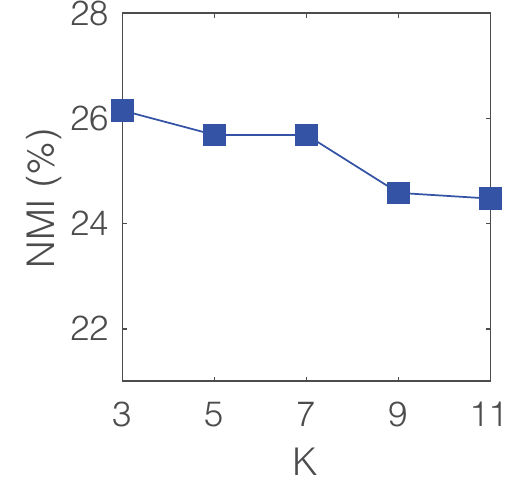}}}
		{\subfigure[{\scriptsize Caltech101-7}]
			{\includegraphics[width=0.24\columnwidth]{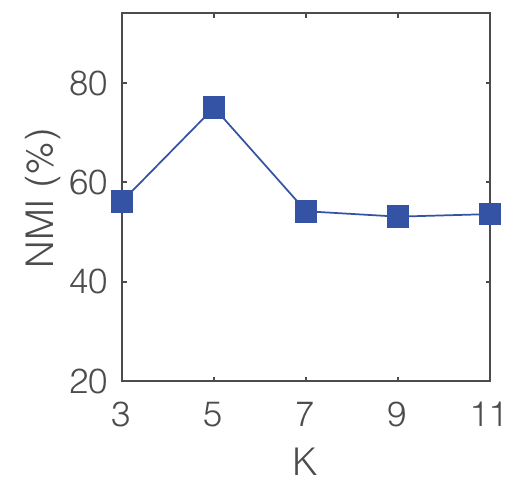}}}\vskip 0.05in
		{\subfigure[{\scriptsize Caltech101-20}]
			{\includegraphics[width=0.24\columnwidth]{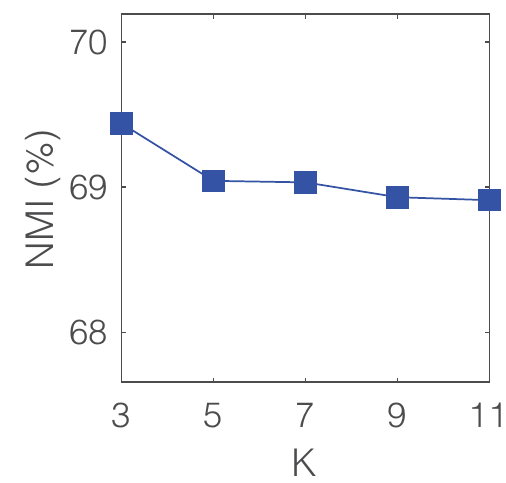}}}
		{\subfigure[{\scriptsize Mfeat}]
			{\includegraphics[width=0.24\columnwidth]{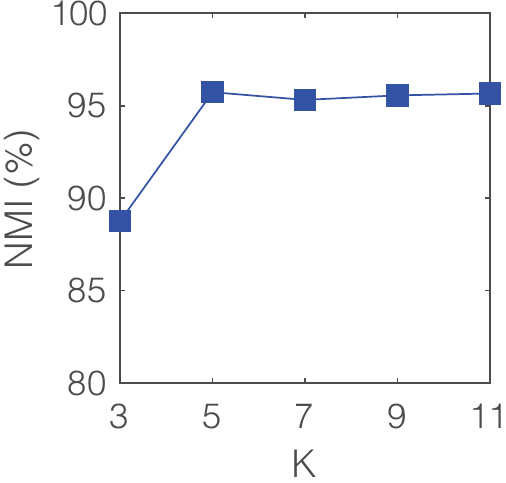}}}
		{\subfigure[{\scriptsize Handwritten}]
			{\includegraphics[width=0.24\columnwidth]{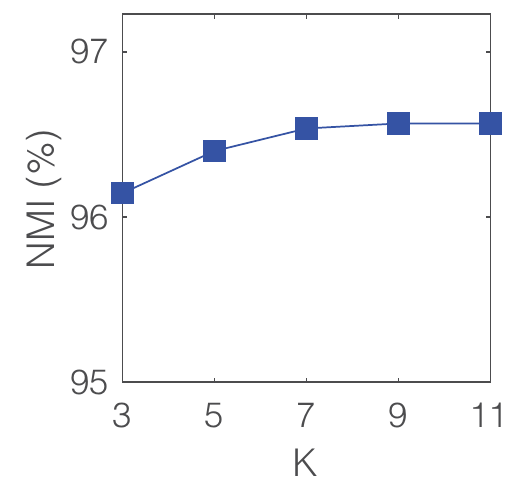}}}
		{\subfigure[{\scriptsize Outdoor-Scene}]
			{\includegraphics[width=0.24\columnwidth]{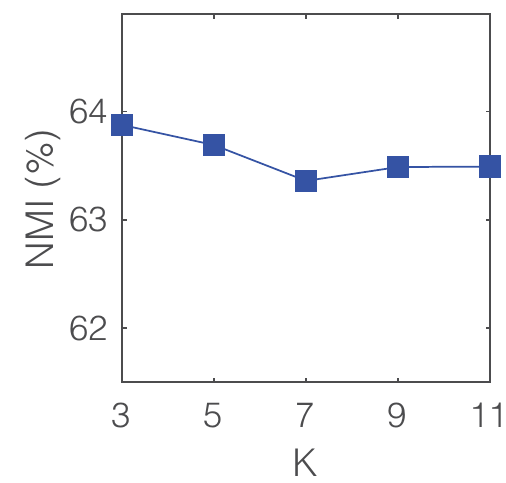}}}\vskip -0.03in
		\caption{Average performance (w.r.t. NMI(\%)) of JMVFG with varying number of nearest neighbors $K$.}\vskip -0.1in
		\label{fig:paraK}
	\end{center}
\end{figure}

\subsubsection{Comparison of the Results via the Indicator Matrix and the Similarity Matrix}
\label{sec:comp_indicator}
In this section, we compare the clustering performances of our JMVFG method using the indicator matrix or the similarity matrix for the multi-view clustering task. Specifically, for the variant using the indicator matrix (denoted as JMVFG-H), we evaluate its performance by using the column index of the largest element in each row of the indicator matrix to form the clustering result. For the variant using the similarity matrix (denoted as JMVFG-S), we evaluate its performance by peforming spectral clustering on the learned similarity matrix. As shown in Fig.~\ref{fig:H_vs_S}, JMVFG-S generally lead to more competitive clustering performance than JMVFG-H on the datasets, probably due to the fact that the spectral clustering (performed on the similarity matrix) has better ability to deal with non-linearly separable data.

\subsubsection{Influence of Number of Nearest Neighbors $K$}
\label{sec:compK}
In this section, we test the influence of number of the nearest neighbors $K$ in our JMVFG method. As the value of $K$ varies from $3$ to $11$, the average performance (w.r.t. NMI) of JMVFG is illustrated in Fig. \ref{fig:paraK}.
As shown in Fig. \ref{fig:paraK}, our JMVFG, our JMVFG method is able to perform stably with varying number of nearest neighbors. Empirically, a moderate value of $K$, e.g., in the range of $[5,10]$, is suggested. In this work, we adopt $K=5$ in the experiments on all the benchmark datasets.

\end{document}